\documentclass[nohyperref]{article}
\usepackage{microtype}
\usepackage{graphicx}
\usepackage{subfigure}
\usepackage{booktabs} %
\usepackage{tabularx}
\usepackage{makecell}
\usepackage[hyphens]{url} %
\usepackage{hyperref}
\usepackage[font=small,skip=0pt]{caption}

\usepackage{tabu}
\usepackage[accepted]{icml2023}

\usepackage{amsmath}
\usepackage{amssymb}
\usepackage{mathtools}
\usepackage{amsthm}

\usepackage[capitalize,noabbrev]{cleveref}

\usepackage{comment}

\theoremstyle{plain}

\newtheorem{thm}{Theorem}
\newtheorem*{thm*}{Theorem}
\newtheorem*{definition*}{Definition}

\theoremstyle{definition}

\theoremstyle{remark}

\theoremstyle{plain}
\newtheorem{fact}{Fact}

\usepackage[textsize=tiny]{todonotes}

\newcommand{\E}{\mathbb{E}}
\newcommand{\R}{\mathbb{R}}

\newcommand{\Supp}{{\text{Supp}}}

\newcommand{\pit}{{\pi_\theta}}
\newcommand{\nabt}{{\nabla_{\theta}}}

\newcommand{\pRL}{p_\mathrm{RLKL}}
\newcommand{\JRL}{J_\mathrm{RLKL}}
\newcommand{\pbin}{p_\mathrm{GDC\_bin}}
\newcommand{\pdistr}{p_\mathrm{GDC\_dist}}
\newcommand{\KL}[2]{\mathrm{KL}(#1||#2)}

\newcommand{\JS}[2]{\mathrm{JS}(#1||#2)}
\newcommand{\TV}[2]{\mathrm{TV}(#1||#2)}

\newcommand{\RLKL}{{RLKL}}

\newif\ifCommentedI

\CommentedIfalse   % Level 1 comments removed

\ifCommentedI
\newcommand{\ptdyI}[1]{\textcolor{teal}{#1}} % Proposed Text by DongYoung 
\newcommand{\fdyI}[1] % colored footnote comment by DongYoung
{{\color{teal}\footnote{\textcolor{teal}{DY: #1}}}}
\newcommand{\pttkI}[1]{\textcolor{red}{#1}} % Tomek
\newcommand{\fbtkI}[1]{{\color{red}\footnote{\textcolor{red}{TK: #1}}}}
\newcommand{\ptgkI}[1]{\textcolor{orange}{#1}} % German
\newcommand{\fgkI}[1]{{\color{orange}\footnote{\textcolor{orange}{GK: #1}}}}
\newcommand{\ptmdI}[1]{\textcolor{blue}{#1}} % Marc
\newcommand{\fmdI}[1]{{\color{blue}\footnote{\textcolor{blue}{MD: #1}}}}
\newcommand{\fjrI}[1]{{\color{yellow}\footnote{\textcolor{yellow}{JR: #1}}}}
\else
\newcommand{\ptdyI}[1]{#1}
\newcommand{\fdyI}[1]{}
\newcommand{\pttkI}[1]{#1}
\newcommand{\fbtkI}[1]{}
\newcommand{\ptgkI}[1]{#1}
\newcommand{\fgkI}[1]{}
\newcommand{\ptmdI}[1]{#1}
\newcommand{\fmdI}[1]{}

\newcommand{\fjrI}[1]{}
\fi
\icmltitlerunning{Aligning Language Models with
Preferences through $f$-divergence Minimization}

\begin{document}

\twocolumn[
\icmltitle{Aligning Language Models with Preferences \\ through $f$-divergence Minimization}

\icmlsetsymbol{equal}{*}

\begin{icmlauthorlist}
\icmlauthor{Dongyoung Go}{navercorp,yonsei}
\icmlauthor{Tomasz Korbak}{sussex}
\icmlauthor{Germán Kruszewski}{nle}
\icmlauthor{Jos Rozen}{nle}
\icmlauthor{Nahyeon Ryu}{navercorp}
\icmlauthor{Marc Dymetman}{marc}
\end{icmlauthorlist}

\icmlaffiliation{navercorp}{Naver Corp}
\icmlaffiliation{yonsei}{Yonsei University}
\icmlaffiliation{nle}{Naver Labs Europe}
\icmlaffiliation{sussex}{University of Sussex}
\icmlaffiliation{marc}{Independent Researcher}

\icmlcorrespondingauthor{Dongyoung Go}{dongyoung.go@navercorp.com}

\icmlkeywords{Machine Learning, ICML}

\vskip 0.3in
]

\printAffiliationsAndNotice{}  %

\begin{abstract}
Aligning language models with preferences can be posed as approximating a target distribution representing some desired behavior. Existing approaches differ both in the 
 functional form of the target distribution and {the {algorithm} used %
 to approximate it.} 
 For instance, Reinforcement Learning from Human Feedback (RLHF) corresponds to minimizing a reverse KL from an \emph{implicit} target distribution arising from a KL penalty in the objective.
On the other hand, Generative Distributional Control (GDC) has an \emph{explicit}  target distribution and minimizes a forward KL from it using the Distributional Policy Gradient (DPG) algorithm.
In this paper, we propose a new approach, $f$-DPG, which allows the use of \emph{any} $f$-divergence to approximate \emph{any} 
target distribution \ptgkI{that can be evaluated.}
$f$-DPG unifies both %
frameworks (RLHF, GDC) and the approximation methods (DPG, RL with KL penalties). We show
the practical benefits of various choices of divergence objectives and demonstrate that there is no universally optimal objective %
{but} \ptdyI{that different divergences \ptmdI{present} different alignment and diversity 
\ptmdI{trade-offs}. 
We \ptmdI{show} that Jensen-Shannon divergence strikes a good balance 
\ptmdI{between these objectives,}
and frequently outperforms forward KL divergence by a wide margin, leading to significant improvements over prior work. \ptmdI{These} distinguishing characteristics 
\ptmdI{between}
divergences persist as the model size increases, highlighting the importance of selecting appropriate divergence objectives.}%
\begin{comment}
that  different divergences are good for approximating different targets. For instance, we discover that for GDC, the Jensen-Shannon divergence  frequently outperforms forward KL divergence by a wide margin, leading to significant improvements over prior work.
\end{comment}
%
%
%
%
\end{abstract}

\section{Introduction}

\begin{figure}[ht]
\centering
\centerline{\includegraphics[width=\columnwidth]{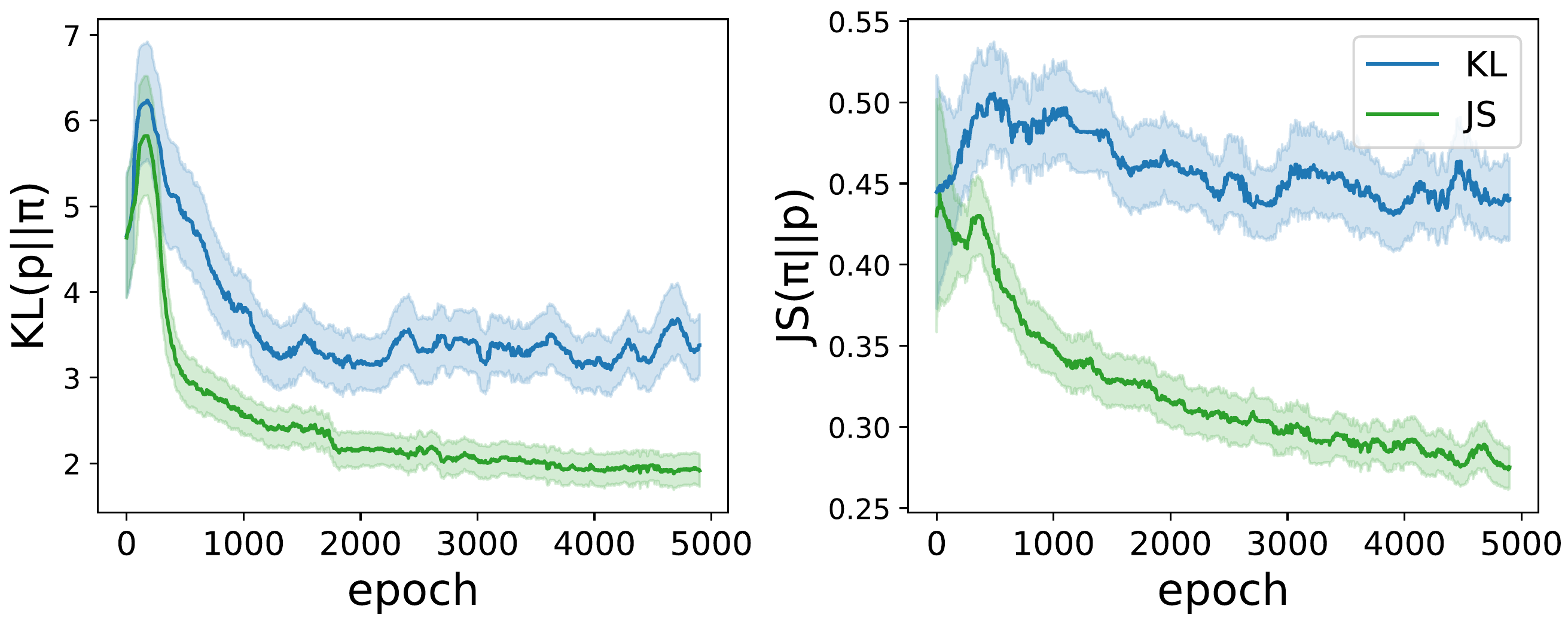}}
\caption{
{On many target distributions, the Jensen-Shannon (JS) divergence (green) outperforms the Kullback-Leibler (KL) divergence (blue) as an 
{\emph{objective,}}
even {when} %
performance is measured in terms of KL from the target $p$ (left panel, $\downarrow$ better). {See Sec. \ref{sec:pointwise}}.%
}%
}\label{fig:gdc-EBM_average_intro}
\end{figure}
{Language models (LMs)}
{have recently revolutionized the field of Natural Language Processing thanks to their generative capabilities, which are useful in a vast number of tasks~\citep{brown:etal:2020,Srivastava:etal:2022}. However, %
{generated texts} }
{can {also} violate widely-held human preferences, e.g. helpfulness \citep{lab}, non-offensiveness \cite{rtp}, truthfulness \cite{lin2021truthfulqa} or equal treatment \cite{cao-etal-2022-theory}.} 
Aligning LMs with human preferences is the problem of %
adapting the LM in such a way that generated content is perceived to match the human's intent~\citep{ouyang2022training} or that it is helpful, honest, and harmless \citep{lab,bai_constitutional}.
{Fundamentally,} an aligned LM can be seen as a desired target distribution that we would like to 
{generate from}~\cite{korbak2022rlBayesian}.
Some approaches leave this distribution {implicit}, to be defined as a side-effect of the proposed intervention. %
{These} include prompting with natural language instructions or demonstrations~\citep{lab}, %
using scorers or safety filters while decoding~\citep{recipes,xu2021bot}, supervised fine-tuning %
on curated data \citep{solaiman2021process,ngo2021_mitigating_harm,weibl2021,chung2022_scaling_instruction} or selected samples from the model %
\cite{star,scheurer2022,cascades}, and fine-tuning the language model using reinforcement learning with a learned reward function that approximates human feedback \citep[Reinforcement Learning from Human Feedback or RLHF;][]{ziegler2019fine,bai2022training,ouyang2022training}.
Instead, \citet{khalifa2021distributional} propose a framework that they 
{name}
Generation with Distributional Control (GDC), {where} they 
define the target distribution {$p$} that represents the aligned LM 
\ptmdI{as an EBM (Energy Based Model), namely an unnormalized version of $p$ that can be evaluated over any input $x$.}
\ptdyI{They then}
{train {the generative model} $\pit$ to approximate $p$}
via methods such as {D}istributional {P}olicy {G}radients~\citep[DPG;][]{parshakova2019distributional}, which 
\ptmdI{minimize}
the forward Kullback-Leibler (KL) divergence $\mathrm{KL}(p|| \pit)$ of 
{$p$} to 
{$\pit$}.
The advantage of such an approach is that it decouples the problem of describing the aligned LM from the problem of approximating it.
Furthermore, even if RL with KL penalties~\citep{todorov,kappen2012optimal,KL_Jaques17,KL_jaquesK19}, %
{the method used to fine-tune a LM in} RLHF, is defined only in terms of reward maximization, it has also been shown to be equivalent to minimizing the \emph{reverse} KL divergence $\mathrm{KL}(\pit|| p)$ of 
{$\pit$} %
to a target distribution {$p$} that can also {be written} explicitly in closed-form~\citep{korbak2022reinforcement}.%

The possibility of 
{approximating various} 
distributions 
{according to} different divergence measures begs the question: Does the choice of a divergence measure matter? In principle, all divergences {lead to} the same optimum, namely the target distribution $p$. However, when we restrict $\pit$ to a certain parametric family that does not include $p$ (i.e., the search space is \emph{mis-specified}), then the minimum {can} be found at different points, leading to optimal models with different properties. Moreover, different %
{divergences} present different loss landscapes: some might make it easier for stochastic gradient descent to find good minima. Finally, the space of possible divergence measures and forms of target distributions is a vast and largely uncharted terrain. Prior work {has} largely failed to decouple the form of a target distribution and the algorithm used for approximating it. 

Here, we introduce $f$-DPG, a new {framework} \ptmdI{for} fine-tuning {an LM} %
to approximate any given 
\ptmdI{target EBM,}
\ptmdI{%
by exploiting}
any \ptmdI{given} divergence in the $f$-divergences family, which includes \ptmdI{not only} the forward {KL} and the reverse KL cited above, but also Total Variation (TV) distance, Jensen-Shannon (JS) {divergence}, among others. 
$f$-DPG generalizes existing approximation techniques \ptmdI{both} DPG and RL with KL penalties algorithms, thus allowing us to investigate new %
{ways}
to approximate the target distributions defined by the GDC and RLHF frameworks.
In particular, we explore the approximation of various target distributions representing different alignment goals, which include imposing {lexical constraints}, {reducing social bias with respect to gender and religion, }%
enforcing factual consistency in summarization, and enforcing compilability of generated code.
We focus our experiments on four instantiations of $f$-DPG, namely KL-DPG, RKL-DPG, TV-DPG and JS-DPG, whose objective is to minimize the forward KL, reverse KL, TV and JS divergences, respectively, and evaluate 
{each experiment}
in terms of approximation quality as measured by all of these 
$f$-divergences.
We show that we can obtain significantly improved results over the {original KL-DPG algorithm \cite{parshakova2019distributional}} by minimizing other $f$-divergences, even when the approximation quality is evaluated under the lens of the forward KL.
Furthermore, we observe that while there is no single best optimization objective for all cases, JS-DPG often strikes a good balance {and significantly improves upon prior work \cite{khalifa2021distributional,korbak2021controlling}{, as illustrated in Fig. \ref{fig:gdc-EBM_average_intro}.}}
\ptdyI{Lastly, we \ptmdI{find} that f-DPG with an optimal objective continues to outperform suboptimal objectives as we scale model size from 127M parameters to 1.5B parameters (Sec.~\ref{sec:scaling_trend}). The smooth and gradual scaling trend observed with increasing model size suggests that our findings will generalize to even larger LMs.}

Overall, the contributions of the paper include:
\vspace{-10px}
\begin{enumerate}
\itemsep0em 
    \item Introducing $f$-DPG, a unifying framework for approximating any \ptdyI{EBM} target distribution by minimizing any $f$-divergence (Sec.~\ref{sec:fdpg}), and deriving a 
    {universal}
    formula for gradient descent with $f$-divergences (Theorem~\ref{thm:f-dpg}).
    \item Extending $f$-DPG to include baselines for variance reduction (Fact~\ref{fact:baseline}); and handling conditional target distributions (Fact~\ref{fact:conditional}).
    \item Investigating the performance of $f$-DPG on a {diverse} array of {thirteen LM alignment} tasks, three forms of target distributions, four $f$-divergence objectives and eight metrics.
\end{enumerate}

\section{Background}\label{sec:RM_DM_prelim}

{We can organize} approaches to LM alignment %
along two axes: how the target distribution is {constructed} and how it is approximated. The first problem roughly corresponds to representing human preferences through the specification of a probability distribution and the second to
{allowing the production of samples}
from that distribution. 

\subsection{Defining a Target Distribution}

The target distribution expresses {an ideal notion of an LM, incorporating human preferences,
as probabilities $p(x)$ over texts $x$ according to} {how well they satisfy the preferences.}
{Formally, $p(x)$ is often defined through a non-negative function $P(x)$ (aka an \emph{energy-based model} or EBM \ptmdI{\citep{lecun2006tutorial}}) such that $p(x)\propto P(x)$. 
\ptmdI{The model} $P(x)$ {(and $p(x)$ after normalization)} can be used to score samples, but not to directly \ptmdI{produce} them because it lacks an autoregressive form}.
In the rest of the paper, %
we will focus on 
target distributions modeling {three types of preferences} prominently employed in recent literature {about} GDC~\cite{khalifa2021distributional} and RLHF~\cite{ziegler2019fine,stiennon2020,ouyang2022training,menick_2022_sparrow,bai2022training}.

\paragraph{Binary preferences} For human preferences naturally expressible as a binary constraint $b(x) \in \{0,1\}$ (e.g. a sample $x$ must never contain a curse word), \citet{khalifa2021distributional} proposed the following target distribution:
\begin{equation}
    \pbin(x) \propto a(x)b(x), \label{eq:p_binary}
\end{equation}
where $a$ is a pretrained LM and $b(x) = {0}$ if $x$ contains a curse and $b(x) = {1}$ otherwise. 
{$\pbin$ is the distribution enforcing that all samples match the binary constraint, which deviates minimally from $a$ as measured by $\KL{\pbin}{a}$}.

\paragraph{Scalar preferences}\label{eq:rlhf}
Some human preferences, such as helpfulness, are more naturally expressed as scalar scores. Alignment with respect to these is typically addressed {with} RLHF~\citep{stiennon2020,ziegler2019fine,ouyang2022training}, which consists of, first, capturing human preferences as a reward function {$r(x)$} (e.g. scores given a reward model trained to predict human preferences) and second, {applying \textrm{RL with KL penalties}~\citep{todorov,kappen2012optimal,KL_Jaques17,KL_jaquesK19} to maximize this reward }%
while penalizing departure from $a(x)$:
{\begin{align}
\JRL(\theta) = \mathbb{E}_{x \sim {\pit} } \left[{r}(x) - \beta \log \frac{\pit(x)}{a(x)}\right]. \label{eq:J_RLKL}
\end{align}}%
This objective can be equivalently framed as
{minimizing the reverse KL, {$\KL \pit \pRL$, where the target distribution $\pRL$ is defined as}}: 
\begin{equation}
\pRL(x)  \propto a(x)\exp({r}(x)/\beta),\label{eq:p_RLKL}
\end{equation}%
where $\beta$ is a hyperparameter \citep{korbak2022reinforcement}. 

\paragraph{Distributional preferences} Finally, there is a class of distributional %
{preferences} \cite{ethical_lm_deepmind} that cannot be expressed as a function of a single sample $x$ but depend on the entire distribution, e.g. a particular gender distribution of persons mentioned in LM samples. \citet{khalifa2021distributional} {model such preferences} through distributional constraints using the following exponential family target distribution 
\begin{equation}
\pdistr(x) \propto a(x)\exp\Big[{\sum_i \lambda_i \phi_i(x)}\Big], \label{eq:p_distributional}
\end{equation}
where $\phi_i$ are features defined %
{over texts} (e.g. 
{the most frequent gender}
of people mentioned in $x$) and $\lambda_i$ are coefficients {chosen} %
{so that the expected values $\E_{x\sim p}\left[\phi_i(x)\right]$ match some desired values $\bar{\mu}_{i}$ (e.g., 50\% gender balance)}.
{The resulting distribution $p_\text{GDC-d}$ matches the target feature moments, while deviating minimally from  $a$ as measured by $\KL{\pdistr}{a}$.}

\subsection{Approximating the target distribution}

Drawing samples from a target distribution $p$ constitutes the inference problem. There are broadly two approaches to this problem: (i) augmenting decoding from $a$ at inference time to obtain samples from $p$ and (ii) training a new parametric model $\pi_\theta$ to approximate $p$ which can then be sampled from directly. %
The first family of approaches includes guided decoding methods \cite{pplm,qin2022cold}, {Monte Carlo sampling techniques such as rejection sampling to sample from simple distributions like $\pbin$~\cite{recipes,ziegler2022}, and Quasi Rejection Sampling (QRS)~\cite{eikema2022an} or MCMC techniques~\citep{miao2019cgmh,GoyalDB22} to sample from more complex distributions, such as $\pdistr$. %
}%
In the rest of the paper, we will focus on the second family: methods that train a new model $\pi_\theta$ to approximate $p$ by minimizing a divergence measure from $p$, $D(\pi_\theta|| p)$.
{\citet{khalifa2021distributional} uses 
Distributional Policy Gradients %
\citep[DPG; ][]{parshakova2019distributional}
to approximate the target distribution by minimizing $\KL{p}{\pit}$, or equivalently, $\mathrm{CE}(p, \pit)$:%
{
\begin{align}
\nabla_\theta \mathrm{CE}(p, \pit) = -\mathbb{E}_{x\sim \pit}{\frac{p(x)}{\pit(x)}\nabla_\theta}\log \pit(x). \label{eq:dpg}
\end{align}%
}}%
\vspace{-10px}
\section{Formal Aspects}

In this section, we describe the $f$-divergence family, and introduce a generic technique, $f$-DPG, for minimizing the $f$-divergence between a target distribution $p$ and a model $\pi_{\theta}$. We then describe the application of $f$-DPG to aligning language models with human preferences.

\subsection{$f$-divergences}
Consider a convex function 
$f:(0,\infty)\rightarrow\mathbb{R}$ with $f(1) = 0$. Let $f(0) \doteq \lim_{t\rightarrow0}f(t)$ and $f^{'}(\infty)\doteq\lim_{t\rightarrow0} t f(\frac{1}{t})$.%
\footnote{The limits are well-defined and take values in $(-\infty,\infty]$. The convention {for $f^{'}(\infty)$} is motivated by the fact that $\lim_{t\rightarrow \infty} f'(t) = \lim_{t\rightarrow 0} t f(\frac{1}{t})$ \cite{hiriart2013convex}.
}
Let $p_1, p_2$ be two distributions
over a discrete set $\mathcal{X}$.
The $f$-divergence 
between $p_1$ and $p_2$
can be defined as
\begin{align}
D_{f}(p_1||p_2)\doteq\E_{x\sim p_2}\left[f\left(\frac{p_1(x)}{p_2(x)}\right)\right]+f^{'}(\infty)\ p_1(p_2=0)
\label{Eq.:f-divergence} 
\end{align}
where $p_1(p_2=0)$
is the $p_1$-mass of the set $\{x\in \mathcal{X}:p_2(x)=0\}$ \cite{polyanski,liese2006divergences}.
The function $f$ is called a generator of $D_f$. 
By convention, if $p_1(p_2=0)=0$, the last term of Eq. \eqref{Eq.:f-divergence} is 
set to $0$
regardless of the value of $f^{'}(\infty)$ (which  
can be infinite).%
\footnote{Based on the commonly made 
assumption that the support of $p_1$ is dominated by the support of $p_2$ ($Supp(p_1) \subset Supp(p_2)$), Eq. \eqref{Eq.:f-divergence} simplifies to
{$D_{f}(p_1||p_2)=\E_{x\sim p_2}\left[f\left(\frac{p_1(x)}{p_2(x)}\right)\right]$.}
}
It can be shown that $D_{f}(p_1||p_2)\ge 0$ for any $p_1$
and $p_2$, {with equality if $p_1=p_2$; conversely,}  if $D_{f}(p_1||p_2) = 0$ and $f$ is strictly convex at $1$, then $p_1=p_2$.

The $f$-divergence family includes many important divergence
measures, in particular KL divergence $\text{KL}(p_1||p_2)$, reverse KL divergence $\text{KL}(p_2||p_1)$, Jensen-Shannon divergence, and Total Variation distance. We list these $f$-divergences and their generators in Tab. \ref{tab:f-dpg-loss}. For more details about notations and properties of $f$-divergences, see App. \ref{app:generalized_divergence} and also \citet{liese2006divergences,polyanski, sason2016f,sason2018f}.
\subsection{Distributional alignment with $f$-divergences}\label{sec:fdpg}

{Let $\mathcal{X}$ be} a discrete {countable or finite} set, {in our case a set of texts.}
Given a target probability distribution $p(x)$ over elements $x\in \mathcal{X}$, our goal is 
to approximate $p$ with a generative model (aka policy) $\pi_\theta$. %
On the other hand, the generative model $\pi_\theta$ is a parametric model, typically an autoregressive neural network, from which we can (i) directly sample and (ii) evaluate probabilities $\pi_\theta(x)$.  

We approach this problem by attempting to minimize the $f$-divergence of
$\pi_{\theta}$ to $p$:%
\footnote{We could 
have chosen to do $\text{min}_{\theta\in\Theta}D_{f}(p||\pi_{\theta})$. However the \emph{perspective transform} $f^{*}(t)\doteq t\ f(\frac{1}{t})$ 
allows interchangeability of arguments: $D_f(\pi_\theta||p)=D_{f^{*}}(p||\pi_\theta)$, 
{making either form possible.}
The form in Eq.~\eqref{Eq.:fdpg-loss} permits a simpler statement of our main theorem. See App. \ref{app:generalized_divergence}, \ref{app:proofs} for details.}%
\begin{align}
\text{min}_{\theta\in\Theta}D_{f}(\pi_{\theta}||p),
\label{Eq.:fdpg-loss}
\end{align}
where $\theta$ varies inside the parametric family $\Theta$. 
Note that when the family $\pit, \theta \in \Theta$ is ``well-specified'', 
i.e., when $\exists \theta_{0}\ \text{s.t.}\ p=\pi_{\theta_{0}}$, the true minimum of Eq~\eqref{Eq.:fdpg-loss} is $0$,
attained at $\theta_{0}$, whatever divergence $D_f$ is chosen. In contrast, when the family is ``mis-specified'' 
i.e. does not include $p$, the distribution $\pit$ with minimal divergence can be strongly dependent on the chosen divergence $D_f$.

Eq.~\eqref{Eq.:fdpg-loss} 
might be solved approximately using stochastic optimization with samples drawn from the distribution $p$, 
as the definition of $D_{f}(\pi_{\theta}||p)$ involves taking the expectation with respect to $p$.
However, it is often not possible to sample directly from $p$, while it is possible to sample from $\pit$. 
Our optimization technique is then based on the following core result, which we prove in App.~\ref{app:proofs}.
\begin{thm}\label{thm:f-dpg}
Let $p$ and $\pit$ be distributions over a discrete set $\mathcal{X}$ such that at least one of the following  conditions holds: 
{(i) $\forall \theta \in \Theta, \Supp(p)\subset \Supp(\pit)$, or 
(ii) $\Supp(\pit)$ does not depend on $\theta$.
}
Then:
\begin{align}
    \nabla_{\theta}D_{f}(\pi_{\theta}||p)=\E_{x\sim\pi_{\theta}}\left[f^{'}\left(\frac{\pi_{\theta}(x)}{p(x)}\right)\nabla_{\theta}\log\pi_{\theta}(x)\right].
\label{Eq.:fdpg-gradient}
\end{align}
\end{thm}
Note that it may happen in Eq~\ref{Eq.:fdpg-gradient} that $p(x) = 0$ and $\pit(x) > 0$, hence $\frac{\pit(x)}{p(x)} = \infty$, in which case the expression $f^{'}\left(\frac{\pi_{\theta}(x)}{p(x)}\right)$ should be understood as denoting the value $f'(\infty)$ as defined earlier.\footnote{
The derivative $f'(t)$ of any convex function $f(t)$ is defined almost everywhere, with the possible exception of a countable number of non-differentiable points, at which a subgradient can be used instead \cite{hiriart2013convex,rockafellar1970convex}. See also App. \ref{app: non-differentiability}.}

In the context of LMs, our domain of application, we will use Thm.~\ref{thm:f-dpg} in situations where $\pit$, being a standard softmax-based autoregressive model, has full support over $\mathcal{X}$ (i.e. $\Supp(\pit)= \mathcal{X}$) for all $\theta$'s, while the support of $p$ might be strictly included in $\mathcal{X}$ in some experiments (Sec.~\ref{sec:pointwise}, \ref{sec:conditional}).

It is instructive to consider Thm.~\ref{thm:f-dpg} in relation to rewards in RL. In the standard policy gradient algorithm \cite{williams1992simple}, to find the model that maximizes the average reward $\E_{x\sim\pi_\theta}\left[r(x)\right]$, 
one computes the gradient of the loss using the formula
$\nabla_{\theta}\E_{x\sim\pi_\theta}\left[r(x)\right]=\E_{x\sim\pi_\theta}\left[r(x)\nabla_{\theta}\log\pi_{\theta}(x)\right]$. 
The gradient in Eq. \ref{Eq.:fdpg-gradient} is very similar, 
with a ``pseudo-reward'' $r_\theta(x) = -f^{'}(\frac{\pi_{\theta}(x)}{p(x)})$, one difference being that now $r_\theta$ depends on $\theta$ (see \cite{korbak2022reinforcement} for related remarks). We refer to the approach in Eq. \ref{Eq.:fdpg-gradient} under the name \emph{$f$-DPG}, {in} reference to the original DPG (Distributional Policy Gradient) approach introduced in \cite{parshakova2019distributional}, which can be seen as a {special case} of $f$-DPG (``KL-DPG'')  {with} $D_f(\pit||p)$ {set} to $\text{KL}(p||\pit)$ {as discussed in Sec. \ref{sec:recovering_other_methods}}.

\subsection{{Adding} a baseline}\label{sec:baseline}
Based on the similarity {to policy gradients,} we adopt the widely used \emph{baseline} technique from RL, as previously studied in \citet{williams1992simple, baxter2001infinite, schulman2015high} and in the context of DPG in \cite{korbak2022reinforcement}. This technique involves subtracting a constant $B$ from the reward term, and does not introduce bias in the estimate of the gradient at a given $\theta$. %
In our {case}, with $r_\theta(x) \doteq -f^{'}(\frac{\pi_{\theta}(x)}{p(x)})$, we can write 
$\nabt D_{f}(\pit||p) = \E_{x\sim\pit} r_\theta(x) \nabt\log\pit(x) = \E_{x\sim\pit} (r_\theta(x) - B)\ \nabt\log\pit(x)${, based on the observation that $\E_{x\sim\pit} \ \nabt\log\pit(x) =0$ (see also App. \ref{sec:baseline-alternative-derivation}).}
\begin{fact}\label{fact:baseline}
Subtracting $B$ from $r_{\theta}(x)$ does not introduce bias into $f$-DPG gradient estimates.
\end{fact}
Typically, $B$ is chosen to be the average of the rewards, $B \doteq \E_{x\sim\pi_\theta}\left[r_\theta(x)\right]$.
In the experiments of Sec.~\ref{sec:experiments}, we use the baseline technique where $B$ is an estimate of the average of pseudo-rewards, unless otherwise specified.

\begin{table*}
\begin{centering}
\begin{tabular}{ccccc}
\toprule
$D_f(\pit||p)$ & {\footnotesize{}$f$} & {\footnotesize{}$f'$} & {\footnotesize{}$f^{'}\left(\frac{\pi_{\theta}(x)}{p(x)}\right)$}& {\footnotesize{}$f^{'}(\infty)$}\tabularnewline
\midrule
{\footnotesize{}Forward KL $\left(\text{KL}(p||\pit)\right)$} & {\footnotesize{}$f(t)=-\log t$} & {\footnotesize{}$f^{'}(t)=-\frac{1}{t}$} & {\footnotesize{}$-\frac{p(x)}{\pi_{\theta}(x)}$}& {\footnotesize{}$0$}\tabularnewline
{\footnotesize{}Reverse KL $\left(\text{KL}(\pit||p)\right)$} & {\footnotesize{}$f(t)=t\log t$} & {\footnotesize{}$f^{'}(t)=\log t+1$} & {\footnotesize{}$-\left(\log\frac{p(x)}{\pi_{\theta}(x)}\right)+1$}& {\footnotesize{}$\infty$}\tabularnewline
{\footnotesize{}Total Variation $\left(\text{TV}(\pit||p)\right)$} & {\footnotesize{}$f(t)=0.5\ |1-t|$} & {\footnotesize{}$f^{'}(t)=\begin{cases}
0.5 & \text{for }t>1\\
-0.5 & \text{for }t<1
\end{cases}$} & {\footnotesize{}$\begin{cases}
0.5 & \text{for }\frac{\pi_{\theta}(x)}{p(x)}>1\\
-0.5 & \text{for }\frac{\pi_{\theta}(x)}{p(x)}<1
\end{cases}$}& {\footnotesize{}$0.5$}\tabularnewline
{\footnotesize{}Jensen-Shannon $\left(\text{JS}(\pit||p)\right)$} & {\footnotesize{}$f(t)=t\log\frac{2t}{t+1}+\log\frac{2}{t+1}$} & {\footnotesize{}$f^{'}(t)=\log\frac{2t}{t+1}$} & {\footnotesize{}$\log 2-\log\left(1+\frac{p(x)}{\pi_{\theta}(x)}\right)$}& {\footnotesize{}$\log 2$}\tabularnewline
\bottomrule
\end{tabular}
\par\end{centering}
\caption{Some common $f$-divergences $D_f(\pi_\theta||p)$. 
In the convention of this table, the $f$ shown corresponds to the order of arguments $D_f(\pit||p)$. Thus the forward KL between the target $p$ and the model, $\text{KL}(p||\pit)$, corresponds to $D_{-\!\log t} (\pit||p)$, and similarly for the reverse KL, $\text{KL}(\pit||p)$, which corresponds to $D_{t \log t} (\pit||p)$, etc. Note that for symmetric divergences (TV and JS) the order of arguments is indifferent: $\text{TV}(\pit||p)=\text{TV}(p||\pit)$, $\text{JS}(\pit||p)=\text{JS}(p||\pit)$.
\label{tab:f-dpg-loss}}
\end{table*}

\subsection{Recovering Some Existing Methods}\label{sec:recovering_other_methods}
Various existing methods for aligning LM with preferences can be included in the $f$-DPG framework.
\paragraph{GDC} In GDC, fitting the policy $\pit$ to the target $p$ {(which is given by either one of Eq.~\ref{eq:p_binary} or Eq.~\ref{eq:p_distributional}}) is done using DPG \cite{parshakova2019distributional}, namely {by} minimizing the \textbf{forward KL}, $\text{KL}(p||\pit)$. In the $f$-DPG framework, $\text{KL}(p||\pit) = D_f(\pit||p)$ with $f(t)=-\log t$, $f'(t) = -1/t$, and Thm.~\ref{thm:f-dpg} leads to the formula:
\[
\nabt D_f(\pit||p) = \E_{x\sim\pit}-\frac{p(x)}{\pit(x)}\nabla_{\theta}\log\pit(x),
\]
which is equivalent to Eq. \ref{eq:dpg}.%

\paragraph{RL with KL penalties}
{
Let's rewrite the target distribution of Eq. \eqref{eq:p_RLKL} as $p(x) \doteq \pRL(x) = 1/Z\ a(x)\ e^{r(x)/\beta}$, where $Z$ is a normaliser. Then $\KL \pit p = D_f(\pit||p)$, with $f(t)=t\log t$ corresponding to \textbf{reverse KL}, and $f'(t) = 1 + \log t$. 
Thm.~\ref{thm:f-dpg} implies that:
\begin{align*}
    &\nabt D_{f}(\pit||p)\\
    &= \E_{x\sim\pit}\left(1+ \log\frac{\pit(x)}{Z^{-1} a(x)\exp(r(x)/\beta)}\right)\nabt\log\pit(x)\\
    &= \E_{x\sim\pit}\left(-\frac{r(x)}{\beta} + \log\frac{\pit(x)}{a(x)}\right)\nabt\log\pit(x),
\end{align*}
where we have exploited the fact that $1+\log Z$ is a constant, hence $\E_{x\sim\pit}
(1+\log Z)\ \nabt \log\pit(x) = 0$.
Up to the constant {factor} $\beta$, this form recovers the usual formula for estimating the gradient of the loss defined in Eq. \eqref{eq:J_RLKL}:}
$
    \nabt J_{\mathrm{\RLKL}}(\theta) = \E_{x\sim\pit}\left(r(x)-\beta\log\frac{\pit(x)}{a(x)}\right)\nabt\log\pit(x).
$
\subsection{Estimating $Z$}
The target distribution $p$ is often defined as $p(x) \propto P(x)$, where $P(x)$ is a non-negative function over $\mathcal{X}$. %
The distribution $p$ can then be computed as $p(x) = 1/Z\ P(x)$, where $Z$ is the normalizing constant (partition function) defined by $\sum_{x\in \mathcal{X}} P(x)$.
An estimate of $Z$ can be obtained by importance sampling, using samples from the current $\pit$, based on the identity $Z = \E_\pit \frac{P(x)}{\pit(x)}$. Each such estimate is unbiased, and by averaging the estimates based on different $\pit$'s, one can obtain a more precise estimate of $Z$, exploiting \emph{all} the samples obtained so far. For details about the estimate of $Z$, see Algorithm \ref{alg:f_dgp} in App.~\ref{app:proofs}, as well as the ablation study in App.~\ref{app:Z_ablation}. 

\subsection{Conditional Target Distributions}
For a conditional task such as machine translation, summarization or dialogue, {where $\pit$ is defined as a conditional distribution $\pit(x|c)$, } {we adapt the conditional generalization of DPG} introduced in \citet{korbak2021controlling}. Given a distribution over contexts $\tau(c)$ {and a map from a context $c$ to a target distribution $p_c$}, we have ({see App. \ref{app:conditional} for details}): 
\begin{fact}\label{fact:conditional}
{$f$-DPG is generalized to the conditional case by optimizing the loss}
\begin{equation}
\E_{c\sim\tau(c)}\left[\nabla_{\theta}D_{f}(\pit(\cdotp|c)||p_{c}(\cdotp))\right].
\end{equation}
\end{fact}

\section{Experiments}\label{sec:experiments}
{We study four instantiations of $f$-DPG, namely KL-DPG, RKL-DPG, TV-DPG and JS-DPG, corresponding to minimizing the forward KL, reverse KL, Total Variation, and Jensen-Shannon divergences, respectively.
We use an exponential moving average baseline with weight $\alpha=0.99$ for all, except for KL-DPG, where we use the analytically computed value of the pseudo-reward expectation, which amounts to {$1$} 
~\citep{korbak2022reinforcement}.}
{We evaluate them on a diverse array of tasks  {including} imposing {sentiment constraints} (Sec.~\ref{sec:scalar}), {lexical constraints} (Sec.~\ref{sec:pointwise}), debias{ing} genders' prevalence and religious groups' regard (Sec.~\ref{sec:distributional}), and context-conditioned tasks, such as enforcing factual consistency in summarization (Sec.~\ref{sec:conditional}) or compilability of generated code (see App. \ref{app:conditional-additional-info}).}
{Unless specified otherwise}, we use a pretrained GPT-2 {``small''}  \cite{radford2019language} with 117M parameters for the initial model. \ptdyI{\ptgkI{Yet,} we demonstrate in Sec.~\ref{sec:scaling_trend} that the observations continue to hold for 
models of larger size. }Implementation details and hyper-parameters are available in App. \ref{app:implement_details}.

\paragraph{Metrics} We report the following 
key metrics. We add task-specific metrics if needed.
\vspace{-5px}
\begin{enumerate}
\itemsep0em 
\item $D_{f}(\pi_{\theta}||p)$, the $f$-divergence between $p$ and $\pit$, with four different $f${'s} corresponding to forward KL, $\KL{p}{\pit}$; reverse KL, $\KL{\pit}{p}$; {T}otal {V}ariation, $\TV{\pit}{p}$; and Jensen-Shannon, $\JS{\pit}{p}$. We use importance sampling to estimate {these divergences}. %
\item {$\text{KL}(\pi_{\theta}||a)$, a measure of the divergence from original {LM} $a$ \cite{ziegler2019fine,khalifa2021distributional}.}
\item {\ptdyI{Alignment score, measured by }moments $\mathbb{E}_{x\sim \pit}{\phi(x)}$ of a feature of interest $\phi(x)$.}
\item {Normalized }Entropy~\citep{berger1996maximum}, a measure of diversity in probability distribution normalized by number of tokens.
\item {{Standard deviation} of a minibatch's pseudo-rewards, $\mathrm{std}(r_\theta(x))$, where $r_\theta$ is defined as in Sec.~\ref{sec:baseline}.}
\end{enumerate}

\subsection{Alignment with Scalar Preference{s}}\label{sec:scalar}
\paragraph{Task} {We begin with the task of maximizing a scalar preference with KL penalties, whose target distribution, $\pRL$, is defined in Eq. \ref{eq:p_RLKL}. We set} %
{$r(x)=\log \phi(x)$} where $\phi(x)$ is the probability {returned by}
a sentiment classifier fine-tuned from Distil-BERT \cite{hf_canonical_model_maintainers_2022}. {This reward function is optimal {for modeling}
a decision-maker which given $k$ different samples $x_1,\dots,x_k$, will pick $x_i$ with probability proportional to $\phi(x_i)$~(see Appendix \ref{app:choice_model})}. %
We set $\beta=0.1${, which is in line with the range of values explored by} \citet{ziegler2019fine}. 
Note that {applying} RKL-DPG {on $\pRL$} %
is equivalent to the RL with KL penalties method, as described in Sec.~\ref{sec:recovering_other_methods}. However, {through $f$-DPG we can explore alternative objectives to approximate the same target.}  %
\paragraph{Results} 
Fig. \ref{fig:scalar-klc_EBM} 
shows the evolution of {the above-mentioned metrics. }%
Further details are given in Fig. \ref{fig:scalar_preference_all} in the Appendix. 
{We observe that whereas RKL{-DPG} achieves by far the best performance in terms of reverse KL, $\KL{\pit}{p}$ {(top-right)}, it fails to minimize all other divergence metrics. This shows that minimizing one divergence does not necessarily imply that other divergences will follow. 
Notably, }
RKL-DPG yields the highest value of \ptdyI{alignment score} $\E_{\pi_\theta}[\phi(x)]$
at the cost of a significant departure from $a$.
{We connect this to the strong influence that low values $p(x)$ have on RKL{-DPG}, which induces a large pseudo-reward for strongly reducing $\pit(x)$ on those samples (see Sec \ref{sec:reward_comparision}) and produces the spike at the beginning of training in $\mathrm{std(rewards)}$. This can lead $\pit(x)$ to concentrate on high-probability regions of $p(x)$, at the cost of diversity, which can also be seen in {the} low entropy of the generated samples.}
{Interestingly, the three remaining variants of DPG (KL, TV and JS) consistently minimize all four tracked divergences, with JS{-DPG} performing best overall.}%

{In App. \ref{app:gen_quality}, we show }{additional metrics on generated sentences{, which} show low diversity 
but high quality 
for RKL-DPG, compared to other $f$-DPGs, suggesting it captures a subset %
of the target distribution (``mode collapse''), as commonly observed in other generative models \cite{Huszar15,CheLJBL17,mescheder2018training}}.%
\begin{figure}[ht]
\centering
\centerline{\includegraphics[width=\columnwidth]{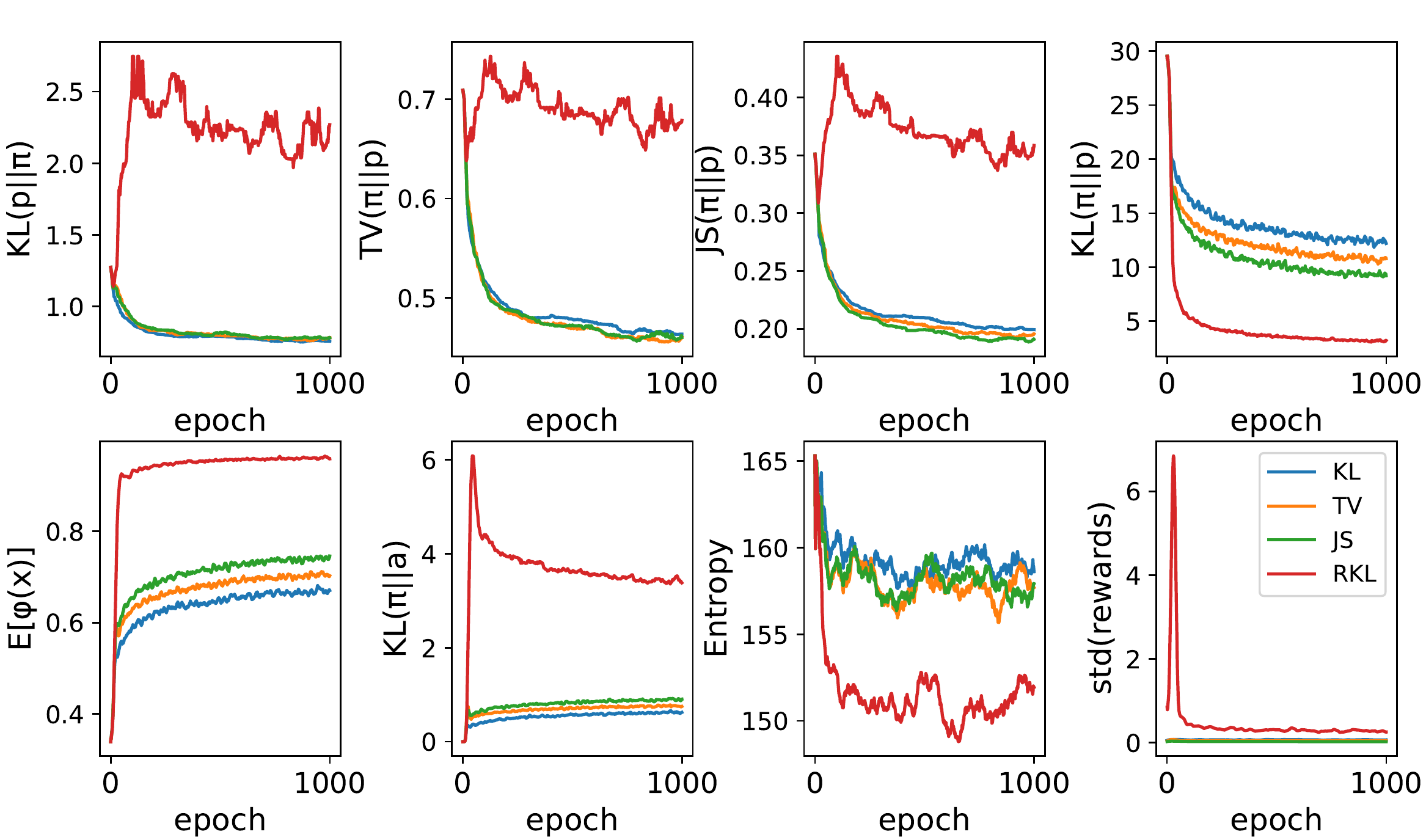}}
\caption{
Comparison of $f$-DPG on sentiment preference. Evaluation metrics: four $f$-divergences $D_f(\pi_\theta||p)$ (↓ better), \ptdyI{alignment score }$\E_{\pi_\theta}[\phi(x)]$ (↑ better), entropy (↑ better), standard deviation of pseudo-reward {$\mathrm{std}(r_\theta(x))$}.}\label{fig:scalar-klc_EBM}
\end{figure}
\subsection{Alignment with {Lexical Constraint}{s}}\label{sec:pointwise}
\paragraph{Task} 
In this task, we constrain the presence of a specific word in the generated text. 
{Following \citet{khalifa2021distributional}, we formulate this goal as a binary preference on the LM by using a target distribution $\pbin$, where $b(x)=1$ iff the target word appears in the sequence $x$, and using a scalar preference target distribution $\pRL$ where $r(x)$ is set in the same way as $b(x)$ above.}
{
Note that 
in the GDC framework, $\pbin(x)=0$ when $b(x)=0$, implying that reverse KL, {namely $\KL \pit p$, becomes infinite},} so RKL-DPG cannot be used {(nor measured)} {for that target}. %
We use four words with {different occurrence frequency}: %
``amazing''($1\cdotp10^{-3}$), ``restaurant'' ($6\cdotp10^{-4}$), ``amusing'' ($6\cdotp10^{-5}$), and ``Wikileaks'' ($8\cdotp10^{-6}$).
\paragraph{Results} 
The aggregated evolution of the metrics 
{for both GDC and RL {with KL penalties} framework}
is presented in Fig.~\ref{fig:lexical_result_merge}
{(Fig.~\ref{fig:gdc-EBM_average_intro} shows a simplified view of Fig. \ref{fig:lexical_result_merge} (a)).}
{Disaggregated} results for each 
task are presented %
on App. \ref{app:additional_figures}.
{We see that all variants of $f$-DPG reduce the divergence from the target distribution across all measured $f$-divergences. Furthermore, as expected, convergence to the target is connected with the success ratio in producing the desired word, $\E_{\pit}\left[b(x)\right]$, {while balancing it with {{a} moderate divergence from $a$}, $\KL{\pit}{a}$}.%
} %
This reflects that approaching the optimal distribution $p$ translates into metrics in the downstream task. %
{Strinklingly, the original KL-DPG is outperformed by all other variants of $f$-DPG, even in terms of forward KL. We hypothesize that this is linked to the high variance of the pseudo-rewards in KL-DPG, %
{as visualized } 
in the last panel of 
{Fig.~\ref{fig:lexical_result_merge} (a) and 
(b)}.
In Sec. \ref{sec:reward_comparision}, we 
{suggest}
an interpretation for this.}
{We also observe that RKL-DPG tends to produce distributions with lower normalized entropy. Despite this effect, }
{w}e found no significant difference in diversity among the generated sentences {(see Tab. \ref{tab:gen_quality_pointwise} in App. \ref{app:gen_quality})}
\begin{figure}[ht]
\centering	
\begin{tabular}{cc}
{\footnotesize{}{(a) lexical constraint with $\pbin$}} \\ 
\centerline{\includegraphics[width=\columnwidth]{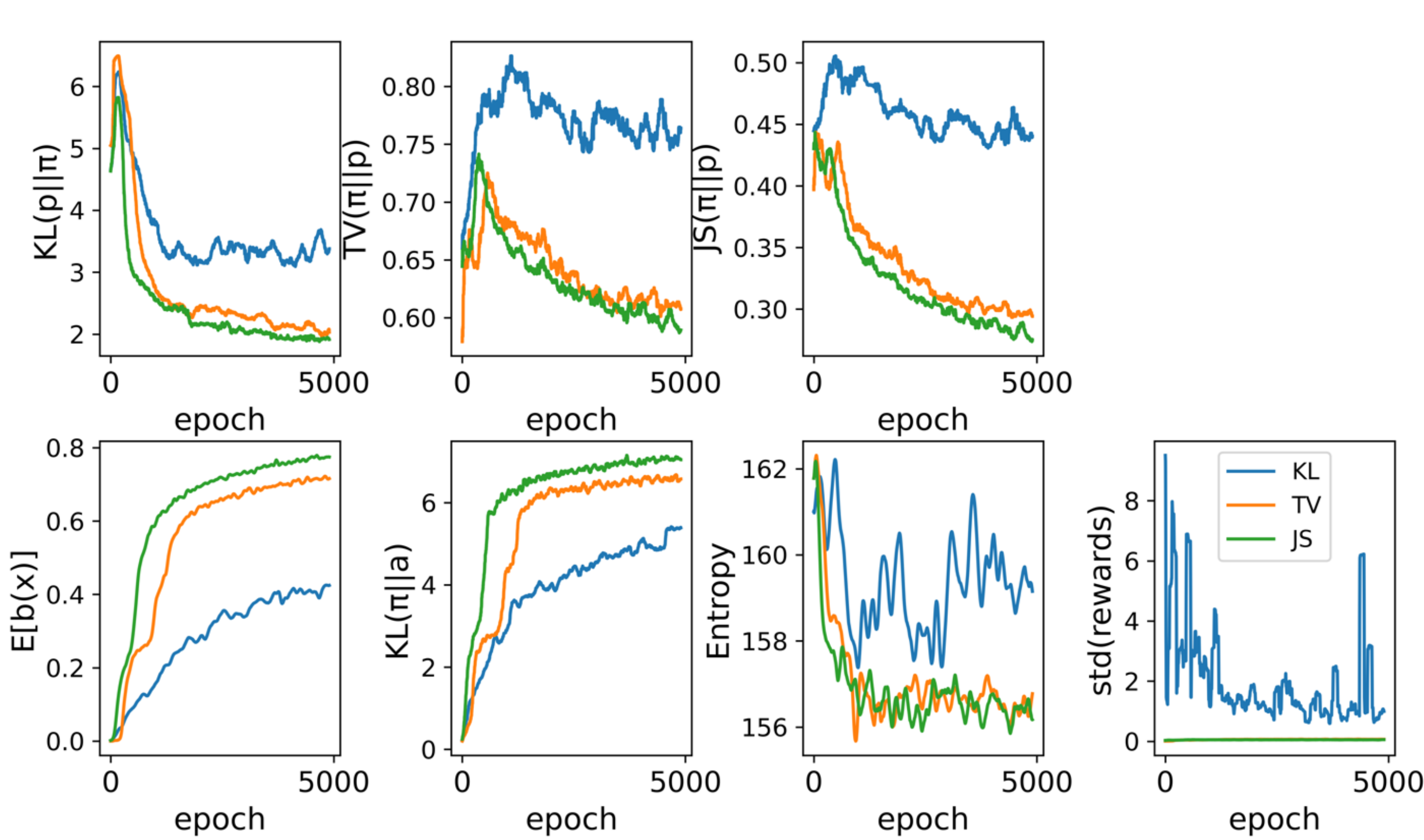}} \\
 {\footnotesize{(b) lexical constraint with $\pRL$}}\\ 
 \centerline{\includegraphics[width=\columnwidth]{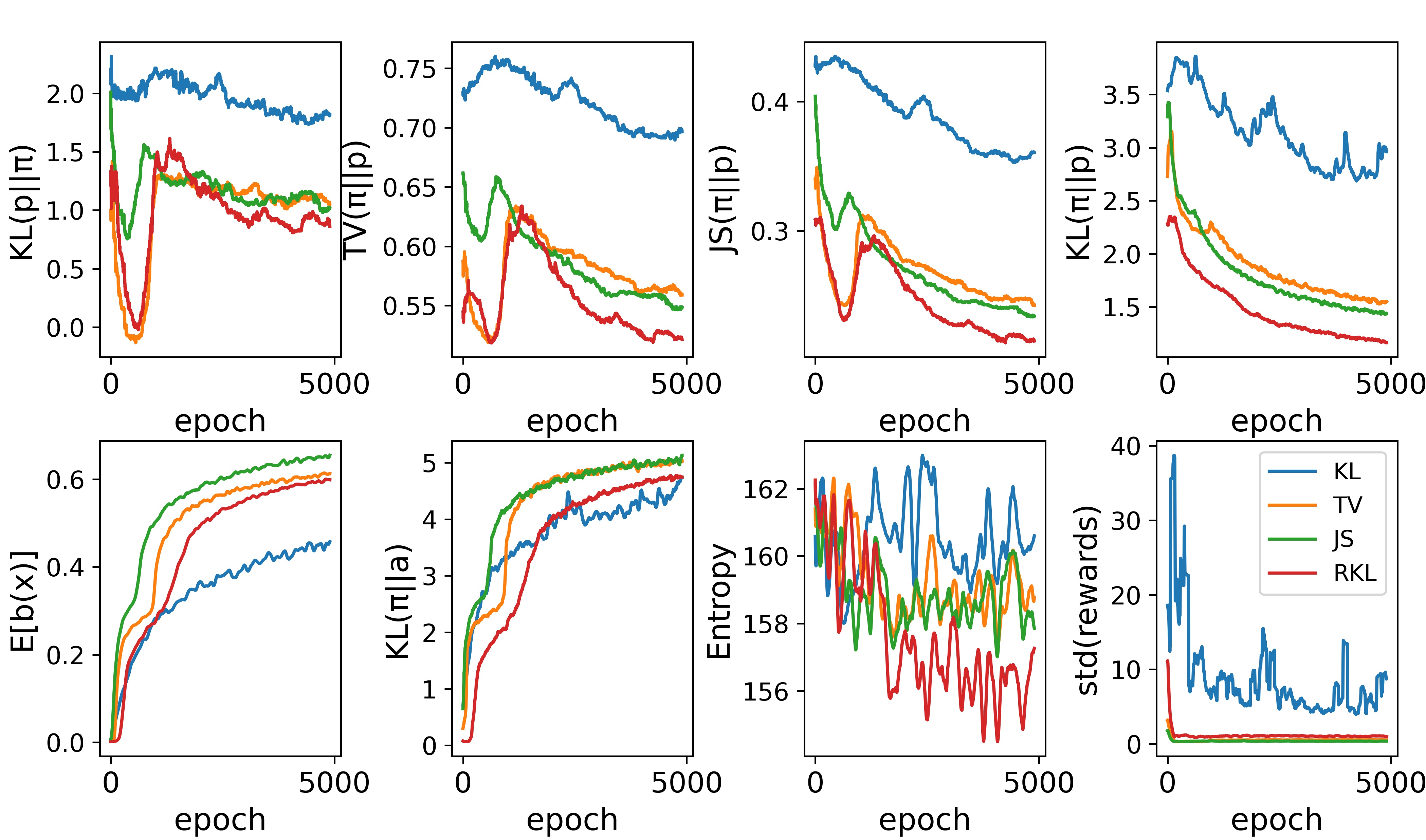}} 
\end{tabular}
\caption{Comparison of $f$-DPG aggregated on four lexical constraints. Standard deviations are suppressed for clarity. Evaluation metrics: four $f$-divergences $D_f(\pi_\theta||p)$ (↓ better), \ptdyI{alignment score }$\E_{\pi_\theta}[b(x)]$ (↑ better), entropy (↑ better), standard deviation of pseudo-reward {$\mathrm{std}(r_\theta(x))$}. }
\label{fig:lexical_result_merge}
\end{figure}
\subsection{Alignment with Distributional Constraint{s}}\label{sec:distributional}
\paragraph{Task} 
We {now investigate enforcing distributional preferences on {the LM}. We focus on debiasing the pretrained model on two kinds of preferences,
{namely} {%
genders' prevalence \cite{khalifa2021distributional} and regard relative to religious groups.}
}
{
{The preferences for the genders' debiasing task are defined as } $\phi_{1}(x)=1$ iff $x$ contains more female than male pronouns, with desired moment $\bar{\mu}_1=0.5$ and $\phi_{2}(x)=1$ iff $x$ contains at least one of the words in the `science' word list compiled by \citet{pplm}, with desired moment $\bar{\mu}_2=1$. 
{For regard debiasing}, we use a single distributional constraint where $0<\phi(x)<1$ is a regard score of the sentence when prompted with {\fontfamily{qcr}\selectfont
Muslims}, evaluated with a pretrained classifier \cite{sheng2019woman}. We set the desired moment $\bar{\mu}=0.568$, the regard score observed {\fontfamily{qcr}\selectfont
Christians}. The initial average regard score given {\fontfamily{qcr}\selectfont
Muslims} is $0.385$.
}
For the first experiment, we use GPT-2 small {as {the }initial model {$a$}}, additionally fine-tuned on the WikiBio dataset \cite{lebret2016neural}{, whereas for the last one we use vanilla GPT-2 small.}
\paragraph{Results} %
{
{We }report the results of {both experiments on} %
{Fig.~\ref{fig:distributional_result_merge}.}
}
For the regard score rebalancing, we {considerably reduce bias in }%
the regard score for two different demographic groups, from initial regard score ratio $\E\left[\phi(x)\right|\text{{\fontfamily{qcr}\selectfont
Christians}}]:\E\left[\phi(x)\right|\text{{\fontfamily{qcr}\selectfont
Muslims}}]=1:0.677$  to $\E\left[\phi(x)\right|\text{{\fontfamily{qcr}\selectfont
Christians}}]:\E\left[\phi(x)\right|\text{{\fontfamily{qcr}\selectfont
Muslims}}]=1:0.801$ {on average}. 
{Interestingly, this task showcases a weakness of TV-DPG: Because the original distribution is already close to the target, the hard-thresholded pseudo-reward has a large variance (last panel of Fig~\ref{fig:distributional_result_merge}(b)), inducing noisy gradient estimates and, consequently, sub-optimal convergence. }%
{Concerning the {gender debiasing} %
experiments, we can see that all other variants of $f$-DPG outperform the original KL-DPG explored in \citet{khalifa2021distributional}, with RKL-DPG 
giving the best results %
{and} better matching the pointwise constraint although {seemingly} at the cost of lower diversity as measured by the entropy.}

\begin{figure}[ht]
\centering	
\begin{tabular}{cc}
{\footnotesize{(a) distributional constraint for gender prevalence}} \\ 
\centerline{\includegraphics[width=\columnwidth]{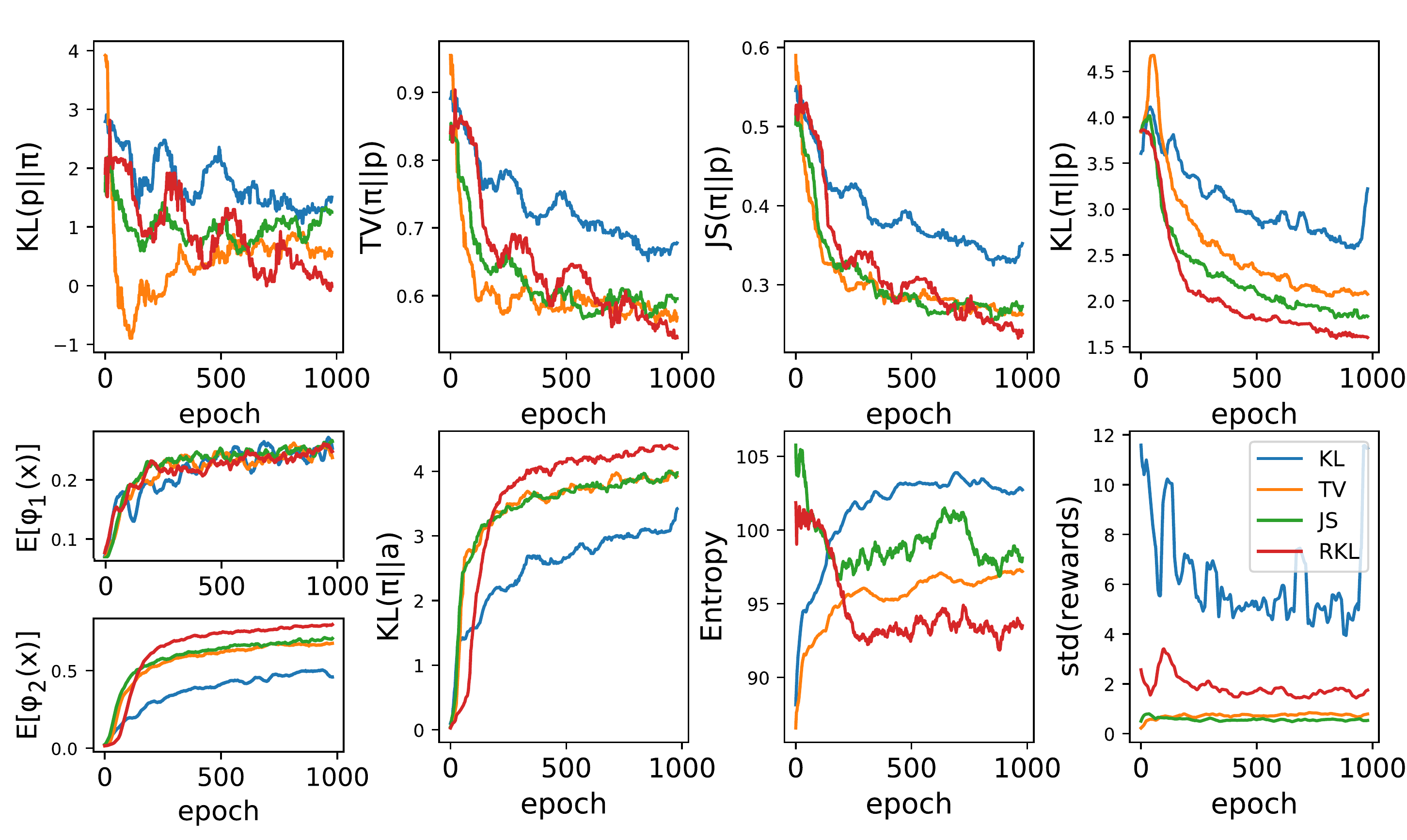}} \\
 {\footnotesize{(b) distributional constraint for regard rebalancing}}\\ 
 \centerline{\includegraphics[width=\columnwidth]{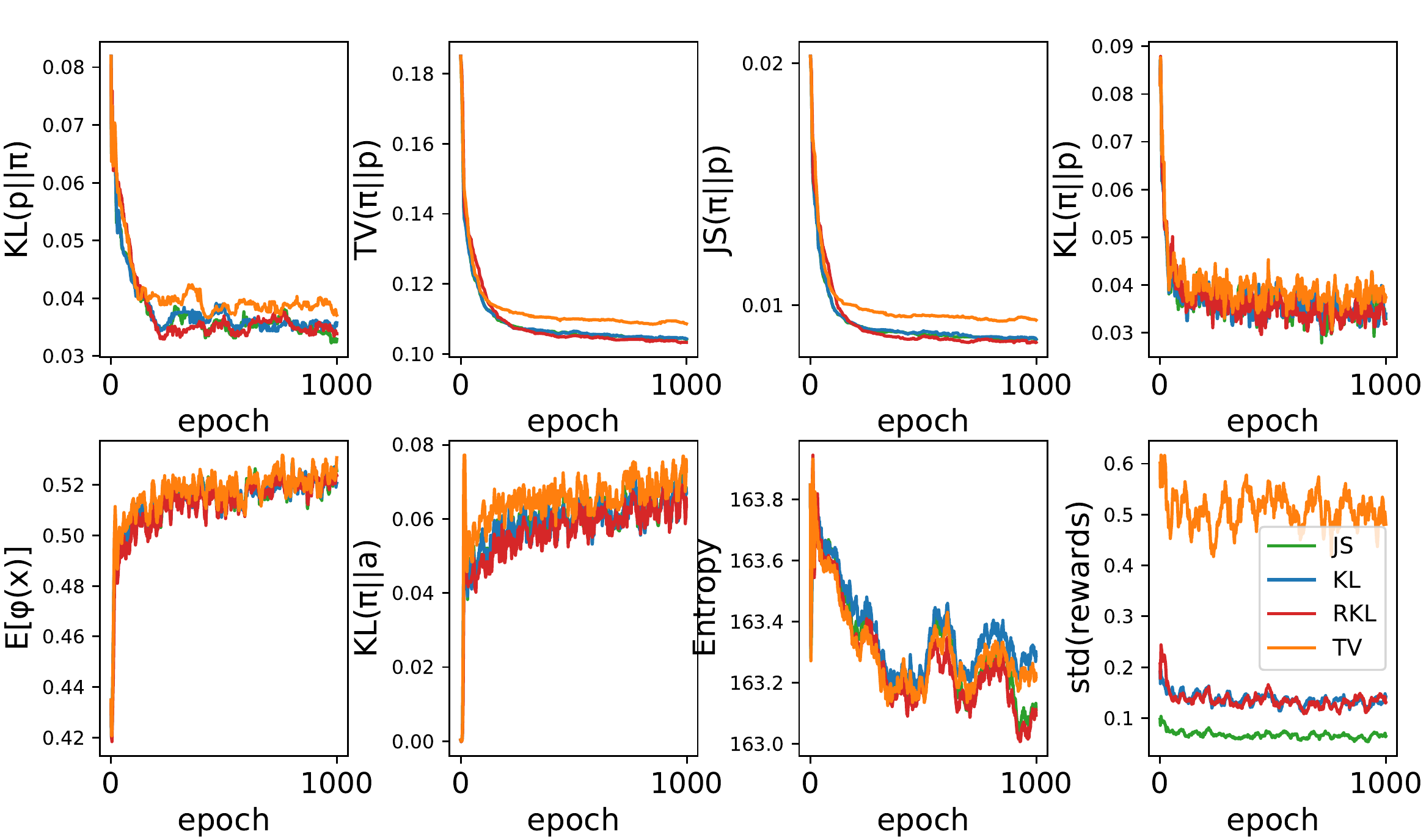}} 
\end{tabular}
\caption{Comparison of $f$-DPG aggregated on distributional constraints. Evaluation metrics: four $f$-divergences $D_f(\pi_\theta||p)$ (↓ better), \ptdyI{alignment score }$\E_{\pi_\theta}[\phi(x)]$ (↑ better), entropy (↑ better), standard deviation of pseudo-reward {$\mathrm{std}(r_\theta(x))$}. }
\label{fig:distributional_result_merge}
\end{figure}
\subsection{Alignment with Conditional Constraint{s}}\label{sec:conditional}
\paragraph{Task} 
We {adopt the conditional task from} %
\citet{korbak2021controlling}, %
{%
{which aims} to constrain {the  T5~\citep{raffel2020exploring}} language model to generate {more} factually faithful {summaries} \cite{MaynezNBM20,nan2021entity}. {Specifically, let} $\text{NER}(\cdot)$ {denote} the set of named entities found in a text. %
{Then, } $b(x,c)=1$ iff $\left[\text{NER}(x)\subseteq\text{NER}(c)\right]\land\left[\left|\text{NER}(x)\right|\ge4\right]${,} and $0$ otherwise. 
{{Following the authors, we }sample source documents from the} {the} CNN/Daily Mail dataset \cite{nallapati2016abstractive}, i.e. $\tau(c)$ is a uniform distribution over a given subset of source documents. %
In addition to the divergences, we evaluate the performance using 
Rouge \cite{lin-2004-rouge}, a measure of summarization quality in terms of unigram overlap between the
source document and ground truth summary (See App. \ref{app:conditional} for additional metrics and more experiment{s} {on code generation with compilability preferences}). %
}  %
\paragraph{Results} 
We present the evolution of metrics in Fig. \ref{fig:summarization_short}. 
The result{s} show that $f$-DPG {increases} the fraction of consistent named entities in summarization{, and 
{interestingly,} this also leads to indirect improvement in the overall quality of
generated summaries compared to ground truth, even though ground truth summaries are not used in training}. 
{As {also} observed in Sec.~\ref{sec:pointwise}, }
JS-DPG leads to better convergence to $p$ than {KL-DPG {as} used} %
in \citet{korbak2021controlling}.

\begin{figure}[ht]
\centering
\centerline{\includegraphics[width=1\columnwidth]{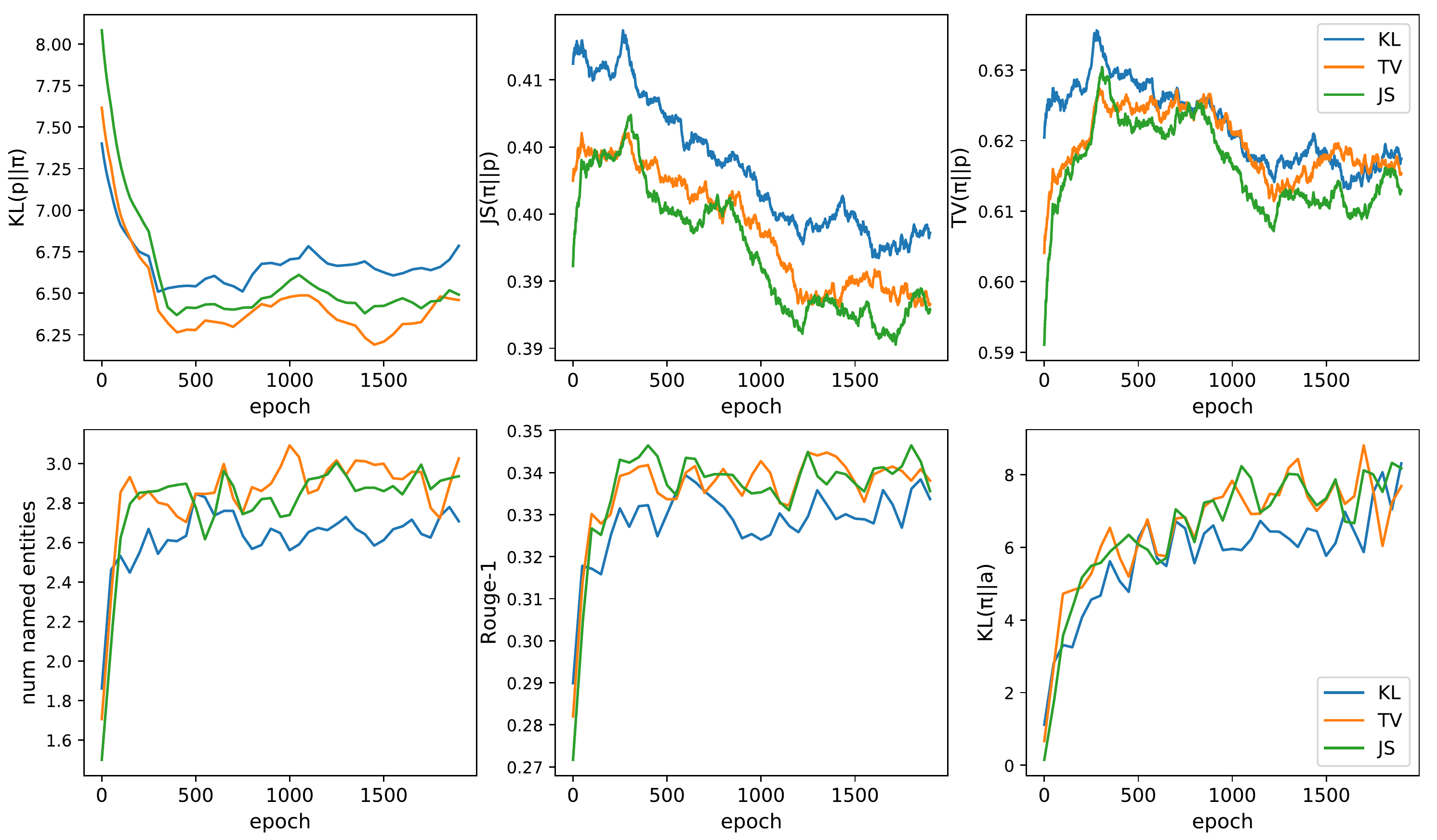}}
\caption{
Comparison of $f$-DPG on factual summarization. Evaluation metrics: 3 $f$-divergences $D_f(\pi_\theta||p)$ (↓ better), number of named entities (↑ better), Rouge (↑ better).}\label{fig:summarization_short}
\end{figure}

\subsection{Scaling Trends of $f$-DPG}\label{sec:scaling_trend}
\ptdyI{We conduct experiments to investigate the effect of model size on our approach using the  scalar preference task described in Sec.~\ref{sec:scalar}. Specifically, we gradually increase the model size from GPT-2 ``small'' (117M \pttkI{parameters}) to ``xl'' (1.5B \pttkI{parameters}) while tracking two important metrics: alignment score, which is measured by the expected reward $\E_{\pi_\theta}[\phi(x)]$, and diversity, which is measured by the entropy.
Figure \ref{fig:scaling-trend} demonstrates that the alignment score steadily improves as the model size increases. However, \ptmdI{we observe} persistent differences between the divergence objectives for different f-DPGs, \pttkI{leaving} the general order between f-DPGs intact with increasing model size (See Fig.~\ref{fig:scaling_learning_curve} \ptmdI{in} App.~\ref{app:additional_figures} for evolution of metrics through training epochs).
The scaling trend of LM alignment\ptmdI{,} characterized by a gradual and predictable increase without sudden shifts in performance, aligns with previous findings in the literature \cite{bai2022training}. Nonetheless, our study further emphasizes the importance of proper divergence objectives, as increasing model size alone does not necessarily bridge the gap between optimal and suboptimal objectives. The smooth and gradual increase of the alignment score as a function of model size suggests that our findings will generalize to even larger LMs.}

\begin{figure*}[ht]
\centering
\centerline{\includegraphics[width=\textwidth]{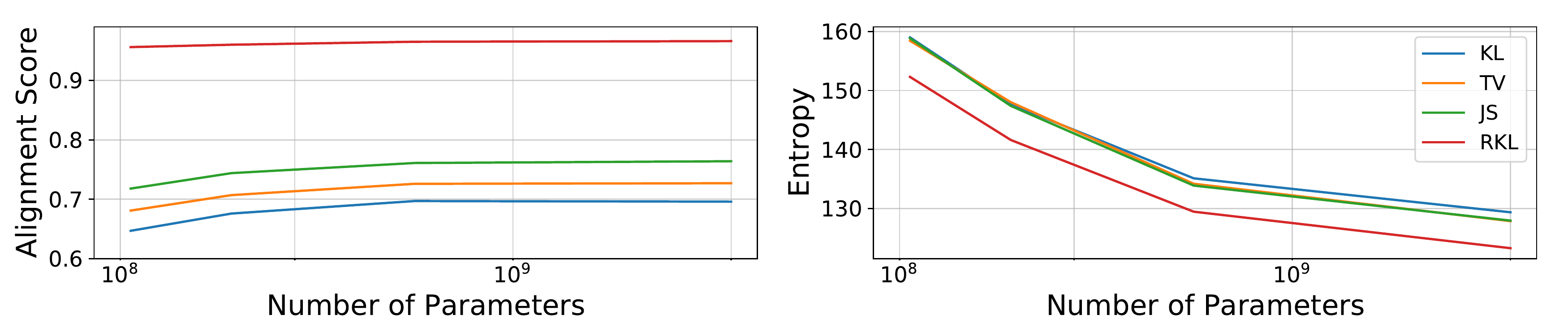}}
\caption{\ptdyI{The scaling trend of $f$-DPG on sentiment preference. 
The $x$-axis denotes number of parameters of the LM $\pi_\theta$ and {the} $y$-axis denotes the alignment score and diversity measured by the expected reward $\E_{\pi_\theta}[\phi(x)]$ and \ptmdI{by} entropy, respectively.}
}\label{fig:scaling-trend}
\end{figure*}

\subsection{Ablation Study}
This section presents just the key findings of our study. %
{Full results} and detailed discussions can be found in {App. \ref{app:ablation}}.%

\paragraph{Effect of parameter family {capacity}}
{All experiments presented so far correspond to {possibly} mis-specified target distribution{s}. To understand whether the observed behavior of different variants $f$-DPG is affected by this factor, we used pre-trained models with the same architecture {as $\pit$ and $p$}.} %
We found 
that KL-DPG again lags considerably {in terms of divergence, while presenting a} high variance of %
{in the} pseudo-reward. %
RKL-DPG shows a significant drop of entropy in the initial phase, but with full capacity of parameter family, the model can recover, and cover the rest of the distribution.
{Additionally, applying {zero-shot} the fine-tuned LMs %
to a summarization task, following \citet{radford2019language}, we found that the they recover to a large extent the quality of {the target distribution.}
}

\paragraph{Effect of training scheme} {We examined different training schemes %
for the lexical constraint on ``amazing'' from Sec.~\ref{sec:pointwise}}. %
{We saw} that the use of a baseline technique improve{s} the performance of the $f$-DPG method, with RKL-DPG showing the greatest benefit. 
Additionally, we found that {even though} a large batch size {is effective at reducing the variance of KL-DPG, we still observe KL-DPG to perform comparatively worse than other divergences.%
} %
{Finally, we observe that our }%
{importance sampling estimates} converged to the true value of $Z$.

\vspace*{-2mm}
\section{Discussion {and Conclusion}}\label{sec:reward_comparision}

\begin{figure}[ht]
\centering
\centerline{\includegraphics[width=\columnwidth]{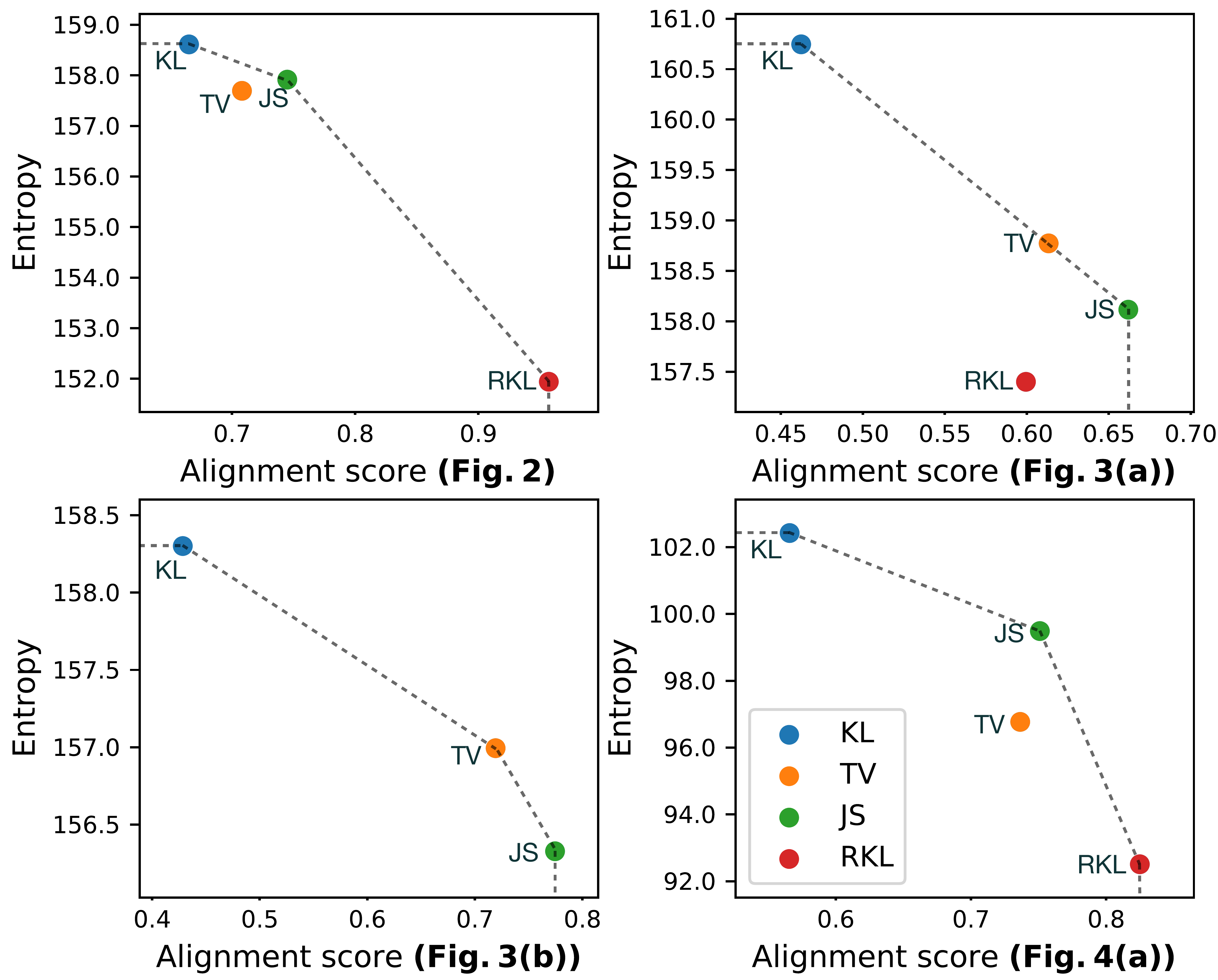}}
\caption{\ptdyI{Pareto frontier of $f$-DPG for different alignment tasks; sentiment preference (Fig.~\ref{fig:scalar-klc_EBM}), lexical constraints (Fig.~\ref{fig:lexical_result_merge}(a), (b)), and distributional constraint for gender prevalence (Fig.~\ref{fig:distributional_result_merge}(a))}}
\label{fig:pareto_frontier}
\end{figure}

\begin{figure}[ht]
\centering
\centerline{\includegraphics[width=\columnwidth]{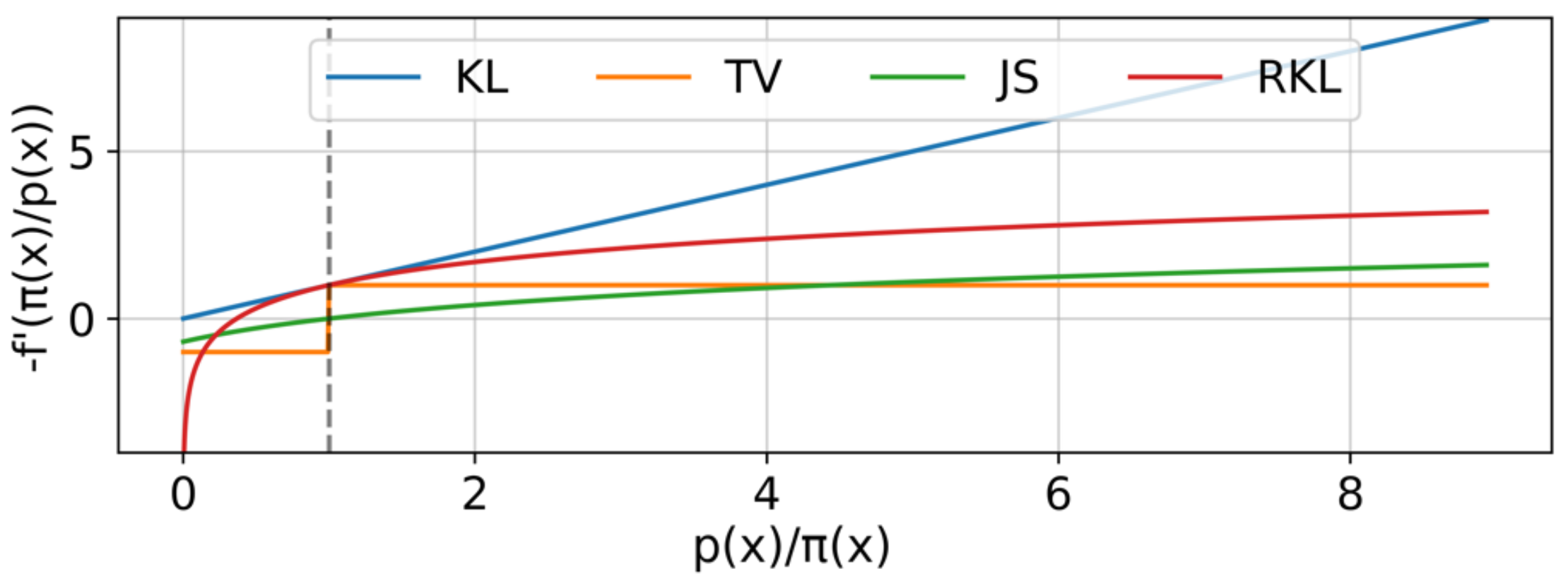}}
\caption{Pseudo-reward{s} for various $f$-divergence{s}. 
The $x$-axis denotes $\frac{p(x)}{\pit(x)}$ and {the} $y$-axis denotes the pseudo-reward. {The dotted line denotes the point where $p(x)=\pit(x)$.}%
}\label{fig:pseudo-reward}
\end{figure}

\vspace*{-2mm}
{A plausible hypothesis would have been}
that each variant of $f$-DPG {is} %
comparatively better 
{at least}
in terms of the $f$-divergence {objective} being optimized. 
{Surprisingly, we found that, save for a few exceptions (Sec.~\ref{sec:scalar}),}
for a given target there is one or a few variants that are the best {across}
all measured divergences.
{Furthermore, we {observed} %
that}
divergence measures can have a significant impact on the performance of the model {depending on the target distribution}.
\ptdyI{Fig.~\ref{fig:pareto_frontier} summarizes the Pareto frontier of the alignment-diversity trade-off of \ptmdI{the} $f$-DPG method. 
The results demonstrate that RKL-DPG and KL-DPG consistently represent two contrasting extremes: RKL-DPG shows high alignment but limited diversity, whereas KL-DPG exhibits low alignment but high diversity. JS-DPG shows a balanced trade-off between alignment and diversity and consistently appeared on the Pareto frontier across all experiments we conducted.
}

Fig. \ref{fig:pseudo-reward} illustrates the {differences between} pseudo-rewards for %
{distinct}
$f$-divergences\ptdyI{, giving a plausible explanation for the observed differences}.
The forward KL loss aims to ensure coverage of the subset where $p(x)>0$, {giving} %
a large pseudo-reward for samples with $p(x)%
{> >}\pi{(x)}$.
However, %
{the optimization}
can be sensitive to sampling noise in {the} finite sample approximation ({see, e.g.,} Sec.~\ref{sec:pointwise}). 
Conversely, the reverse KL loss results in extreme negative rewards for samples with $p(x)
{< <}\pit{(x)}$, leading {$\pit$ to avoid %
{such} regions and resulting in distributional collapse} 
(Sec.~\ref{sec:scalar}).
%
%
%
%
%
\begin{comment}
On the other hand, {the} Jensen-Shannon loss {gives robust rewards in both directions and prevent{s} $\pit$ from heavily relying on a single direction}
(Sec \ref{sec:pointwise}).
Total Variation loss {is also more robust to outliers thanks to the} 
hard-threshold pseudo-reward, but it can lead to {high variance }
when $\pit \approx p$ 
(Sec.~\ref{sec:distributional}).
\end{comment}
Total Variation loss is robust to outliers thanks to its hard-thresholded pseudo-reward, however it can lead to {high variance} behavior when $\pit \approx p$ (Sec.~\ref{sec:distributional}). 
On the other hand, {the} Jensen-Shannon loss {gives smooth and} robust rewards in both directions and prevent{s} $\pit$ from heavily relying on a single direction{, making it a reasonable default choice as confirmed by our experiments}.
{To conclude, we} {propose a flexible framework for {approximating a target distribution} {by minimizing any $f$-divergence, unifying earlier approaches for aligning language models.} 
{Our results on a diverse array of tasks show that minimizing 
{well-chosen}%
$f$-divergences leads to significant gains over previous work.}}
\ptdyI{
The fact that increasing the model size improves the alignment score but does not inherently bridge the gap between objectives\fmdI{(22/05/2023) ``the gap towards the target'' (?)\ptdyI{DY: Here, what I meant is gap between the suboptimal objective and the optimal objective}} underscores the importance of selecting appropriate divergence objectives.}
%
\begin{comment}
{To conclude, we }%
propose a flexible framework for {approximating a target distribution}  {by minimizing any $f$-divergence}
%
that{, in particular,}\fmdI{(23/01/2023) it does not only unify ...} unifies 
DPG {used by GDC} %
and RL with KL penalties {used by RLHF.} %
{Through extensive} {experiments} %
{we} highlight the importance of carefully choosing {an} appropriate divergence function for {a given alignment objective.}%
\end{comment}

%
%

%

%

%
%
%
%
%
%
%
%
%
%
%
%
%

%
%
%
%
%
%
%
%
%

%
%
%
%
%

%
%
%

%
%

%

%
\bibliography{f-DPG-rebibed}
\bibliographystyle{icml2023}

\newpage
\appendix
\onecolumn

\newcommand{\fstar}{{f^{*}}}

\section{Complements on Formal Aspects and Proofs}
\label{app:complements-proofs-formal}

\subsection{Equivalent definitions for $f$-divergences}
\label{app:generalized_divergence}

The definition of $f$-divergences of Eq.~\ref{Eq.:f-divergence} is equivalent to a second definition, in a more ``symmetrical'' format, following \cite{liese2006divergences}, which will help in some derivations, in particular in the proof of Theorem \ref{thm:f-dpg}. 

% % Here we give a general definition of f-divergences between two probability distributions, without restrictions on supports, in a useful ``symmetrical'' format, following \cite{liese2006divergences}. We then show that this definition is equivalent to the definition of Eq.~\ref{Eq.:f-divergence},
% and then
% give examples of common $f$-divergences in the notation of the general definition.
\begin{definition*}[$f$-divergence: ``symmetrical'' format]
The $f$-divergence $D_{f}(p||q)$, where $p$ and
$q$ are distributions over a discrete set $\mathcal{X}$ can be defined as
\begin{align}
D_{f}(p||q)\doteq \sum_{\{x:\ p(x)>0,\ q(x)>0\}}q(x)\ f(\frac{p(x)}{q(x)}) \quad +f(0)\ q(p=0) \quad +\fstar(0)\ p(q=0),
\label{eq:generalized_divergence}
\end{align}
where the generator function $f:(0,\infty) \rightarrow\mathbb{R}$ 
is a convex function satisfying $f(1)=0$.
We denote by $q(p=0)$
the $q$-mass of the set $\{x:p(x)=0\}$, i.e. $q(p=0)=\sum_{\{x:p(x)=0\}}q(x)$
and similarly for $p(q=0)$.  
\end{definition*}
In this definition, the function $\fstar(t)$ is the so-called \emph{perspective transform} of $f$
defined by $\fstar(t)=t\ f(\frac{1}{t})$. It can be shown to  be also a convex function
$\fstar: (0,\infty) \rightarrow\mathbb{R}$ with $\fstar(1)=0$ and
$f^{**}=f$. As we mentioned in the main text, we also have the following important ``swapping'' property: $D_f(p,q)=D_\fstar(q,p)$. 

Following \citet{liese2006divergences, polyanski}, we use the conventions: 
\begin{align}
    &f(0) \doteq \lim_{t\rightarrow0}f(t),\quad \fstar(0) =\lim_{t\rightarrow0}\fstar(t) = \lim_{t\rightarrow0}t\ f(\frac{1}{t}),\\
    &0\ f(0)\doteq 0,\quad 0\ \fstar(0)\doteq 0,\text{\qquad including when } f(0) = \infty \text{ and } \fstar(0)=\infty,\\
    &f'(\infty) \doteq \fstar(0) = \lim_{t\rightarrow0}t\ f(\frac{1}{t}).
\end{align}
For the existence of the limits in these equations, where $f(0)$ and $\fstar(0)$ can take values in $\R \cup \{\infty\}$, as well as for the motivation for defining $f'(\infty) \doteq \lim_{t\rightarrow0}t\ f(\frac{1}{t})$, one may refer to \citep{liese2006divergences} and \citep[\S 2.3]{hiriart2013convex}.

\paragraph{Equivalence of definitions \ref{Eq.:f-divergence} and \ref{eq:generalized_divergence}} In order to prove this equivalence, after noting that $f'(\infty) = \fstar(0)$, it remains to show that 
$\E_{x\sim q}f\left(\frac{p(x)}{q(x)}\right)$
is equal to $\sum_{\{x:\ p(x)>0,\ q(x)>0\}}q(x)\ f(\frac{p(x)}{q(x)}) + f(0)\ q(p=0)$. We have:
\begin{align*}
    \E_{x\sim q}\ f(\frac{p(x)}{q(x)}) &= \sum_{\{x:\ q(x)>0\}} q(x)\ f(\frac{p(x)}{q(x)}) \\
    &= \sum_{\{x:\ q(x) > 0,\ p(x) > 0\}} q(x)\ f(\frac{p(x)}{q(x)})     
    \quad + \sum_{\{x:\ q(x) > 0,\ p(x) = 0\}} q(x)\ f(0)\\
    &= \sum_{\{x:\ q(x) > 0,\ p(x) > 0\}} q(x)\ f(\frac{p(x)}{q(x)})     
    \quad + f(0)\ q(p=0),
\end{align*}
which concludes the proof.

\subsection{Illustrations of a few $f$-divergences}
Let's now see how the notion of $f$-divergence can be applied to a few common cases.

\paragraph*{Forward and reverse KL}

By the standard definition for KL divergence, we have, for $\KL p \pi$, the ``forward KL'' from a model $\pi$ to a target $p$: 
% (often used to characterize the ``forward KL'' from a model $q$ to a target $p$):  
\begin{align}
\KL p \pi = 
            \begin{cases}
			\E_{x\sim p}\ \log\frac{p(x)}{\pi(x)} & \text{if } \Supp(p)\subset \Supp(\pi),\\
            \infty, & \text{otherwise}.
		  \end{cases}
    \label{eq:KL-illustration}
\end{align}

If we take $f(t) = - \log t$, as in Table \ref{tab:f-dpg-loss}, then we have $f(0) = \infty$. On the other hand we see that $\fstar(t) = t\ \log t$ and $\fstar(0) = 0$.
We can then write, using \eqref{eq:generalized_divergence}:
\begin{align*}
    D_{f}(\pi||p)
    &= \sum_{\{x:\ \pi(x)>0,\ p(x)>0\}} - p(x)\ \log(\frac{\pi(x)}{p(x)}) \quad + \infty\ p(\pi=0) + 0\ \pi(p=0)\\
    &= \sum_{\{x:\ \pi(x)>0,\ p(x)>0\}} p(x)\ \log(\frac{p(x)}{\pi(x)}) \quad + \infty\ p(\pi=0),
\end{align*}
where $\infty\ p(\pi=0)$ is null for $\Supp(p)\subset \Supp(\pi)$ and infinite otherwise. Hence $D_{f}(\pi||p) = \KL p \pi$, the forward KL from $\pi$ to $p$.

Now, consider the ``reverse KL'' from $\pi$ to $p$, namely $\KL \pi p$. Based on the previous derivation, and with the same $f(t) = - \log t$ we can write it as $\KL \pi p = D_{f}(p||\pi)$, but using the perspective function $\fstar(t) = t \log t$, we can also write it (as we actually do in Table \ref{tab:f-dpg-loss}) as $D_\fstar(\pi||p) = D_{t\,\log t}(\pi||p)$.

\paragraph*{Total Variation divergence}
The Total Variation divergence between $p$ and $\pi$ is standardly defined as $\TV p \pi = \frac{1}{2}\ \sum_{x\in \mathcal{X}} |p(x)-\pi(x)|$. We then have $\TV p \pi = \TV \pi p$. 
Let's then define $f(t) = \frac{1}{2} |1-t|$. We have $f(0)=1/2$, $\fstar(t)=f(t)$, and $\fstar(0)=1/2$. Then, using \eqref{eq:generalized_divergence}:
\begin{align*}
    D_{f}(\pi||p)
    &= \sum_{\{x:\ \pi(x)>0,\ p(x)>0\}} \frac{1}{2}\ p(x)\ \left|1-\frac{\pi(x)}{p(x)}\right| \quad + \frac{1}{2}\ p(\pi=0) + \frac{1}{2}\ \pi(p=0)\\
    &= \sum_{\{x:\ \pi(x)>0,\ p(x)>0\}} \frac{1}{2}\  \left|p(x)-\pi(x)\right| \quad + \frac{1}{2}\ p(\pi=0) + \frac{1}{2}\ \pi(p=0)\\
    &= \sum_{\{x:\ \pi(x)>0,\ p(x)>0\}} \frac{1}{2}\  \left|p(x)-\pi(x)\right|  
   \quad + \frac{1}{2}\ \sum_{\{x:\ \pi(x)=0,\ p(x)>0\}} \left|p(x)-\pi(x)\right| \\
   &\qquad\qquad\qquad\qquad + \frac{1}{2}\  \sum_{\{x:\ \pi(x)>0,\ p(x)=0\}} \left|p(x)-\pi(x)\right| \\
    &= \frac{1}{2}\ \sum_{x\in \mathcal{X}}  \left|p(x)-\pi(x)\right|, 
\end{align*}
and therefore $\TV p \pi = D_{f}(\pi||p)$, and also $\TV p \pi = \TV \pi p = D_\fstar(p||\pi) = D_f(p||\pi)$.

\subsection{Proof of Theorem \ref{thm:f-dpg}}\label{app:proofs}
We restate the theorem here for convenience.
\begin{thm*}[Theorem \ref{thm:f-dpg}]
Let $p$ and $\pit$ be distributions over a discrete set $\mathcal{X}$ such that at least one of the following  conditions holds: 
(i) $\forall \theta \in \Theta, \Supp(p)\subset \Supp(\pit)$, or 
(ii) $\Supp(\pit)$ does not depend on $\theta$.
Then:
\begin{align}
    \nabt D_{f}(\pit||p)=\E_{x\sim\pit} \left[f^{'}\left(\frac{\pit(x)}{p(x)}\right)\nabt\log\pit(x)\right].
\end{align}
\end{thm*}
\begin{proof}
Based on definition \eqref{eq:generalized_divergence} we have:
\begin{align*}
\nabt &D_{f}(\pit ||p) =\sum_{\{x:p(x)>0,\pit(x)>0\}}\ p(x)\ \nabt f(\frac{\pit(x)}{p(x)})+f'(\infty)\nabt \pit (p=0)+f(0)\nabt p(\pit =0)\\
 & =\sum_{\{x:p(x)>0,\pit(x)>0\}}\ p(x)\ f'(\frac{\pit(x)}{p(x)})\ \nabt \frac{\pit(x)}{p(x)}+f'(\infty)\nabt \pit (p=0)\\
 & =\sum_{\{x:p(x)>0,\pit(x)>0\}}\pit(x)f'(\frac{\pit(x)}{p(x)})\nabt \log\pit(x)+f'(\infty)\nabt \pit (p=0)\\
 & =\sum_{\{x:p(x)>0,\pit(x)>0\}}\pit(x)f'(\frac{\pit(x)}{p(x)})\nabt \log\pit(x)+f'(\infty)\nabt \left[\sum_{\{x:p(x)=0,\pit(x)>0\}}\pit(x)\right]\\
 & =\sum_{\{x:p(x)>0,\pit(x)>0\}}\pit(x)f'(\frac{\pit(x)}{p(x)})\nabt \log\pit(x)+\sum_{\{x:p(x)=0,\pit(x)>0\}}\pit(x)f'(\infty)\nabt \log\pit(x)\\
 & =\sum_{\{x:\pit(x)>0\}}\pit(x)f'(\frac{\pit(x)}{p(x)})\nabt \log\pit(x)\\
 &= \E_{x\sim \pit}\ f'(\frac{\pit(x)}{p(x)})\ \nabt \log\pit(x).
\end{align*}
In the first line of this derivation, we use the previously introduced notation $f'(\infty) \doteq \fstar(0)$, employed in particular by \cite{polyanski}, which is motivated by the fact that $\lim_{t\rightarrow\infty}f^{'}(t)=\lim_{t\rightarrow\infty}\frac{1}{t}f(t)=\fstar(0)$ (See \citep{hiriart2013convex}). 
In the second line, we employ a variant of the chain-rule for derivatives of multivariate functions. We also exploit the fact that the condition (i) stating that the support of $p$ is contained in the support of $\pit $ for all $\theta \in \Theta$ implies that $\nabt p(\pit =0)=\nabt 0=0$, and that the condition (ii) that the support of $\pit $ does not depend on $\theta$ also implies that $\nabt p(\pit =0)=0$.
In the fourth line, we write $\pit (p=0)$ as a sum.
In the sixth line, we allow the notation $f'(\frac{\pit(x)}{p(x)})$ instead of $f'(\infty)$ when $p(x)=0$ and $\pit(x)>0$.
\end{proof}

\paragraph{Working with the opposite divergence $D_f(p||\pit)$} In case one may prefer to work with a divergence $D_f(p||\pit)$ having the opposite argument order, then one can use the identity $D_f(p||\pit) = D_{\fstar}(\pit||p)$ to conclude that under the exact same conditions (i) or (ii) as previously, we have: 
\[
\nabt D_{f}(p||\pit) = \nabt D_{\fstar}(\pit||p)  =  \E_{x\sim\pit }\left[\fstar^{'}\!\!\left(\frac{\pit(x)}{p(x)}\right)\ \nabt \log\pit(x)\right],
\]
where the derivative is applied to the perspective transform of $f$.
%\fmdI{(16/01/2023) I found your (DY's) second proof using the opposite direction (``Theorem \ref{thm:f-dpg} with $D_f(p||\pit)$'') quite interesting as a confirmation of the first one, but I am not sure whether we need two proofs, and also IMO is somewhat less intuitive than the first. The result with the opposite direction is however useful, and is stated above. As a minor point about the second proof, which is probably possible to fix, at one point you allow the expression $f'(\frac{p(x)}{\pit(x)})$ with $\pit(x)>0, p(x) \ge 0$, therefore the expression $f'(0)$, but this derivative is not defined (although it might be as a limit, but it would need some care, and I am not sure it would be what we need).} % Got it! I agree with removing unneccesary proof

\subsection{About non-differentiability of $f$}
\label{app: non-differentiability}
In practice when sampling from $\pit$ in Eq~\eqref{Eq.:fdpg-gradient}, the problem of non-differentiability can be neglected, and recourse to subgradients is typically unnecessary, even for $f$'s that have non-differentiability points (such as e.g. the generator $f(t)= 0.5 |1-t|$ for the Total Variation divergence). Indeed, let $T_{nd} \doteq \{t: f(t) \text{ is non differentiable at } t\}$, and let $\Theta_{nd} \doteq \{\theta: \exists x \in \mathcal{X}:\frac{\pit(x)}{p(x)} \in T_{nd}\}$ be the set of $\theta$'s for which $f^{'}\left(\frac{\pi_{\theta}(x)}{p(x)}\right)$ is undefined on at least one $x$. Then $\Theta_{nd}\subset \R^d$ (with $d$ the parameter dimension) is the countable union of countable sets, hence is countable, and therefore of null measure inside $\R^d$. This means that, almost surely over $\theta$, the RHS of Eq~\ref{Eq.:fdpg-gradient} is well-defined for all $x$'s.

\subsection{$f$-DPG algorithm}
%\fmdI{(17/01/2023) This algorithm needs a few more comments. Note that I replaced the previous (correct) version of the moving average with a more standard one.} 

\begin{algorithm}[ht]
   \caption{$f$-DPG}
   \label{alg:f_dgp}
\begin{algorithmic}
   \STATE {\bfseries Input:}  unnormalized target distribution $P(\cdot)$, initial model $a(\cdot)$, $D_f$ generator $f(\cdot)$
   % \FOR{$i=1$ {\bfseries to} $m-1$}
   \STATE {\bfseries Initialize:} $\pit(\cdot) \gets a(\cdot)$, $Z \gets 0$, $N \gets 0$
   \qquad \COMMENT{initialize model $\pit$, partition $Z$, sample size $N$ for moving average}
   \FOR{each iteration}
   \FOR{each episode}
   \STATE sample $x$ from $\pit(\cdot)$
   \STATE $N \gets N+1$
   %\STATE $Z \gets Z+\frac{P(x)/\pit(x)-Z}{N}$
   %\STATE $Z \gets \frad{(N-1)\ Z \ +\ \frac{(P(x)/\pit(x))}{N}$
   \STATE \ptmdI{$Z \gets \frac{(N-1) Z \ +\ (P(x)/\pit(x))}{N}$}
   \qquad \COMMENT{Estimate $Z$ with historical samples, using a moving average}
   \STATE $p(\cdot) \gets P(\cdot)/Z$
   \STATE $\theta \gets \theta + \alpha^{(\theta)}f^{'}\left(\frac{\pit(x)}{p(x)}\right)\nabt\log\pit\ptmdI{(x)}$
   \qquad \COMMENT{Update $\pit$ according to Thm. \ref{thm:f-dpg}}
   \ENDFOR
   \ENDFOR
   \STATE {\bfseries Output:}  $\pit$
\end{algorithmic}
\end{algorithm}

\subsection{Baseline: alternative derivation}
\label{sec:baseline-alternative-derivation}

The generator function is not uniquely determined for a given $f$-divergence: 
\begin{fact}
For generators $f,\ g$ such that $f(t)=g(t)+c(t-1),\ c\in\mathbb{R}$, $D_{f}(p_1||p_2)=D_{g}(p_1||p_2)$.
\end{fact}

We provide here an alternative way to introducing baselines, based on a change of generator.
\begin{thm*}[Baseline based on change of generator]
If $D_{f}(\pit ||p)$ is a divergence with any generator $f$, and $B \in \R$, 
there exists a generator $g$ with the same divergence $D_{f}(\pit ||p)=D_{g}(\pit ||p)$ such that
\begin{align*}
\nabt D_{g}(\pit ||p) & =\E_{x\sim\pit }\left[\left(f^{'}\left(\frac{\pit(x)}{p(x)}\right)-B\right)\nabt \log\pit(x)\right]\\
 & =\nabt D_{f}(\pit ||p).
\end{align*}
\end{thm*}

\begin{proof}
Recall that %the facts of $f$-divergence that 
$D_{f}(\pit ||p)=D_{g}(\pit ||p)$
%if and only if exists $b\in\mathbb{R}$ such that 
when
$g(x)=f(x)-B(x-1)$. %,$ with arbitrary $b\in\mathbb{R}$.
Therefore, $\nabt D_{f}(\pit ||p)=\nabt D_{g}(\pit ||p)$
with $g^{'}\left(\frac{\pi_\theta(x)}{p(x)}\right)=f^{'}\left(\frac{\pi_\theta(x)}{p(x)}\right)-B$. 
\end{proof}

\section{Extended Related Work}

\paragraph{RL for LMs}
There is a large reinforcement learning inspired literature about steering an autoregressive sequential model towards optimizing some global reward over the generated text. This includes REINFORCE \cite{williams1992simple} for Machine Translation \cite{RanzatoCAZ15}, actor critic for Abstractive Summarization \cite{PaulusXS18}, Image-to-Text \cite{liu2016not}, Dialogue Generation \cite{LiMRJGG16}, and Video Captioning \cite{PasunuruB17}. With respect to rewards, some approaches for Machine Translation and Summarization \cite{RanzatoCAZ15, BahdanauBXGLPCB17} directly optimize end task rewards such as BLEU and ROUGE at training time to compensate for the mismatch between the perplexity-based training of the initial model and the evaluation metrics used at test time. Some others use heuristic rewards as in \cite{LiMRJGG16, tambwekar2019controllable}, in order to improve certain a priori desirable features of generated stories or dialogues.

Several studies, have considered incorporating a distributional term inside the reward to be maximized. In particular \citet{KL_Jaques17, KL_jaquesK19, ziegler2019fine, stiennon2020} have applied variations of KL-control \cite{todorov2006linearly, KappenGO13} which adds a penalty term to the reward term so that the resulting policy does not deviate too much from the original one in terms of KL-divergence. The overall objective with the KL-penalty is maximized using an RL algorithm of choice including: PPO \cite{schulman2017proximal} as in \citet{ziegler2019fine} or Q-learning  \cite{MnihKSGAWR13} as in \citet{KL_Jaques17}. This approach recently get a huge attention with its impact with using the human data to train aligned language models in LaMDA \cite{LaMDA}, InstructGPT \cite{ouyang2022training}, Sparrow \cite{glaese2022improving}, and CAI \cite{bai_constitutional}. Similar work involving model self-critique and natural language feedback includes \cite{zhao2021ethical, scheurer2022, saunders2022self}

\paragraph{$f$-divergence objectives for generative models}

In the literature, there have been several studies exploring the use of $f$-divergences in generative models. \citet{Goodfellow} introduced the concept of GANs and their connection to the Jensen-Shannon divergence. \citet{nowozin2016f} proposed a variational expression of $f$-divergences as a loss function for GANs. Theoretical insight on the relationship between divergence choice and the convergence of probability distributions was provided by \citet{ArjovskyCB17}.  Additionally, \citet{TheisOB15} discussed potential drawbacks of forward KL divergence in generative models and \citet{Huszar15} proposed a generalization of Jensen-Shannon divergence that interpolates between KL and reverse KL and has Jensen-Shannon as its midpoint. 

The connections between RL and divergence minimization have also been explored, with studies showing that entropy regularization in RL can be viewed as minimizing reverse KL divergence between reward-weighted trajectory and policy trajectory distributions \cite{KappenGO13, Sergey}. Other studies have also explored the use of forward KL divergence in RL \cite{PetersS07, NorouziBCJSWS16}. Additionally, a unified probabilistic perspective on $f$-divergence minimization in imitation learning has been presented for both discrete and continuous control environments \cite{KeC00LS21,ghasemipour20a}.

\ptdyI{\citet{wang2018variational} introduced \ptmdI{variational inference with adaptive $f$-divergences and demonstrated its effectiveness in RL, with focus on continuous sample spaces. Their Proposition 4.2.1 is similar to our theorem \ref{thm:f-dpg}. However, our result} exhibits greater generality by defining $D_f(\pi_\theta||p)$ without requirements of absolute continuity in either direction \cite{polyanski, liese2006divergences}. We note that this generalization is crucial for LM alignment, as the case of $p(x)=0,\ \pit(x)>0$ \ptmdI{can} easily occur.}\fmdI{(22/05/2023) I reformulated a bit this paragraph.}

\section{Implementation Details}\label{app:implement_details}

All models were implemented using PyTorch \cite{paszke2019pytorch} and HuggingFace Transformers \cite{wolf2020transformers} with the Adam optimizer \cite{kingma2014adam}. Training was performed on Nvidia V100 GPU, with the longest run taking approximately 2 days. Hyperparameter details are listed in Tab. \ref{tab:implement_details}. Pretrained models are available on the Huggingface Model Hub under the specified model names. 
\ptdyI{We focused on searching for hyperparameters based on KL-DPG, which served as the baseline method we aimed to improve upon, providing it with an initial advantage. To ensure that all methods were evaluated under comparable settings, we tuned the hyperparameters once for all $f$-DPG methods.}

\begin{table}
\begin{centering}
\begin{tabu}{cl}
\tabucline[1.5pt]{-}
Experiment & Hyperparameters\tabularnewline
\tabucline[1.5pt]{-}
Common & 
\begin{tabular}{@{}l@{}}
{\fontfamily{qcr}\selectfont
batch size = 258, optimizer = Adam,}\\
{\fontfamily{qcr}\selectfont
learning rate schedule = constant with warmup (100 epochs)}
\end{tabular}
\tabularnewline

\hline 
Sentiment preference & 
\begin{tabular}{@{}l@{}}
{\fontfamily{qcr}\selectfont
original model = gpt2, learning rate = $1\times10^{-5}$}\\
{\fontfamily{qcr}\selectfont
maximum length = 40, batch size = 2048, total epochs=1000}
\end{tabular}
\tabularnewline

\hline 
Lexical(RLKL)  & 
\begin{tabular}{@{}l@{}}
{\fontfamily{qcr}\selectfont
original model = gpt2, learning rate = $1\times10^{-5}$,}\\
{\fontfamily{qcr}\selectfont
maximum length = 40, total epochs=5000}
\end{tabular}
 \tabularnewline

 \hline 
Lexical(GDC)  & 
\begin{tabular}{@{}l@{}}
{\fontfamily{qcr}\selectfont
original model = gpt2, learning rate = $1.41\times10^{-5}$,}\\
{\fontfamily{qcr}\selectfont
maximum length = 40, total epochs=5000}
\end{tabular}
 \tabularnewline

\hline 
Female50\% Science100\% & 
\begin{tabular}{@{}l@{}}
{\fontfamily{qcr}\selectfont
original model = mkhalifa/gpt2-biographies, }\\
{\fontfamily{qcr}\selectfont
learning rate = $1.41\times10^{-5}$, maximum length = 40, }\\
{\fontfamily{qcr}\selectfont
total epochs=1000}\end{tabular}
\tabularnewline

\hline 
Regard balancing & 
\begin{tabular}{@{}l@{}}
{\fontfamily{qcr}\selectfont
original model = gpt2, learning rate = $5\times10^{-6}$,}\\
{\fontfamily{qcr}\selectfont
maximum length = 40, batch size = 2048, total epochs=1000}
\end{tabular}
\tabularnewline

\hline
Summarization & 
\begin{tabular}{@{}l@{}}
{\fontfamily{qcr}\selectfont
original mode=t5-small, learning rate = $1\times10^{-4}$,}\\
{\fontfamily{qcr}\selectfont
maximum length = 128, total epochs=2000}
\end{tabular}
\tabularnewline

\hline 
Code generation & 
\begin{tabular}{@{}l@{}}
{\fontfamily{qcr}\selectfont
original mode=gpt-neo-125M, learning rate = $1\times10^{-4}$,}\\
{\fontfamily{qcr}\selectfont
maximum length = 128, total epochs=2000}
\end{tabular}
\tabularnewline

\hline 
GPT2 approximation & 
\begin{tabular}{@{}l@{}}
{\fontfamily{qcr}\selectfont
original model = lvwerra/gpt2-imdb, learning rate = $5\times10^{-6}$}\\
{\fontfamily{qcr}\selectfont
maximum length = 40, total epochs=8000}
\end{tabular}
\tabularnewline

\tabucline[1.5pt]{-}
\end{tabu}
\par\end{centering}
\caption{Hyperparameters used throughout all experiments}\label{tab:implement_details}
\end{table}

\section{Additional Experiments}
\subsection{Generation Quality}\label{app:gen_quality}
\paragraph{Metrics} To see if different objective affects the quality of the generated sentences, we report the following metrics on experiment in Sec.~\ref{sec:scalar}, Sec.~\ref{sec:pointwise}.
\vspace{-5px}
\begin{enumerate}
\itemsep0em 
\item Distinct-n~\citep{li2015diversity}, a measure of text diversity in terms of the frequency of {repeated n-grams }%
within a single sample $x$.
\item Self-BLEU-n~\citep{zhu2018texygen}, a measure of text diversity on a distributional level
across samples.
\item Perplexity, a measure of text fluency with exponentiation of the negative average per-token log-probability under a language model. We use a separate model Distil-GPT-2 \cite{wolf2020transformers} to calculate perplexity to avoid inflated estimates \cite{liu2016not}.
\end{enumerate}

\paragraph{Results}
Tab. \ref{tab:scalar_gen_quality} provides additional metrics for the generated sentences and their diversity on scalar preferences. The notably low entropy and high Self-BLEU of RKL-DPG again indicate low diversity of RKL-DPG at the distributional level, {whereas} other $f$-DPGs have similar values to each other. On the other hand, in quality for individual samples as measured by the perplexity metric, RKL-DPG shows better quality,
which suggests that RKL-DPG captures a subset of the target distribution, an observation that is frequently discussed in other generative models \cite{Huszar15,CheLJBL17,mescheder2018training}.
We provide metrics for the generated sentences aggregated on lexical constraint in Tab. \ref{tab:gen_quality_pointwise}.  
We found no significant difference in diversity among the generated sentences.

\begin{table}
    \begin{centering}
    {}%
    \begin{tabular}{ccccc}
    \hline 
     {{L}oss} & {Entropy} & {Self-BLEU-5} & {Dist-1} & {Perplexity}\tabularnewline
     \hline 
    {KL} & {159.09 (9.58)} & {0.62 (0.01)} & {0.88 (0.01)} & {58.87 (7.48)}\tabularnewline
    {TV} & {157.60 (8.91)} & {0.65 (0.01)} & {0.88 (0.01)} & {59.48 (5.25)}\tabularnewline
    {JS} & {158.04 (8.62)} & {0.64 (0.01)} & {0.88 (0.01)} & {59.67 (6.23)}\tabularnewline
    {RKL} & {151.04 (7.99)} & {0.70 (0.01)} & {0.87 (0.01)} & {53.15 (4.14)}\tabularnewline
    \hline 
    \end{tabular}
    \par\end{centering}
\caption{Quality of the generated text metrics~{for the experiment on scalar preferences (Sec.~\ref{sec:scalar})}.  entropy (↑ better), Self-BLEU-5 (↓ better), Distinct-1 (↑ better), and Perplexity (↓ better).}\label{tab:scalar_gen_quality}
\end{table}

\begin{table}[H]
    \begin{centering}
    {}%
    \begin{tabular}{ccccc}
    \hline 
     {Loss} & {$\E\left[b(x)\right]$} & {Self-BLEU-5} & {Dist-1} & {Perplexity}\tabularnewline
     \hline 
    {KL} & {0.45 (0.09)} & {0.66 (0.02)} & {0.96 (0.00)} & {90.59 (11.74)}\tabularnewline
    {TV} & {0.60 (0.12)} & {0.67 (0.01)} & {0.96 (0.01)} & {80.52 (8.79)}\tabularnewline
    {JS} & {0.66 (0.14)} & {0.67 (0.01)} & {0.95 (0.01)} & {79.53 (8.80)}\tabularnewline
    {RKL} & {0.60 (0.20)} & {0.66 (0.02)} & {0.95 (0.01)} & {79.49 (7.79)}\tabularnewline
    \hline 
    \end{tabular}
    \par\end{centering}
    \caption{Quality of the generated text metrics for the experiment on lexical constraint (Sec.~\ref{sec:pointwise}).  $\E_{\pit}[b(x)]$ (↑ better), Self-BLEU-5 (↓ better), Distinct-1 (↑ better), and Perplexity (↓ better).}\label{tab:gen_quality_pointwise}
    \end{table}

\section{$f$-DPG on Conditional Target Distributions}\label{app:conditional}
Let $C$ be a discrete (potentially infinite) set of conditions $c$.
The problem of fine-tuning a pretrained model $a(x|c)$ to satisfy
a control objective (e.g. generating factually correct summaries)
can be seen as a constraint satisfaction problem: finding a model
$p_{c}(x)$ that meets the demands of the control objective but at
the same time stays as close as possible to the original pretrained
model $a(x|c)$. A control objective can be defined in terms of a
binary scorer $b(x,c)$ such that $b(x,c)=1$ if a sample $(c,x)$
satisfies a constraint given by a control objective (e.g. $x$ is
factually correct with respect to $c$) and $b(x,c)=0$ otherwise. 

For each $c\in C$, we can frame the problem of finding the unique
model $p_{c}(x)$ such that (i) $b(x,c)=1$ for all samples $x\sim p_{c}(x)$,
and (ii) $p_{c}(\cdot)$ has minimal KL divergence from $a(\cdot|c)$
as an instance of the unconditional case already considered by \citet{khalifa2021distributional}. Following our example, $p_{c}$ could be a distribution
over factually correct summaries of $c$ as similar as possible to
a distribution over summaries which the original model $a$ would
produce for a document $c$. Therefore, $p_{c}$ can be represented
as {a} distribution $p_{c}(x)$ of the following form:
\[
p_{c}(x)=1/Z_{c}a(x|c)b(x,c).
\]

Let $\mathcal{P}$ a conditional distribution over $C$ {which} is defined
as a function from $C$ to the set of unconditional distribution{s} $p_{c}$
over $\mathcal{X}$. While $\mathcal{P}$ represents the target conditional
model optimally reconciling distance from $a(x|c)$ and the control
objective, direct use of $\mathcal{P}$ for sampling is intractable
for two reasons. First, $\mathcal{P}$ actually represents a potentially
infinite collection of unconditional models of the form $p_{c}(\cdot)$.
Second, each of these unconditional models still cannot be easily
sampled from because 
{it does} not admit an autoregressive factorization.
To address this problem, \citet{korbak2021controlling} instead try to find a generative
model $\pit$ approximating $p$ on average across context{s}
by minimizing the expected $\KL{p_{c}}{\pit}$ or equivalently
expected cross-entropy $\text{CE}(p_{c},\pit)$ between $\pit$
and multiple $p_{c}$’s:

\[
\E_{c\sim\tau(c)}\left[\text{CE}(p_{c}(\cdot),\pit(\cdot|c))\right],
\]
with its gradient taking the following form:
\begin{align*}
\E_{c\sim\tau(c)}\left[\nabla_{\theta}\text{CE}(p_{c}(\cdot),\pit(\cdot|c))\right] & =\E_{c\sim\tau(c)}\left[\E_{x\sim p_{c}(x)}\left[\nabla_{\theta}\log\pit(x|c)\right]\right]\\
 & =\E_{c\sim\tau(c)}\left[\E_{x\sim\pit(x|c)}\left[\frac{p_{c}(x)}{\pit(x|c)}\nabla_{\theta}\log\pit(x|c)\right]\right].
\end{align*}
This can be seen as a conditional extension of Eq. \ref{eq:dpg}. 

A natural extension of this objective for $f$-DPG
{is}%
$\E_{c\sim\tau(c)}\left[D_{f}(\pit(\cdotp|c)||p_{c}(\cdotp))\right]$,
an extension that 
{includes}
expected $\KL{p_{c}}{\pit{(\cdot|c)}}$.
Thm.~\ref{thm:f-dpg} implies that {the} gradient of this objective takes the following
form:
\[
\E_{c\sim\tau(c)}\left[\nabla_{\theta}D_{f}(\pit(\cdotp|c)||p_{c}(\cdotp))\right]=\E_{c\sim\tau(c)}\left[\E_{x\sim{\pit(x|c)}}\left[f^{'}\left(\frac{\pit(x{|c})}{p_{{c}}(x)}\right)\nabla_{\theta}\log\pit(x{|c})\right]\right]{.}
\]

{
\subsection{Additional Conditional Preferences Experiments and Details}
\label{app:conditional-additional-info}
}
\paragraph{Task}
{Here, we also} {%
evaluate the conditional task on code generation. For that, we condition on Python {function signatures in the} Python150 dataset \cite{raychev2016probabilistic} which consists of Python source code obtained from GitHub}. We again split {disjoint} train/test set{s of function signatures} and set $\tau(c)$ as a uniform distribution. %
With given prompt $c$, we check compilability of a Python function {definition} %
obtained by concatenating $[c, x]$ and trying to execute it. $b(x,c) = 0$ iff the Python interpreter raises an exception.  For the initial model we use GPT-Neo-125, a variant of GPT-Neo \cite{black2021gpt} on Hugging-face Transformers \cite{wolf2020transformers}.

\paragraph{Metrics for summarization} In addition to the divergences, we evaluate the quality and factual consistency of generated summaries using the following metrics:
\vspace{-5px}
\begin{enumerate}
\itemsep0em 
\item Precision-source \cite{nan2021entity}, defined as 
$\left[|\text{NER}(x)\cap\text{NER}(c)\right|]/|\text{NER}(x)|$
, the percentage
of named entities in the summary that can be found
in the source. Low precision-source indicates severe
hallucination.
\item Recall-target \cite{nan2021entity}, defined as $\left[|\text{NER}(x)\cap\text{NER}(c)\right|]/|\text{NER}(t)|$, the percentage of named entities in the target summary $t$ that can be found in the generated summary $x$.
\item Rouge \cite{lin-2004-rouge}, a measure of summarization quality in terms of unigram overlap between the
source document and ground truth summary.
\end{enumerate}

\paragraph{Metrics for code generation}  We evaluate the quality of generated Python
functions using the following metrics:
\vspace{-5px}
\begin{enumerate}
\itemsep0em 
\item PEP8 error count, the average number of violations of
PEP8.
\item Compilability, the fraction of samples $[c, x]$ that compile.
\end{enumerate}

\paragraph{Results}
Fig. \ref{fig:codegen} presents the evolution of metrics in Code generation. Consistent with the result on factual summarization, $f$-DPG {increases} the fraction of compilable functions, {while decreasing} 
the average number of PEP8 violations. Again, JS-DPG leads to better convergence to $p$ than {KL-DPG used} in \citet{korbak2021controlling}.

\begin{figure}[H]%
\centering
\centerline{\includegraphics[width=\columnwidth]{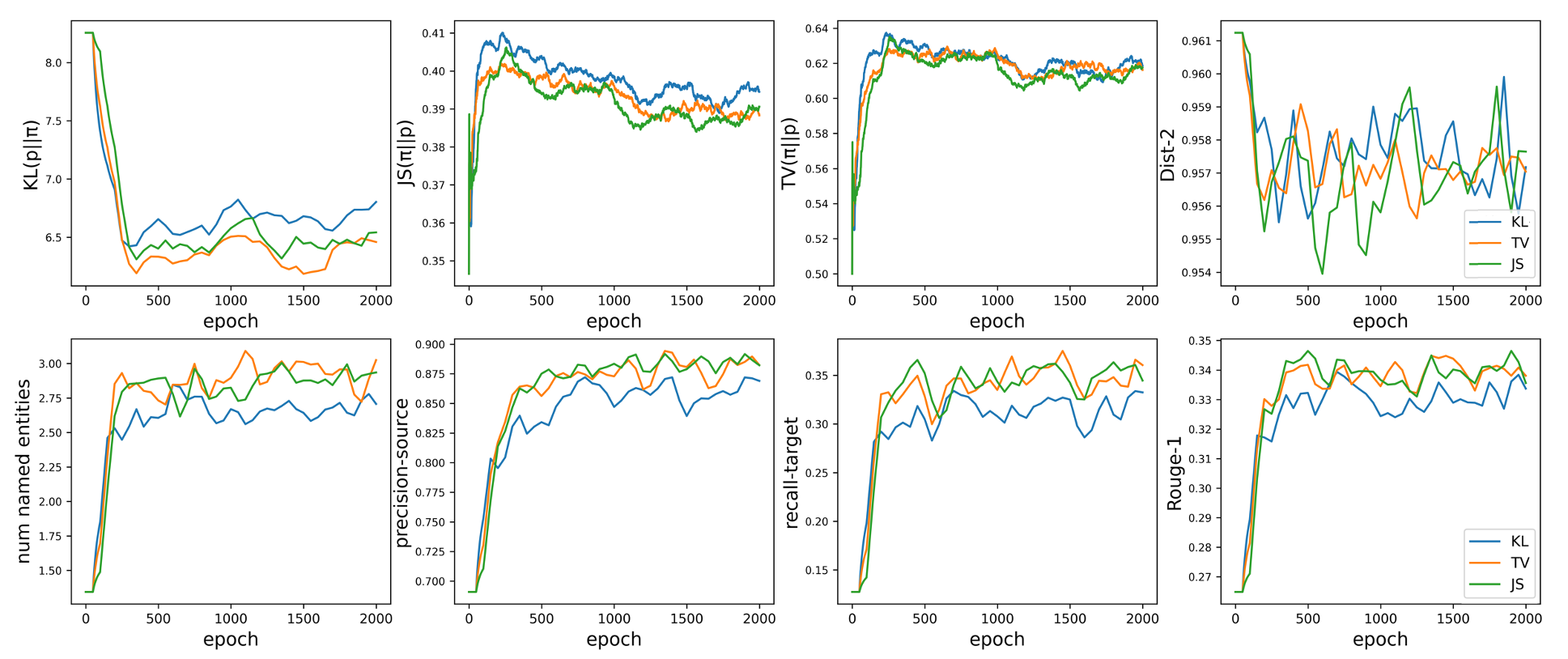}}
\caption{Summarization}\label{fig:summarization}
\end{figure}

\begin{figure}[H]%
\centering
\centerline{\includegraphics[width=\columnwidth]{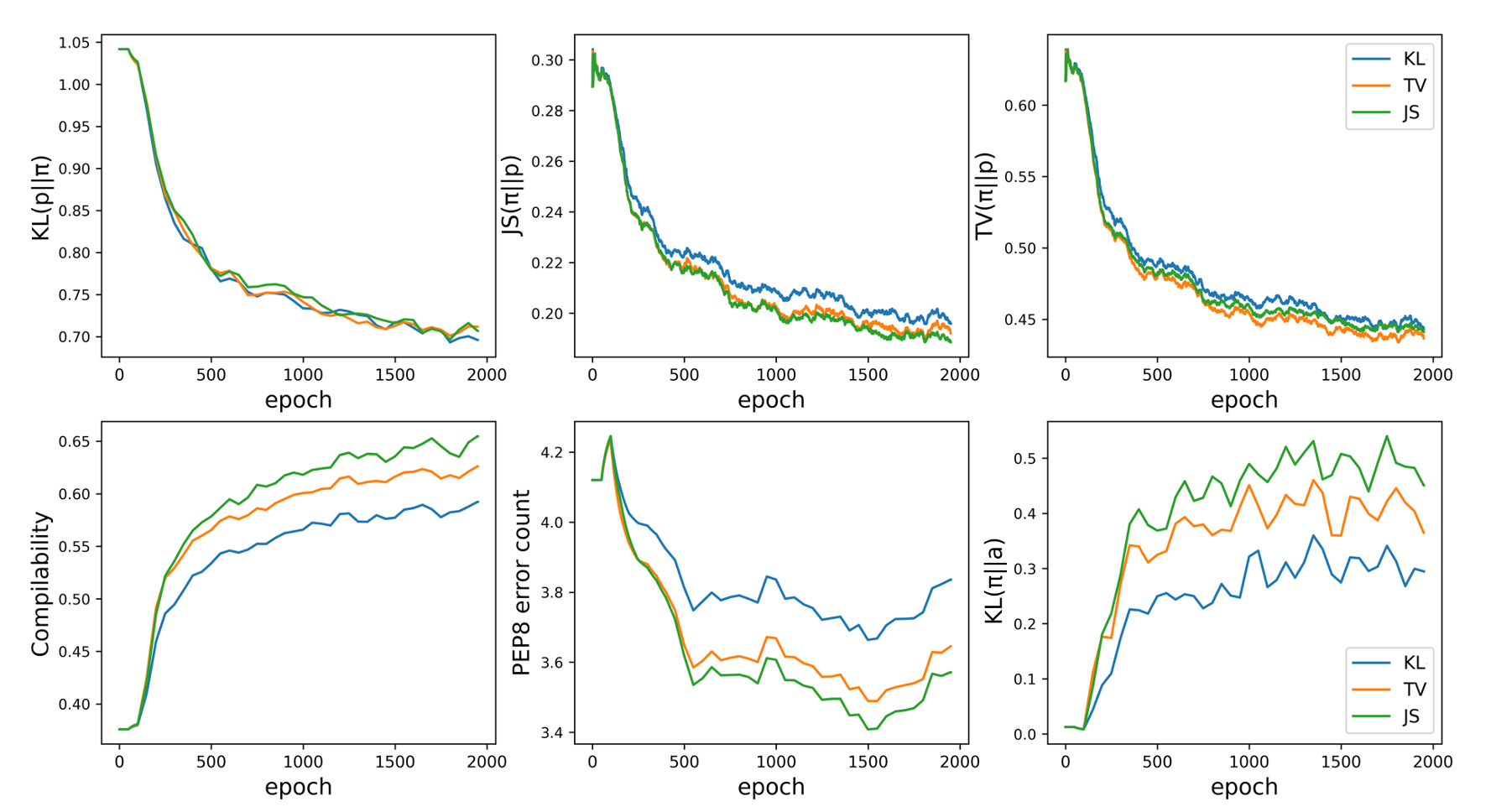}}
\caption{Code generation}\label{fig:codegen}
\end{figure}

\section{Optimal Reward Model for a Decision Maker with a Categorical Distribution}
\label{app:choice_model}

Let’s assume we have a dataset $\mathcal{D}$ containing $M$ tuples $(x_1, ..., x_n)$ of samples and a choice function $h(x_1, ..., x_n) \in \{0,1\}^n$ that returns a one-hot vector to signal the preferred sample. The reward model $r$ in RLHF is trained by first defining a discrete choice model $f_r$ parametrized by the reward model we want to learn:
$$
f_r(x_1, ..., x_n) = \mathrm{softmax}(r(x_1), ..., r(x_n))
$$
and then learning the reward model by minimizing the loss
\begin{align}
\mathrm{loss}(r) &= \mathbb{E}_{(x_1, ..., x_n) \sim \mathcal{D}} ~\mathrm{CE}(h, f_r) \label{eq:rhlf_reward_loss}\\
&= - \mathbb{E}_{(x_1, ..., x_n) \sim \mathcal{D}} ~ h(x_1, ..., x_n) \cdot\log f_r\left( x_1, ..., x_n\right),
\end{align}
Thus, the optimal reward model is given by the function $r$ such that $h(x_1, ..., x_n) = f_r(x_1, ..., x_n)$ as it minimizes the CE in Eq. \ref{eq:rhlf_reward_loss}. 
Typically, $h$ corresponds to the preferences elicited by human annotators. However, let's make a simplifying assumption that humans make choices according to an internal scoring function $\phi(x)$ so that $h_\phi(x_1, ..._, x_n) \sim \mathrm{Categorical}(\phi(x_1), ..., \phi(x_n)),$ or in other words, 
$$
h_\phi(x_1, ..., x_n) = 1 \mathrm{~at~index~}i \mathrm{~with~probability~}\frac{\phi(x_i)}{\sum_{j=1}^n{\phi(x_j)}}.
$$

Now, let’s suppose we have access to $\phi$. Then, we note that if we set 
$$
r_\phi(x) = \log\phi(x),
$$
we get 
\begin{align}
f_{r_\phi}(x_1, ..., x_n) &= \mathrm{softmax}(\log (\phi(x_1)), ..., \log (\phi(x_n)))\\
&= \mathrm{categorical}(\phi(x_1), ..., \phi(x_n)),
\end{align}

and thus, $r_\phi$ is an optimal reward model for $h_\phi$.

\section{Additional Figures}\label{app:additional_figures}

\begin{figure}[H]%
\centering	
\begin{tabular}{cc}
\centerline{\includegraphics[width=\columnwidth]{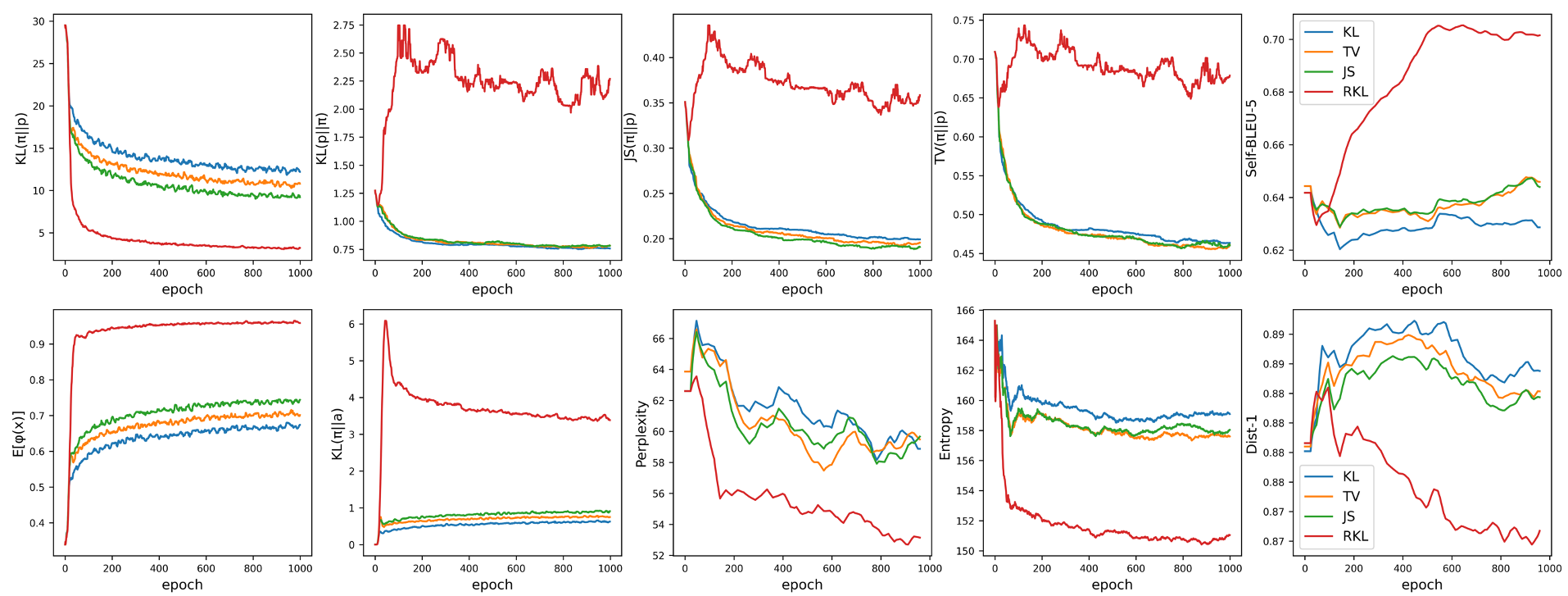}}
\end{tabular}
\caption{Evaluation of metrics in sentiment preference}
\label{fig:scalar_preference_all}
\end{figure}

\begin{figure}[H]
\centering
\includegraphics[scale=0.9]{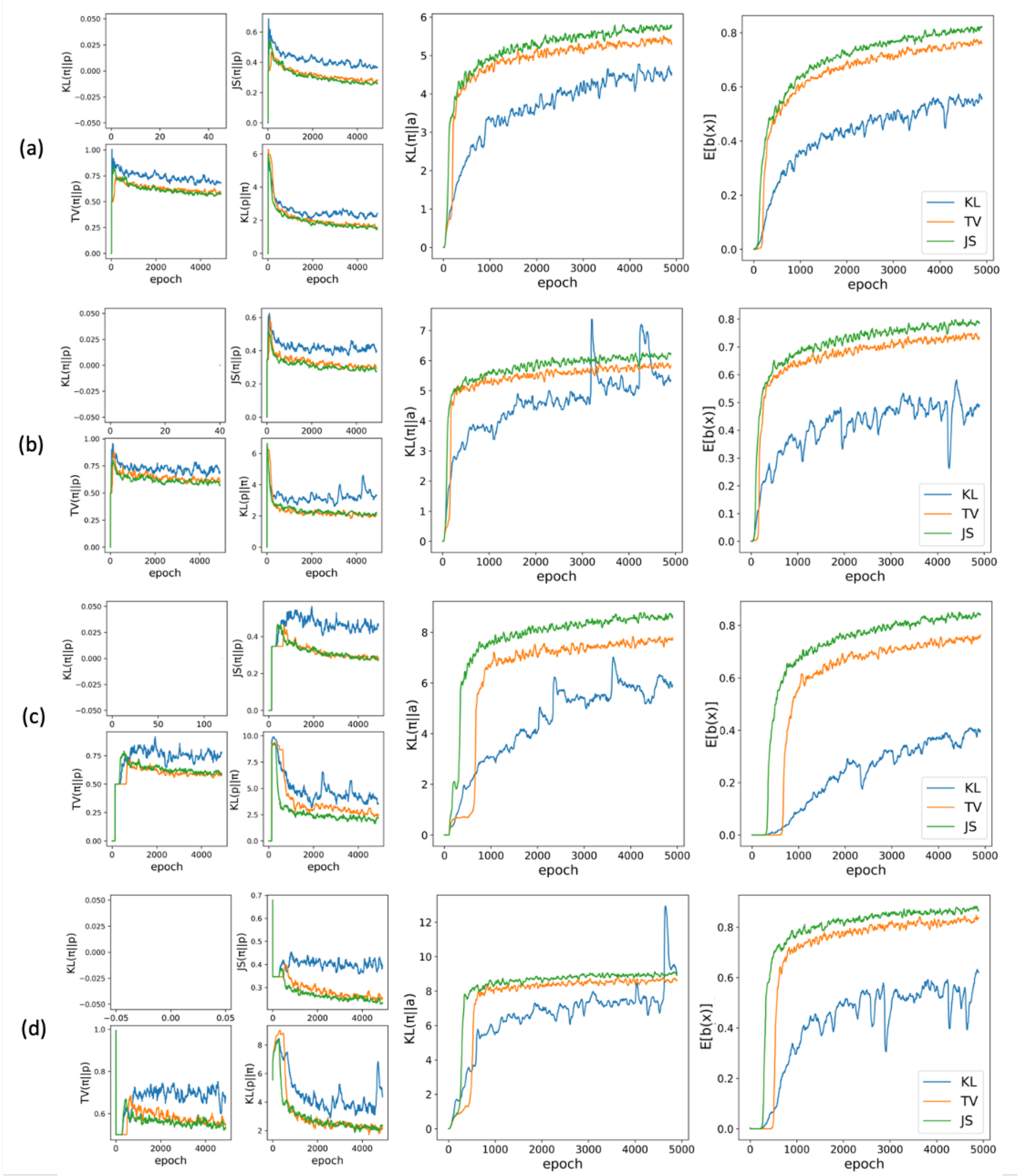} 
\caption{Evaluation metrics: four $f$-divergences $D_f(\pi_\theta||p)$ (↓ better), $\E_{\pi_\theta}[\phi(x)]$ (↑ better), $\text{KL}(\pi_\theta||a)$ (↓ better) with target distribution induced from GDC framework to constrain the existence of single word, (a) amazing, (b) restaurant, (c) amusing, (d) Wikileaks. Note that reverse KL cannot be defined in this case in which $p(x)=0$ for some points}\label{fig:single_gdc_EBM}
\end{figure}

\begin{figure}[H]
\centering
\includegraphics[scale=0.9]{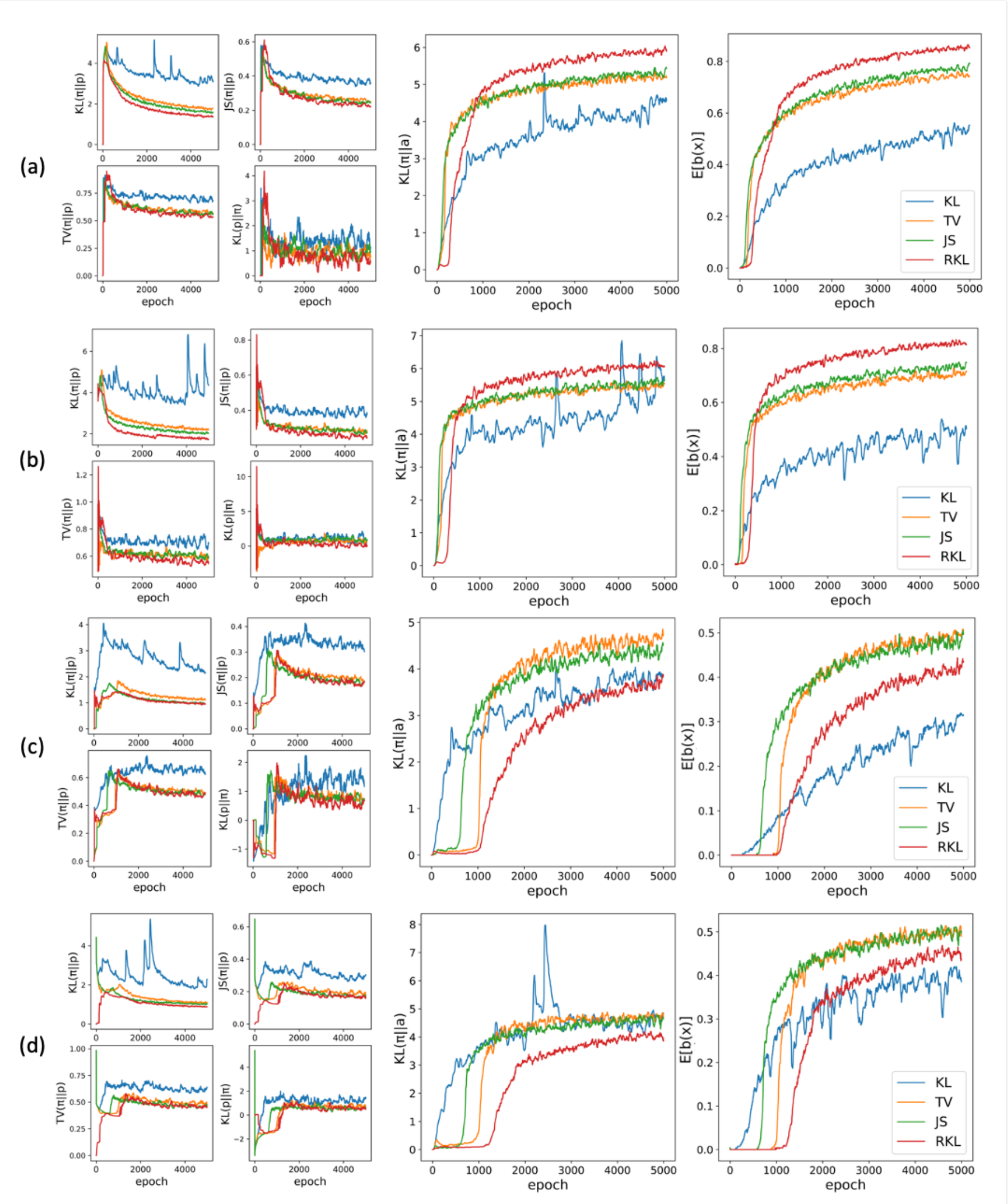} 
\caption{Evaluation metrics: four $f$-divergences $D_f(\pi_\theta||p)$ (↓ better), $\E_{\pi_\theta}[\phi(x)]$ (↑ better), $\text{KL}(\pi_\theta||a)$ (↓ better) 
 with target distribution $\pRL$ to constrain the existence of single word, (a) amazing, (b) restaurant, (c) amusing, (d) Wikileaks.}\label{fig:single_klc_EBM}
\end{figure}

\begin{figure}[H]%
\centering	
\begin{tabular}{cc}
(a) \\ 
\includegraphics[width=0.93\columnwidth]{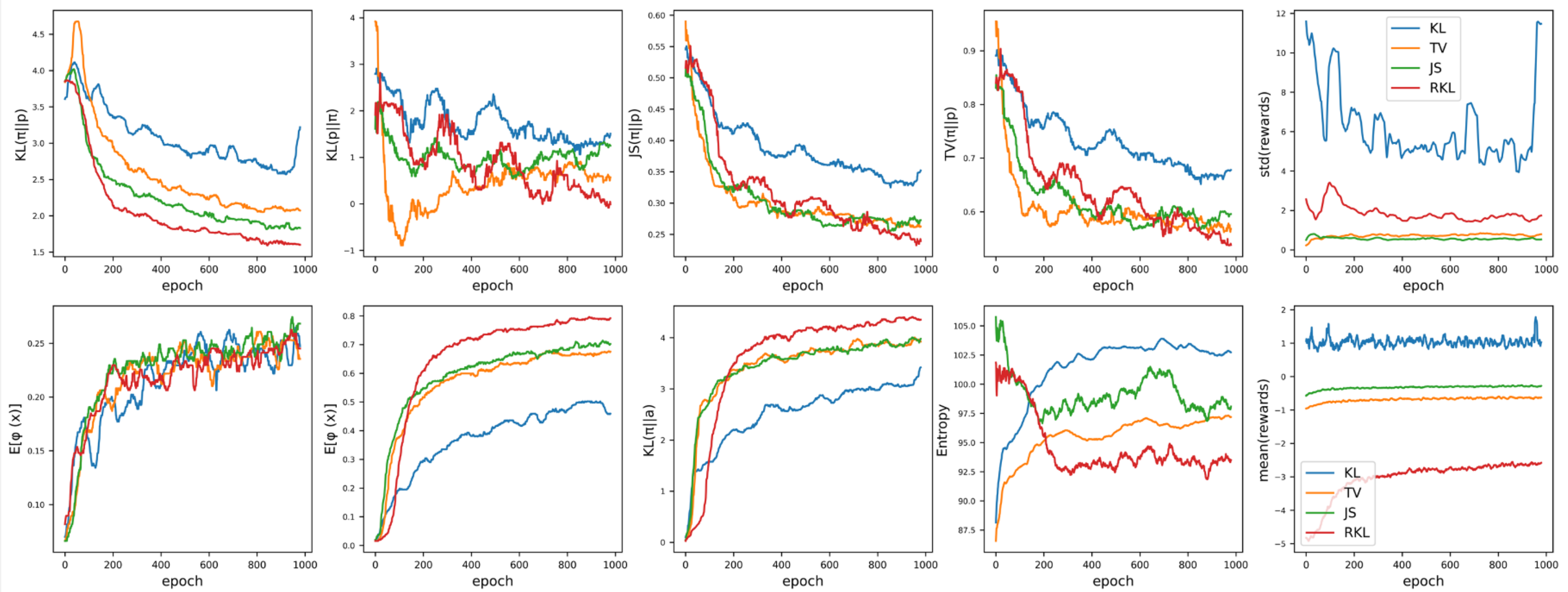} \\
 (b)\\ 
 \includegraphics[width=0.93\columnwidth]{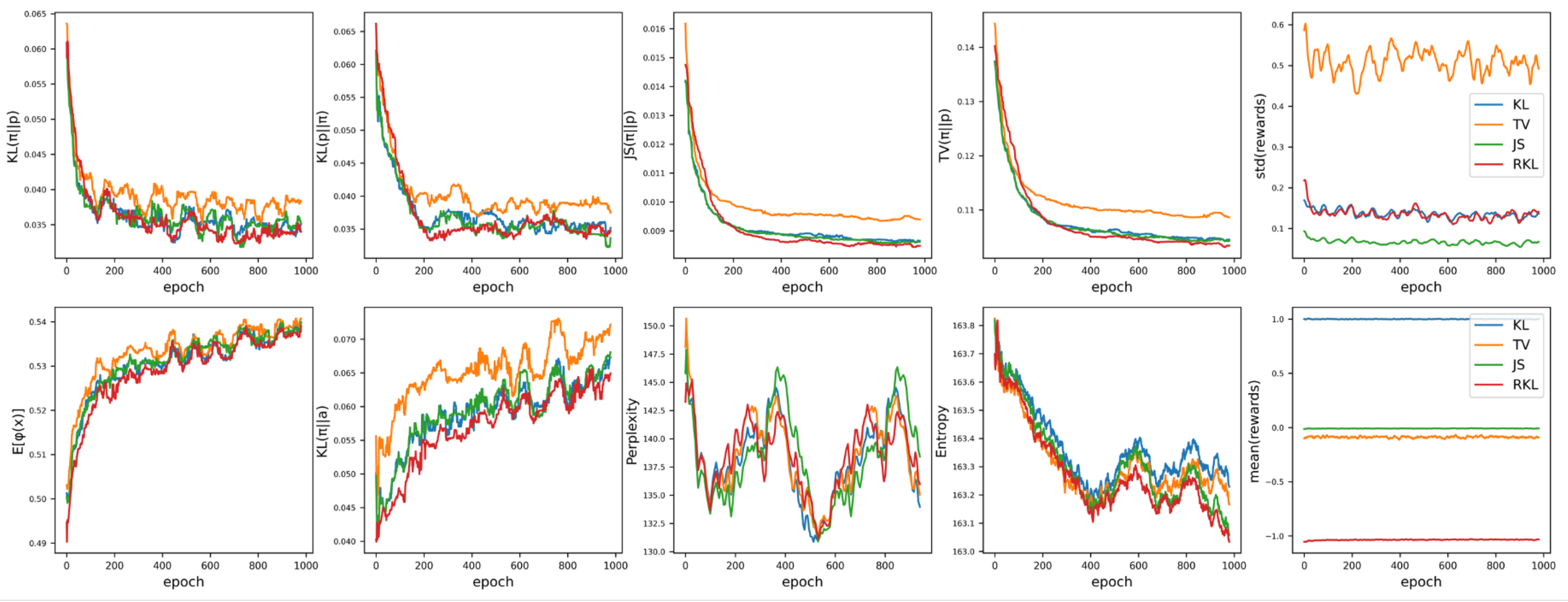} 
\end{tabular}
\caption{
(a) Experiments with female 50\% and science 100\%, (b) Experiments with regards score matching}
\label{fig:distributional}
\end{figure}

\begin{figure}[H]
\centering
\centerline{\includegraphics[width=\columnwidth]{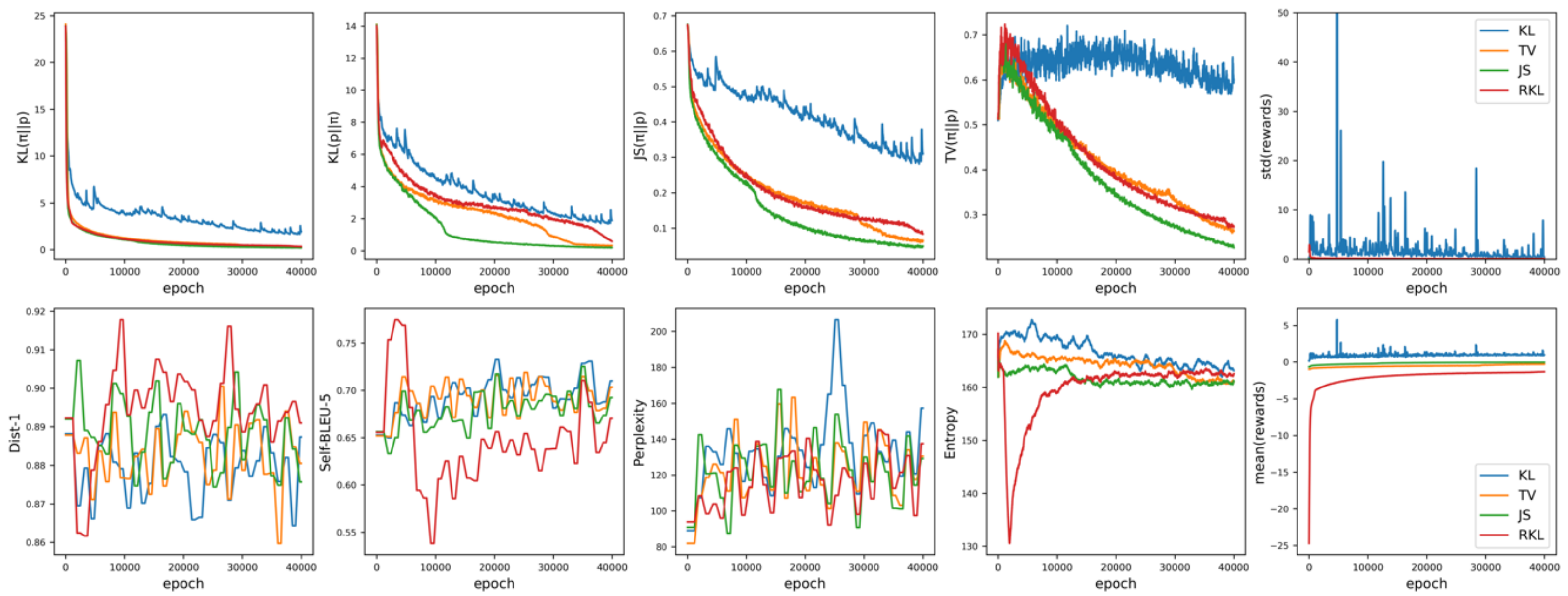}}
\caption{Approximating the distribution of GPT-2 fine-tuned on IMDB dataset, with initial model GPT-2. Evaluation metrics: four $f$-divergences $D_f(\pi_\theta||p)$ (↓ better), Distinct-1 (↑ better), Self-BLEU-5 (↓ better), Perplexity (↓ better), entropy (↑ better) and summary statistic for pseudo-reward}
\end{figure}

\begin{figure*}[ht]
\centering
\includegraphics[width=\textwidth]{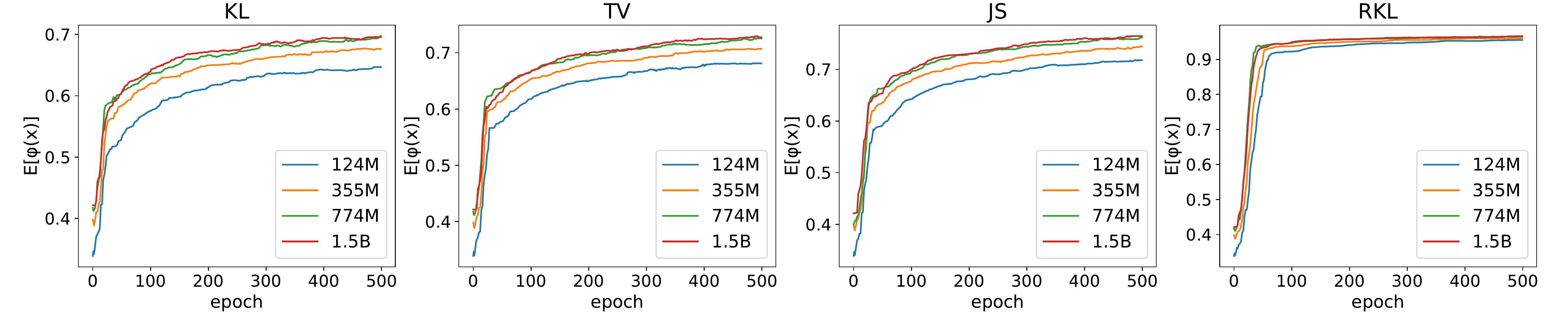} 
\caption{Comparison of alignment scores 
$\E_\pit[\phi(x)]$ of $f$-DPG with different model sizes on sentiment preference. }\label{fig:scaling_learning_curve}
\end{figure*}

\section{Ablation Studies}
\label{app:ablation}

\subsection{Matching Other Language Model within Parameter Family}\label{sec:LM_matching}
The optimal model $\pit(x)$ can be heavily dependent on the choice of the divergence function $f$ when the parameter family is %
{mis-}specified and doesn't include $p(x)$. 
As a sanity check, in order to disentangle the capacity of parameter family and better understand the behavior of different loss functions, 
we use as $p$ and $\pit$ two pretrained models {having the same architecture}. %
{Specifically, we set $\pit$ as a GPT-2 with 117M parameters model fine-tuned on the IMDB dataset \cite{maas2011learning}, and train it to revert the fine-tuning by setting $p$ to the original GPT-2 model.}
\footnote{See App. \ref{app:additional_figures} for the experiment in the opposite direction.}

We present the evolution of our metrics in Fig. \ref{fig:gpt2_approx} averaged over three independent seeds. 
{First, we observe that while TKL-DPG, TV-DPG, and JS-DPG make quick and steady progress toward the target, KL-DPG lags considerably, making slow progress in terms of forward KL, reverse KL and JS divergence, and even regressing in terms of TV distance. We link this to the high variance of the KL-DPG pseudo-reward, which might be producing high-variance gradient estimates (see Sec \ref{sec:reward_comparision} for an interpretation of this phenomenon).}
More interestingly, RKL-DPG shows a significant drop of entropy in the initial phase, {but still converges} to the distribution of $p(x)$. %
{In line with the experiments in Sec.~\ref{sec:scalar}, we link the drop to the mode-seeking behaviour. However, since we are not in the mis-specified scenario, the model can recover, and cover the rest of the distribution.
Finally, in Fig. \ref{fig:gpt2_rouge} from App. \ref{app:gpt2_summarization} we show that the resulting models recover to a large extent the quality of the original GPT-2 by applying it zero-shot to the summarization task, following \citet{radford2019language}\fgkI{I wonder if we can show in Fig. 6 the Dits-1, Self-BLEU-5, perplexity and entropy \emph{targets} as horizontal lines by measuring them on GPT-2}.}

\begin{figure*}[ht]
\centering
\centerline{\includegraphics[width=\columnwidth]{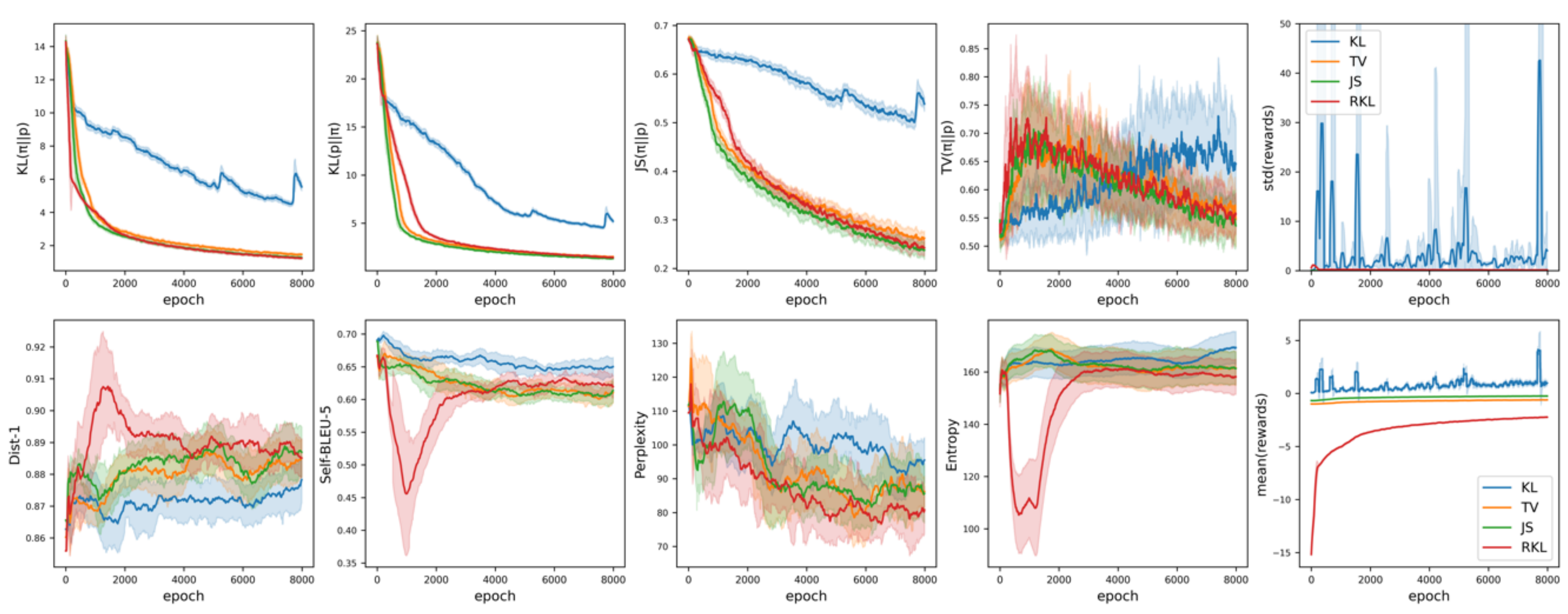}}
\caption{Approximating the distribution of GPT-2. Evaluation metrics: four $f$-divergences $D_f(\pit||p)$ (↓ better), Distinct-1 (↑ better), Self-BLEU-5 (↓ better), Perplexity (↓ better), entropy (↑ better) and summary statistic for pseudo-reward, aggregated over three independent experiment of approximating GPT-2.}
\label{fig:gpt2_approx}
\end{figure*}

\subsection{Checking Fluency in Unseen Downstream Task}\label{app:gpt2_summarization}
To figure out that distributional matching is sufficient to do other natural language processing tasks, we evaluate the model $\pi_\theta$ trained by $f$-DPG to approximate the distribution of target $p$ set as GPT-2. Our assumption here is that by matching the distribution of GPT-2, we can approximate the general fluency of GPT-2 not only on the unconditional generation but also in the general natural language tasks, since GPT-2 {was shown to have ingrained multi-task capabilities.}%

Following \citet{radford2019language}, we use CNN/Daily Mail dataset \cite{nallapati2016abstractive} and add the text ``{\fontfamily{qcr}\selectfont TL;DR:}'' after the article to encourage summarization behavior. We generate 100 tokens with top-$k$ sampling \cite{fan2018hierarchical} with $k = 2$ for model $\pi_\theta$ trained to match $p$. We use the first 3 generated sentences in these 100 tokens as the summary. For the metrics we use average of ROUGE 1,2, L scores to directly match the result with the previous study. Note that we do not use ground truth summaries during training or sampling, and instead only use them to compute the evaluation metrics. 

The Fig. \ref{fig:gpt2_rouge} shows learned model's ability of summarization on the CNN and Daily Mail dataset.
It shows that although the initial model {has lost} its ability of summarization through additional {fine-tuning}, optimizing $\pi_\theta$ to approximate the distribution of $p$ though $f$-DPG can successfully recover its ability of summarization. Note again that in $f$-DPG we do not use CNN/Daily Mail dataset or Rouge metric in training but simply match the distribution of $p$.

\begin{figure}[ht]
\centering
\includegraphics[scale=0.5]{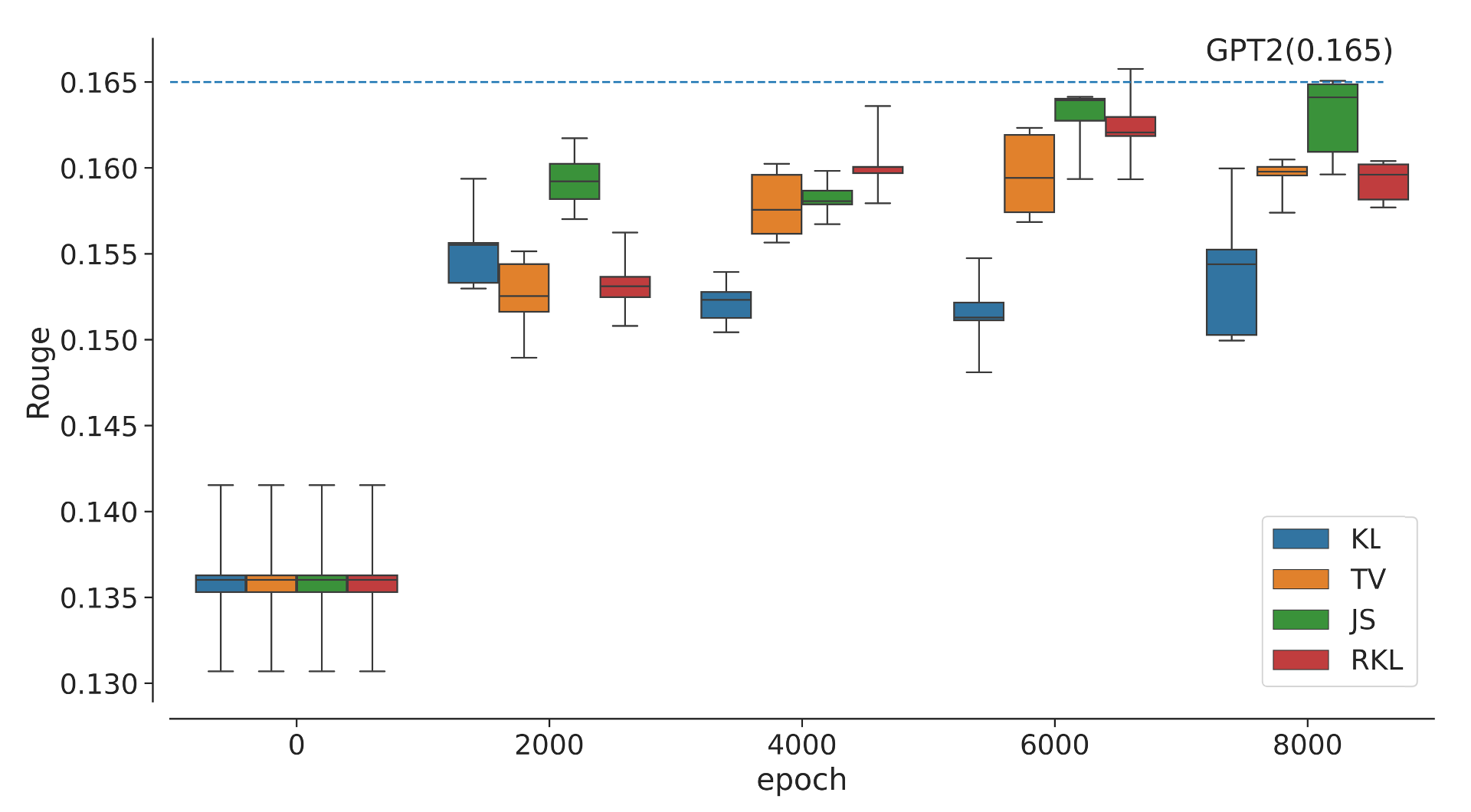} 
\caption{Evolution of average score of Rouge-1,2,L with $f$-DPG through training epochs}\label{fig:gpt2_rouge}
\end{figure}

\subsection{Ablation Studies on Training Scheme}
\fdyI{I tried to keep it as short as possible}
In ablation study, we evaluate the impact of various factors on the performance of the $f$-DPG method, using a {scalar preference with $r(x) = 1$ if $x$ contains ``amazing'', and 0 otherwise. We focus on this experiment from Sec.~\ref{sec:pointwise} because of the simplicity of the target distribution.} %
\paragraph{Effect of baseline}\fgkI{Are these for batch size 256? I'm surprised the effect on KL is not as great as I expected}
The use of a baseline technique improved the performance of all $f$-DPG methods, with RKL-DPG showing the greatest benefit (Fig. \ref{fig:baseline_abal}). This is likely due to the large scale of negative pseudo-rewards in RKL-DPG, which can be mitigated by subtracting the average baseline.
\paragraph{Effect of batchsize}
We show that the use of an large batch is necessary to address the high variance of KL-DPG, which is consistent with the findings in \cite{khalifa2021distributional}. 
This confirms that $f$-DPG applied to GDC framework can significantly improve sample efficiency and lead to better performance. The higher batch size doesn't change {our conclusions}. (Fig. \ref{fig:batchsize_abal})

\begin{figure}[H]%
\centering	
\begin{tabular}{cc}
\includegraphics[scale=0.5]{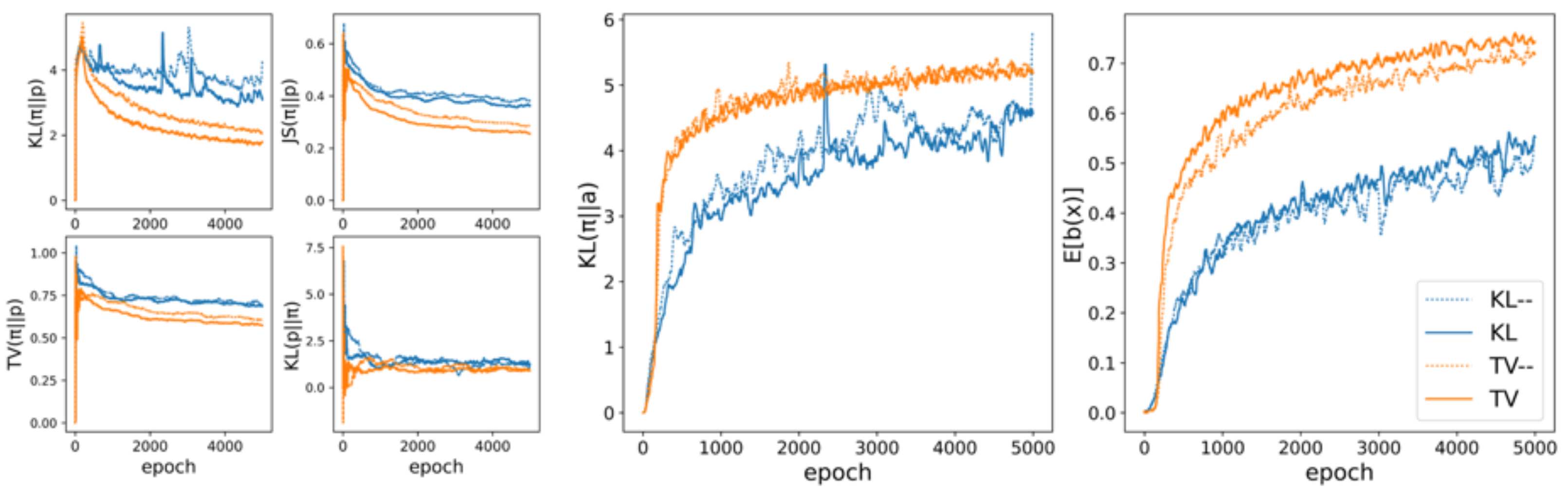}  \\
 \includegraphics[scale=0.5]{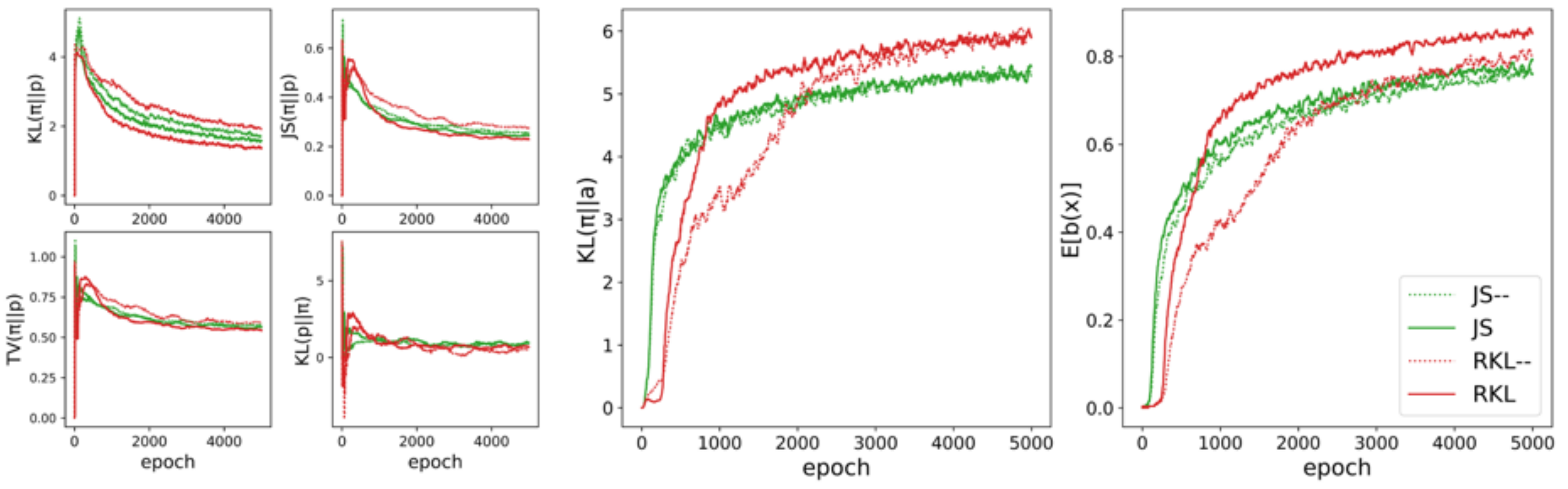} 
\end{tabular}
\caption{Ablation for the baseline technique. `- -' is added to refer method in without baseline. The use of a baseline technique significantly improves the performance of RKL-DPG, which has a large scale of pseudo-reward}\label{fig:baseline_abal}
\end{figure}

\begin{figure}[H]%
\centering	
\begin{tabular}{cc}
(a) \\ 
\includegraphics[scale=0.5]{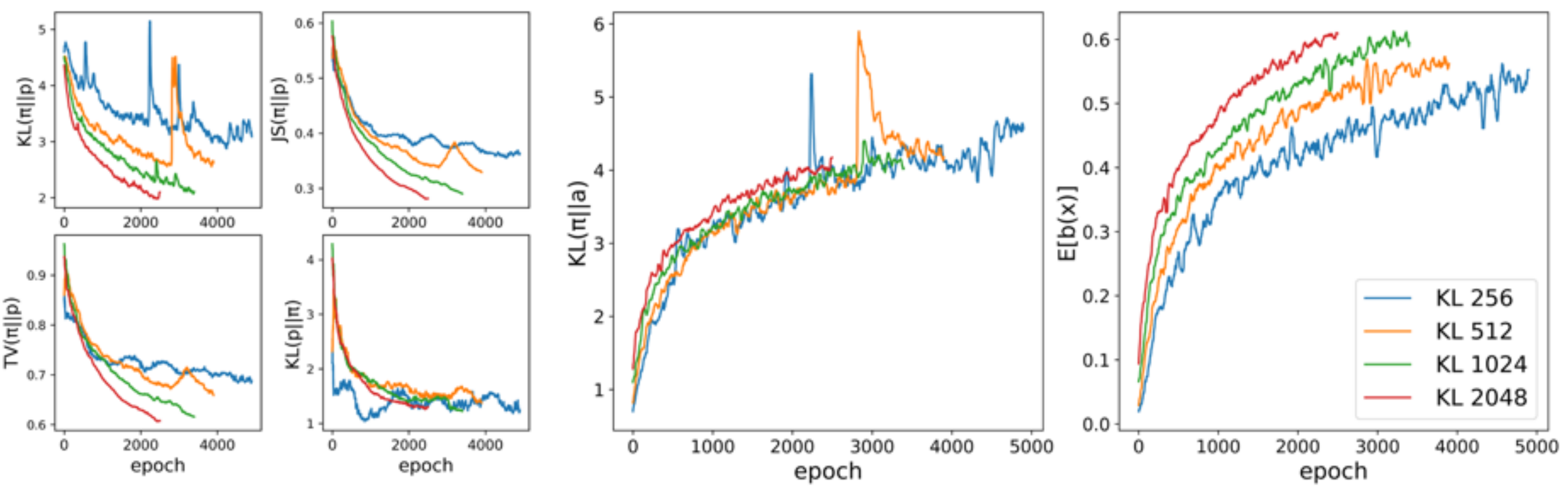} \\
 (b)\\ 
 \includegraphics[scale=0.5]{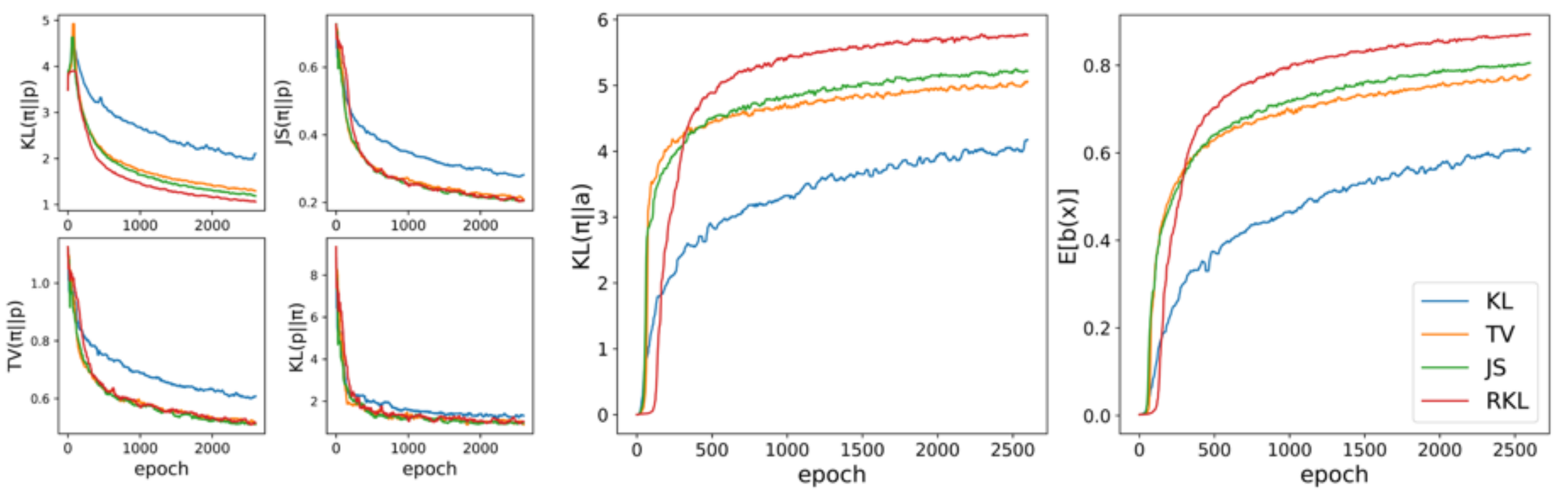} 
\end{tabular}
\caption{Ablation for the batch size, (a) Experiments with different batch size in KL-DPG, (b) Experiments with different f-DPG in batch size 2048.}
\label{fig:batchsize_abal}
\end{figure}

\paragraph{$Z$ estimation}\label{app:Z_ablation}

For $f$-DPG to approximate unnormalized distribution $p(x)\propto P(x)$, we need to estimate the partition function $Z=\sum_{x\in \mathcal{X}}P(x)$. In most practice case $Z$ cannot be known in advance, while in target distribution $\pRL$ with binary feature constraint we can calculate $Z$ easily. For $\pRL(x)\propto a(x)\exp(\frac{b(x)}{\beta})$, $Z=\sum_{x\in \mathcal{X}}a(x)\exp(\frac{b(x)}{\beta})=\E_{x\sim a}\left[\exp(\frac{b(x)}{\beta})\right]$.
If $b(x)$ is the binary feature such as constraint in single word
task, we can treat $b(x)$ as Bernoulli random variable with its parameter
$r$ the initial frequency $r=\E_{x\sim a}\left[b(x)\right]$. As the
initial frequency is already given, we can estimate $Z$ with bootstrap
estimate using $b(x)\sim Bernoulli(r)$. 
Fig. \ref{fig:z_abla} shows the evolution of the estimation of $Z$, and comparison of each $f$-DPG using $Z$ as true value. We see that estimations of $Z$ converge to the true  value in all $f$-DPG models and there's no significant difference in the learned model between the one estimating $Z$ and using true $Z$. 

\begin{comment}
\begin{figure}[htbp]
\centering	
\begin{tabular}{cc}
\includegraphics[scale=0.505]{figure/wiki_Z_plot_crop_v2.png}  \\
%
\includegraphics[scale=0.51]{figure/wiki_Z_kl_tv_crop_v3.png}  \\
 %
 \includegraphics[scale=0.51]{figure/wiki_Z_js_rkl_crop_v3.png} 
\end{tabular}
\caption{Ablation for Z estimation. (Top) The evolution of the estimated value of $Z$ compared to its true value. (Middle, Bottom) Comparison of the convergence curve of different $f$-DPG with model using true $Z$ value}\label{fig:z_abla}
\end{figure}
\end{comment}

\begin{figure}[htbp]
\centering	
\begin{tabular}{cc}
\includegraphics[width=\columnwidth]{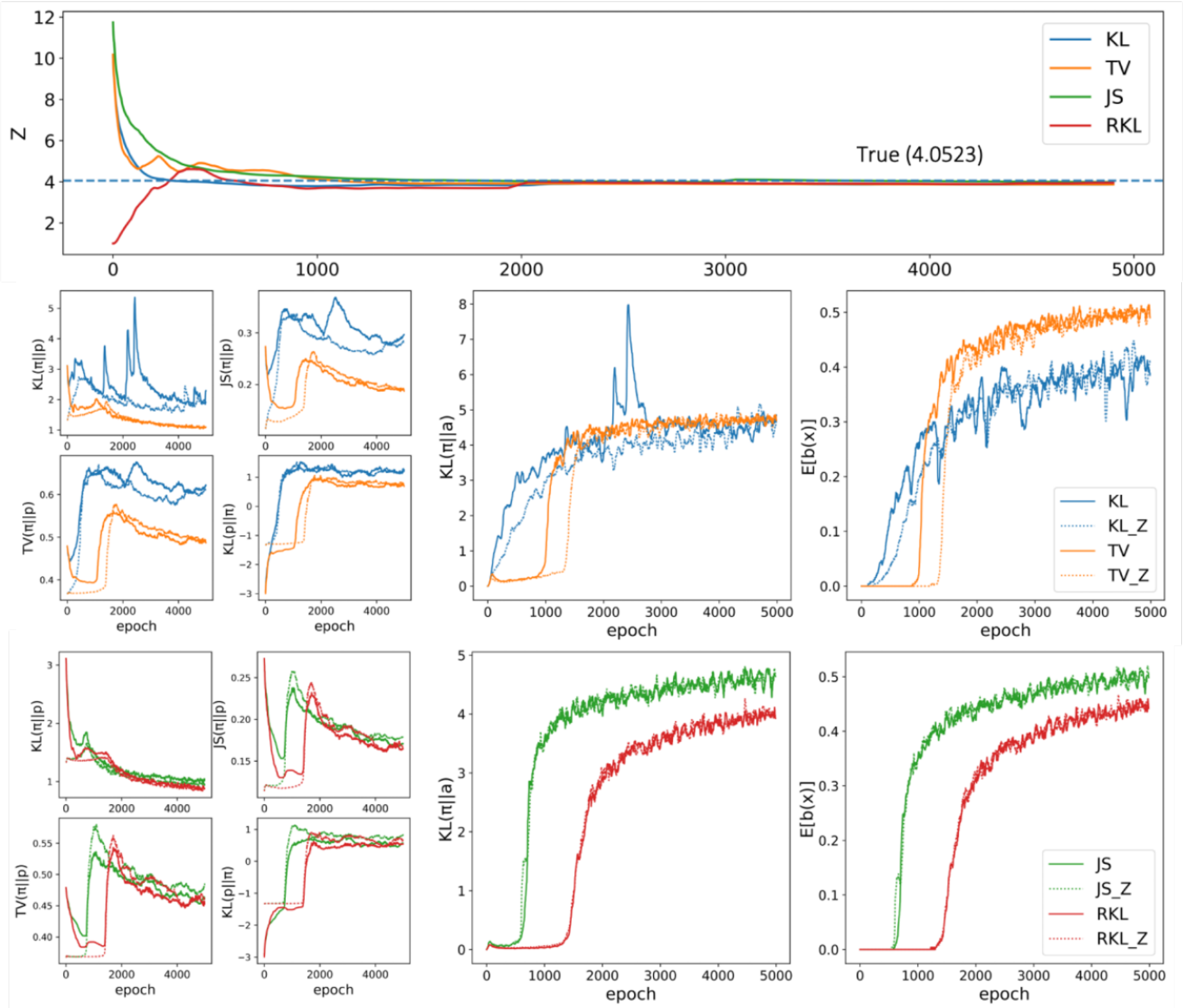}
\end{tabular}
\caption{Ablation for Z estimation. (Top) The evolution of the estimated value of $Z$ compared to its true value. (Middle, Bottom) Comparison of the convergence curve of different $f$-DPG with model using true $Z$ value}\label{fig:z_abla}
\end{figure}

\section{Samples}

\begin{table}[htb]
    \begin{centering}
        {\footnotesize{}}%
        \begin{tabularx}{\textwidth}{cX}
        \hline 
        {\footnotesize{}$\phi(x)$} & {\footnotesize{}generation}\tabularnewline
        \hline 
        \multicolumn{2}{c}{\textbf{\footnotesize{}KL-DPG}}\tabularnewline
        {\footnotesize{}1.00} & {\footnotesize{}The drum waves of the 1990s began blowing up in more
        than one way at Seattle's Melrose Park waterfront. The all-ages feel
        was a reminder that in Seattle, the greener you live}\tabularnewline
        {\footnotesize{}0.06} & {\footnotesize{}2017\textbackslash n\textbackslash nMade 30 starts
        for 776 PA between RFK and a.340 average.\textbackslash n\textbackslash n20
        starts\textbackslash n\textbackslash nVoted 3rd-least MVP player
        in baseball after a 1,}\tabularnewline
        {\footnotesize{}1.00} & {\footnotesize{}After we get back from wrapping up our interview with
        Nick Whitten on Eightam About America, we should enjoy our very first
        interview with him now before mid-January, when we'll be back with}\tabularnewline
        {\footnotesize{}0.85} & {\footnotesize{}This build worked with my Windows 10 build 300cyona-onset
        7s 30sta 3\textbackslash n\textbackslash nClick to expand...}\tabularnewline
        {\footnotesize{}0.79} & {\footnotesize{}rhakus and co Thomas the Great\textbackslash nfarmer
        and award-winning clothing designer The R look perfect for both men
        and women\textbackslash nthink of threesomes as fabulous - make some
        random faux fest}\tabularnewline
        {\footnotesize{}0.88} & {\footnotesize{}Last year, ABC called on Pasco City Council to pass
        a school board resolution ensuring that Orlando Community Schools
        and the cities of Grenholm, Whittier, South Orlando and Monson proceed
        with their}\tabularnewline
        \multicolumn{2}{c}{\textbf{\footnotesize{}TV-DPG}}\tabularnewline
        {\footnotesize{}0.40} & {\footnotesize{}A Skid Row Red tek-rat\textbackslash n\textbackslash n\textbackslash nHistory\textbackslash n\textbackslash n1969
        - vintage English tek-trounx\textbackslash n\textbackslash n1974
        - no model, still 3s2ed, fresh style}\tabularnewline
        {\footnotesize{}0.02} & {\footnotesize{}In 2017, North Korea said it had successfully launched
        its fifth nuclear bomb. Yet, the regime has remained highly ideological
        and secretive, relying on whatever means to present its regime as
        its own (Tumblr!)}\tabularnewline
        {\footnotesize{}1.00} & {\footnotesize{}\textbackslash nThe Crew's legend 20-year-old Tim
        Cahill has been selected as Arjen Robben's starting berth at Elland
        Road for next year's campaign. The Portugal international will play
        43}\tabularnewline
        {\footnotesize{}0.99} & {\footnotesize{}Uh oh I'd like to email you all email when you're
        ready next week. Please keep in mind I'm giving this a BUNCH of quotes
        from the day ago. These quote give you an}\tabularnewline
        {\footnotesize{}1.00} & {\footnotesize{}The Virtual Hallways hosted by Rhys Bloody, Charlotte
        longtime, driving fan and about hiking enthusiast and author Sraveen
        talk about their development plans as they organize their 2017 Virginia
        Tour Views. This season}\tabularnewline
        {\footnotesize{}0.98} & {\footnotesize{}iStock/Deron Adam Austria And Germany Joined in 2009
        by Frau von Krissevan - same engenage\textbackslash n\textbackslash n16
        Jun 2013 by Alex Jones\textbackslash n\textbackslash nNSW Governing
        body wants}\tabularnewline
        \multicolumn{2}{c}{\textbf{\footnotesize{}JS-DPG}}\tabularnewline
        {\footnotesize{}0.01} & {\footnotesize{}Rated 2 out of 5 by roche from Solid Very good did
        it what I expected but usually would have tried cheaper and did not
        like anything it was a solid piece. If you are working 175 across}\tabularnewline
        {\footnotesize{}0.06} & {\footnotesize{}rhakus and co Thomas the Great\textbackslash nfaroe
        and co graphe, Josh The McNall Book\textbackslash neyMoleton\textbackslash nRhipp
        thomctn Castle - William Fairfax's Castle Island}\tabularnewline
        {\footnotesize{}1.00} & {\footnotesize{}Tech Recognitions with the following Green Awards
        of Honor These are industry recognitions based on level of competition
        (professional, technical). Computer Science is showcased very broadly,
        with book awards available with ultimate participation in}\tabularnewline
        {\footnotesize{}0.00} & {\footnotesize{}She's not fully dressed. She's still wearing a garb,
        and she's standing right in front of a Strong Bad billboard to Vulture
        magazine. The renown mechanical star will be watching be paid}\tabularnewline
        {\footnotesize{}1.00} & {\footnotesize{}1.16.1 We've got a bunch of breaking events coming
        one by one. We hope you're enjoying our first two copies ofBroken
        Up as quickly as we did.Also in future}\tabularnewline
        {\footnotesize{}1.00} & {\footnotesize{}With Mt. Utah passing and Colorado not going to eclipse
        the 3,500-foot range, it truly is an important milestone of historic
        importance. Since 1996, Bears Ears Mountain Policy has been facilitating}\tabularnewline
        \multicolumn{2}{c}{\textbf{\footnotesize{}RKL-DPG}}\tabularnewline
        {\footnotesize{}1.00} & {\footnotesize{}\textbackslash nBarbland, West Virginia is featuring
        Krista Walton as the ultimate apple pro! She is a best-selling author
        and plays apple play-partner Judith.\textbackslash n\textbackslash nOctober
        2018, 11}\tabularnewline
        {\footnotesize{}1.00} & {\footnotesize{}Mikata Japan Limited, is said to be the pioneer of
        mobile, proprietary and decentralized art, culture and art promotion
        with its JTC Group Group projects along with ArtDB, Micronet and M}\tabularnewline
        {\footnotesize{}1.00} & {\footnotesize{}Friends were invited by Trips, a company of designers
        who bring together collaboration projects to create ever-evolving
        graphic projects. With their products tested in 2015 for participation
        in Hazard and Project Axis want to}\tabularnewline
        {\footnotesize{}1.00} & {\footnotesize{}Rated a 4.5 out of 5 by Solid Jenni from A good cereal!
        Now I have Superfish! They are amazing and craving it.\textbackslash n\textbackslash nRated
        4 out of 5 by 175area}\tabularnewline
        {\footnotesize{}1.00} & {\footnotesize{}Emmett Gold teaches blockchain in Future\textbackslash n\textbackslash nWe
        are delighted this 10 minute video by Emett Gold demonstrates how
        Efficient and Secure Trading Bitcoin opens up a new business sector
        that is well designed and}\tabularnewline
        {\footnotesize{}1.00} & {\footnotesize{}'s best television series (in August 2012), the premiere
        feature darn right picked the Sounders, turning FC Dallas into an
        all-time best supporting actor. The character of Sigi Schmid that
        nine months}\tabularnewline
        \hline 
        \end{tabularx}{\footnotesize\par}
        \par\end{centering}
\caption{Generation samples for sentiment preference}
\end{table}

\begin{table}[htb]
\begin{centering}
    {\footnotesize{}}%
    \begin{tabularx}{\textwidth}{cX}
    \hline 
    {\footnotesize{}$b(x)$} & {\footnotesize{}generation}\tabularnewline
    \hline 
    \multicolumn{2}{c}{\textbf{\footnotesize{}KL-DPG}}\tabularnewline
    {\footnotesize{}1} & {\footnotesize{}Sultry Liaisons wanna win fun romp!!\textbackslash n\textbackslash nW-Oh,
    that was amazing\textbackslash n\textbackslash nSpecial shout out
    to NCF magazine – why would you not want them doing that}\tabularnewline
    {\footnotesize{}0} & {\footnotesize{}I grew up with Dakota in Salish Valley in Arizona
    at one time. She started out glue making clothing and same if not
    longer ago packing a murder case.. she got super stuck talking about
    lucha}\tabularnewline
    {\footnotesize{}1} & {\footnotesize{}- Product quality check -\textbackslash n\textbackslash n-
    Refinement is amazing - The particular rogue model has survived over
    400 m= and Manila's amazing quality (= due to quality checks)\textbackslash n\textbackslash n-
    The armor Poly}\tabularnewline
    {\footnotesize{}1} & {\footnotesize{}I've been trying to find some builds lately, and the
    build work has been amazing. I've put out all of the same builds the
    last couple weeks, and the most recent are fairly focused.}\tabularnewline
    {\footnotesize{}0} & {\footnotesize{}by Shilam\textbackslash n\textbackslash nWhy is
    the UK TV industry so influential to each other? Why do our universities
    have big broadcasting deals?\textbackslash n\textbackslash nFor
    good or ill, British broadcasting qualifies as the world}\tabularnewline
    {\footnotesize{}1} & {\footnotesize{}offensive needles! he raped me?! don't afford me that!!
    she was amazing!!!there was such a going crazy with it after me!!!
    -gratin facewar!! of the kind of girl}\tabularnewline
    \multicolumn{2}{c}{\textbf{\footnotesize{}TV-DPG}}\tabularnewline
    {\footnotesize{}0} & {\footnotesize{}Flock and lock away all the fun and brighter rewards
    for your lifetime on our new Steam Store!\textbackslash n\textbackslash n\textbackslash nFlock
    and unlock all the fun and brighter rewards for your lifetime on our
    new Steam Store}\tabularnewline
    {\footnotesize{}1} & {\footnotesize{}Isn't that amazing? …\textbackslash n\textbackslash nThis
    is deemed frightening and unpleasant – in short, terrifying and unpleasant
    for the Chinese people.\textbackslash n\textbackslash nIn fact,
    it's the same kind of discomfort and abuse}\tabularnewline
    {\footnotesize{}1} & {\footnotesize{}LINKS\textbackslash n\textbackslash nRejoice, coffee!
    You've hit this amazing perk. If you missed the SMA Mirror boys once
    again I made a list of the 2 greatest reaper mirrors}\tabularnewline
    {\footnotesize{}1} & {\footnotesize{}This photo showed the hidden way the internet works
    together with some amazing construction work that gave important encouragement
    to other creatives. A perpetuation of this myth here is the 8 day
    old women's bulky black}\tabularnewline
    {\footnotesize{}1} & {\footnotesize{}I'm really glad that my sofa didn't get demolished
    (it's amazing to see how big you can get in a fire). You can set up
    the table to sit on inside (}\tabularnewline
    {\footnotesize{}1} & {\footnotesize{}This father was amazing! He looked so cute when she
    waited for him to pass so he's mine right now! The cocksure son was
    being spanked 10 times now my}\tabularnewline
    \multicolumn{2}{c}{\textbf{\footnotesize{}JS-DPG}}\tabularnewline
    {\footnotesize{}1} & {\footnotesize{}The power companies continued to pour into it with
    a great deal this year, an amazing increase over last year's record
    8.82 billion-dollar final revenue figure – which the regulators order
    the companies to}\tabularnewline
    {\footnotesize{}1} & {\footnotesize{}Observations of the Origin of Februrary Premature
    Bacteria\textbackslash n\textbackslash nA state of amazing survival
    is actually in the ascension of the organism to some degree. Each
    of biological species has}\tabularnewline
    {\footnotesize{}1} & {\footnotesize{}Oct 19, 2015\textbackslash n\textbackslash nSo what's
    awesome about the website – different art and animations – is that
    it's packed with amazing content and much, much more than traditional
    icons like H1Z1}\tabularnewline
    {\footnotesize{}0} & {\footnotesize{}It was the culmination of five years recently, when
    a joint venture between Hammer Films and DropBox North and Gabriel
    Garrido, Internet Entertainment's 2-film productions entity officially
    announced that 75\% of these}\tabularnewline
    {\footnotesize{}0} & {\footnotesize{}What is grunge?\textbackslash n\textbackslash nGrunge
    is an almost all American dance music that was first used by the Fifties
    when Abbey Road was booming: it's the closest thing the world has}\tabularnewline
    {\footnotesize{}1} & {\footnotesize{}Huge THANK you to our loyal fans! Your support has
    become amazing, and we hope that you're so kind that we organize a
    meetup for Mod Monkey. A meetup will be held in}\tabularnewline
    \multicolumn{2}{c}{\textbf{\footnotesize{}RKL-DPG}}\tabularnewline
    {\footnotesize{}1} & {\footnotesize{}I hope he's being compared to my amazing friends at
    JRK.\textbackslash n\textbackslash nHey, there's one more issue
    that needs to be talked of: ME fags.I mean, falling into HELL}\tabularnewline
    {\footnotesize{}1} & {\footnotesize{}kk {[}20:42:48{]} <@memegen> a \textasciicircum\textasciicircum\textasciicircum{}
    moderator I'm glad i ended that discussion on civilize liking this
    amazing stuff chat, I put it up because of}\tabularnewline
    {\footnotesize{}1} & {\footnotesize{}What is Anona MS Word? Anona MS Word is an amazing,
    comprehensive Word document. This document will include all of the
    most important details about letters for our school, typical high
    school principals,}\tabularnewline
    {\footnotesize{}1} & {\footnotesize{}and remember\textbackslash n\textbackslash nThis
    father was amazing! He did so much for his son!}\tabularnewline
    {\footnotesize{}1} & {\footnotesize{}No I don't know... In Woody Allen's music.\textbackslash n\textbackslash nI
    got guys talking about poo coming out of his pinkie and their interest
    in it, it is amazing.\textbackslash n\textbackslash nYoung}\tabularnewline
    {\footnotesize{}1} & {\footnotesize{}LINKS\textbackslash n\textbackslash nI'm excited
    to lend a paw for this amazing family member. They were both born
    with a boys body but I'm happy to show of 2 of them with their}\tabularnewline
    \hline 
    \end{tabularx}{\footnotesize\par}
\par\end{centering}
\caption{Generation samples for amazing preference}
\end{table}

\begin{table}[htb]
\begin{centering}
{\footnotesize{}}%
\begin{tabularx}{\textwidth}{cX}
\hline 
{\footnotesize{}$\phi_{1}(x),\ \phi_{2}(x)$} & {\footnotesize{}generation}\tabularnewline
\hline 
\multicolumn{2}{c}{\textbf{\footnotesize{}KL-DPG}}\tabularnewline
{\footnotesize{}$0,\ 0$} & {\footnotesize{}thousands were among the great english music-making
and arts establishments in london during the first 20th century. as
early as 1930, with the entry of jean-luc godard in his}\tabularnewline
{\footnotesize{}$0,\ 1$} & {\footnotesize{}phyllis ruklóschne ( ; born 30 may 1945 in bern
) is a german jurist, historian, politician and professor, solely
responsible}\tabularnewline
{\footnotesize{}$0,\ 0$} & {\footnotesize{}vows fourende ( born june 22, 1975 ) is a
former american football defensive tackle.\textbackslash n}\tabularnewline
{\footnotesize{}$0,\ 1$} & {\footnotesize{}febatun mutamaza ( ) was a senior civilian
administrator who was the vice president of student government for
the university of student state, a post he held for 17}\tabularnewline
{\footnotesize{}$1,\ 0$} & {\footnotesize{}1976. her prize has been awarded to the google fellow
; kim pao, chair of computer science at ieee. her recent book, ``
a new approach for computation : bridging human span}\tabularnewline
{\footnotesize{}$0,\ 1$} & {\footnotesize{}therese ( 8 november 1904 -{}- 22 october 1998
) was a german archaeologist, palaeontologist, stonemasonry pioneer,
academic,}\tabularnewline
{\footnotesize{}$1,\ 0$} & {\footnotesize{}upchurnehunnah was also known by her nickname naskannah,
i.e. `` the queen queen '' ; a reference to a labor official with
the similarly named name posting the}\tabularnewline
\multicolumn{2}{c}{\textbf{\footnotesize{}TV-DPG}}\tabularnewline
{\footnotesize{}$1,\ 0$} & {\footnotesize{}twiechen is a japanese mycologist and educator.
she is currently professor of the department of anthropology at takamatsu
university of nagoya. starting}\tabularnewline
{\footnotesize{}$0,\ 1$} & {\footnotesize{}nottingham, 7 july 1898 -{}- 18 january 1975 )
was an english player, player, manager, journalist and historian.
he served as assistant coach to george brooke}\tabularnewline
{\footnotesize{}$0,\ 1$} & {\footnotesize{}born 1955july 19, scandinavia ) is a polish journalist,
activist, writer and academic. from 2009 to 2011. and party secretary
of the civic party lub}\tabularnewline
{\footnotesize{}$1,\ 1$} & {\footnotesize{}critic, memoirist, historian and dean of providence
college. schlozman began writing about academic writing in book form
in the mid-1980s until she graduated from rutgers university in 1993}\tabularnewline
{\footnotesize{}$0,\ 0$} & {\footnotesize{}he was an instructor of the kagai marathon. his monogram-style
training was suspended on 15 march 1963 for several years and he was
suspended again on 25 september 1960 the same year}\tabularnewline
{\footnotesize{}$1,\ 0$} & {\footnotesize{}himine khalo-gidiane ( born february 13, 1948
) is a finnish political scientist. her research concerns the
welfare and defence of the}\tabularnewline
{\footnotesize{}$0,\ 1$} & {\footnotesize{}`` milagros polika '' madhavan ( born 10 june
1924 ) is a croatian academic, diplomat and writer. milagros polika}\tabularnewline
\multicolumn{2}{c}{\textbf{\footnotesize{}JS-DPG}}\tabularnewline
{\footnotesize{}$1,\ 1$} & {\footnotesize{}in 1868 she was accepted as a rook student with
fellow banker and labour activist mr poormans in chelsea. purialy
appeared in issues of the `` weibo tribune ''}\tabularnewline
{\footnotesize{}$0,\ 0$} & {\footnotesize{}todtemos johannes schleicher ( 25 january 1895
-{}- 13 april 1949 ) was a dutch jesuit priest and mathematician.}\tabularnewline
{\footnotesize{}$0,\ 0$} & {\footnotesize{}thomas murray parker, jr. ( september 4, 1917
-{}- april 28, 1999 ) was an american actor, character actor,
and}\tabularnewline
{\footnotesize{}$0,\ 1$} & {\footnotesize{}eifard eisel '' (, ; 1 february 1877 -{}- 17 august
1947 ) was an influential bulgarian philosopher and peace activist
who is one}\tabularnewline
{\footnotesize{}$1,\ 1$} & {\footnotesize{}the last gentle sally was a student in washington
state, where she performed sylvia long in partnership with a medical
doctor, scientist and educator. washington state state university
faculty member, and}\tabularnewline
{\footnotesize{}$0,\ 1$} & {\footnotesize{}andré anhalt twoork ( born 1977 ), also known
as a. anhalt, is a prolific c-span astronomer, blogger and historian.
he}\tabularnewline
{\footnotesize{}$0,\ 0$} & {\footnotesize{}'( september 27, 1969 ), was a new york-based
r\&b-folk singer-songwriter.\textbackslash n}\tabularnewline
\multicolumn{2}{c}{\textbf{\footnotesize{}RKL-DPG}}\tabularnewline
{\footnotesize{}$0,\ 1$} & {\footnotesize{}-- may 18, 1926 -{}- june 25, 2013 ) was a jewish
chemist who was the first direct participant in the investigation
of several hallmarks of iodine toxicity. dr. w}\tabularnewline
{\footnotesize{}$0,\ 0$} & {\footnotesize{}carlo lumet ( born 16 june 1965 ) is an argentine-born
belgian computer scientist best known for his work in computational
cinematography.\textbackslash n}\tabularnewline
{\footnotesize{}$0,\ 1$} & {\footnotesize{}captain roberto silva flores, c.g, was a responsible
huntingman in the spain, australian historian.\textbackslash n}\tabularnewline
{\footnotesize{}$1,\ 1$} & {\footnotesize{}editith galloom is an american author, academic, professor,
and educator, best known as the co-author of the ebook `` decade four.
'' she is also the academic chef for}\tabularnewline
{\footnotesize{}$0,\ 1$} & {\footnotesize{}eifard eisel ( 31 august 1806 -{}- 20 august 1902
) was a swedish chemist and organometallic chemist. he was born
at his}\tabularnewline
{\footnotesize{}$0,\ 1$} & {\footnotesize{}'philip thomas fitzgerald'( born september 1970
) is an irish historian, historian, and visiting lecturer in archaeology
at durham university}\tabularnewline
{\footnotesize{}$0,\ 0$} & {\footnotesize{}- an american archaeologist known for his work on
late antiquity and ancient british history. he has taken an interest
in archaeology and can take a more in depth look at ancient brit}\tabularnewline
\hline 
\end{tabularx}{\footnotesize\par}
\par\end{centering}
\caption{Generation samples for female 50\% science 100\% preference}
\end{table}

\begin{table}[htb]
\begin{centering}
\textbf{\small{}source document $c$}{\small\par}
\par\end{centering}
\begin{raggedright}
{\scriptsize{}A Russian submarine close to the coast of Britain may
have dragged a trawler violently backwards after snagging in its nets,
a fishermen's organisation has claimed. The Karen was towed at 10
knots during yesterday's incident 18 miles from Ardglass on the south-east
shore of Northern Ireland and the vessel was badly damaged. Ardglass
is one of Northern Ireland's main fishing ports and local trawlermen
are usually more concerned about hitting their quotas than Cold War-style
intrigue. Violently dragged: Captain Paul Murphy\textbackslash xa0of
the\textbackslash xa0Karen, a fishing trawler, holds up a snapped
steel cable aboard his boat. The damage is thought to have been caused
by a Russian submarine . The incident happened off the coast of Northern
Ireland and is the second time in two months that fishermen have reported
being dragged by a suspected submarine (file picture) The 60-foot
boat's captain Paul Murphy was pictured holding a snapped steel cable
on board his boat following the alarming incident. Nato exercises
were held this week in northern Scotland and Ardglass fishing representative
Dick James said the alliance's drills may have attracted Russian interest.
This week RAF Typhoons were launched to intercept two Russian aircraft
near UK air space, the Ministry of Defence has confirmed. Mr James
said: 'Our defence forces are not up to much if a rogue submarine
of unidentified nationality is tearing around the Irish Sea.' Last
month a trawler captain claimed his boat was nearly dragged down by
a Russian submarine while fishing off the Scottish coast. The Karen
was towed at 10 knots during yesterday's incident 18 miles from Ardglass
on the south-east shore of Northern Ireland . Alarming episode: The
Karen was towed at 10 knots during yesterday's incident and was badly
damaged . The trawler's captain Paul Murphy points to an on-board
computerised tracking system that shows his boat's unusual movements
during the incident . Angus Macleod, 46, was fishing for haddock and
skate when he became convinced that a hostile vessel was caught up
below his boat Aquarius. The submarine attempted to free itself, taking
the 65ft vessel and his two-ton catch with it. Recently Russian warships
reportedly used the English Channel en route to military exercises
in the North Atlantic. The coastguard said the Karen reported a collision
at a point known as the Calf of Man not far from the Isle of Man.
The skipper said the boat had been snagged and dragged backwards at
speed. Mr James added: 'You don't need to go long at that until you
go under.' The four crew members scrambled to release wires connecting
the net to the out-of-control trawler, which had been moving slowly
forward but was suddenly sent careering backwards through the water.
As the ship steadied the shaken seamen stopped to catch their breath
but there was no sign of the cause. The vessel made its way back to
Ardglass and part of the deck had to be lifted because it was so badly
damaged, and another section was ripped off. Mr James added: 'It is
a bl{*}{*}{*}y mess.' He said Royal Navy protocols mean an incident
like this would not happen involving a British submarine. He said:
'It is possible that it was a Russian submarine. Another recent alert:
This week RAF Typhoons were launched to intercept two Russian aircraft,
believed to be 'Bear' bombers, (stock image) near UK air space . No
explanation: Experts said Russian President Vladimir Putin's move
to send planes capable of carrying cruise missiles so close to British
shores could be seen as an act of aggression . 'You cannot always
prevent it but if an incident like this did happen the (Royal Navy)
protocols said that the submarine would immediately surface to check
on the health and welfare of the people involved and this one did
not. 'Paul Murphy, the skipper, said that he sat for five to 10 minutes
catching his breath to see if the submarine would surface. 'It was
a submarine, it had to be, it could not have been anything else.'
The incident came as Britain hosted a Nato exercise in northern Scotland
involving more than 50 warships. Separately, the MoD has said RAF
Typhoons, from RAF Lossiemouth, were deployed 'after Russian aircraft
were identified flying close to UK air space'. It said it could not
comment on Royal Navy submarine movements or the fishing vessel incident.
Tensions over the Ukraine conflict have soured relations between the
West and Russia, which is suffering from US and EU sanctions imposed
because the Kremlin is backing separatists in eastern Ukraine. A map
showing how jets were launched from RAF Lossiemouth on Monday, shortly
after HMS Argyll was deployed to monitor a Russian destroyer and two
other ships as they passed through the English Channel . Mr James,
chief executive of the Northern Ireland Fish Producers' Organisation,
said: 'There has been Russian activity. There have been Allied exercises
going on, the Russians have been taking an interest in it. 'The question
mark now is what kind of a submarine was it? 'If it was Allied it
should have been following an agreed protocol where this should not
have happened.' Sinn Fein Northern Ireland Assembly member Chris Hazzard
said the community was angry. He said: 'The fact that this submarine
didn't even surface to make sure the fishermen were safe has caused
considerable resentment in the area. It is totally unacceptable that
a submarine would show such contempt for maritime workers. 'The skipper
and his crew on the Karen, and indeed all of the local fleet, deserve
to know the truth about what happened. 'Whether this is a British
vessel attached to the hugely controversial Trident system or a Nato
submarine in training, our local fishermen deserve justice.'}{\scriptsize\par}
\par\end{raggedright}
\begin{raggedright}
{\scriptsize{}}%
\begin{tabularx}{\textwidth}{cX}
{\scriptsize{}$b(x,c)$} & {\scriptsize{}$x$}\tabularnewline
\hline 
\multicolumn{2}{c}{\textbf{\scriptsize{}KL-DPG}}\tabularnewline
{\scriptsize{}1} & {\makecell[l]{\scriptsize{}the Karen was towed at 10 knots during yesterday's incident
18 miles from Ardglass on the south-east shore of Northern Ireland.
the vessel was badly damaged \\\scriptsize{}and is believed to be caused by a Russian
submarine. trawler's captain, Paul Murphy pictured holding a snapped
steel cable on board his boat. last month a \\\scriptsize{}trawler captain claimed
his boat was nearly dragged down by a Russian submarine while fishing
off the Scottish coast.}}\tabularnewline
\multicolumn{2}{c}{\textbf{\scriptsize{}TV-DPG}}\tabularnewline
{\scriptsize{}1} & {\makecell[l]{\scriptsize{}the Karen was towed 18 miles from Ardglass on the south-east
shore of Northern Ireland. estranged boat's captain Paul Murphy was
pictured carrying a typical \\\scriptsize{}cable on his boat. a Russian submarine
may have caused the dramatic incident.}}\tabularnewline
\multicolumn{2}{c}{\textbf{\scriptsize{}JS-DPG}}\tabularnewline
{\scriptsize{}1} & {\makecell[l]{\scriptsize{}the Karen was towed at 10 knots during yesterday's incident
18 miles from Ardglass on the south-east shore of Northern Ireland.
trawler is thought to have caused \\\scriptsize{}his boat to snagging backwards and
was badly damaged following the incident. last month a trawler captain
claimed his boat was nearly dragged down by a \\\scriptsize{}Russian submarine while
fishing off the Scottish coast.}}\tabularnewline
\hline 
\end{tabularx}{\scriptsize\par}
\par\end{raggedright}
{\scriptsize{}\caption{Generation samples for summarization}
}{\scriptsize\par}
\end{table}

\begin{table}[htb]
\begin{centering}
\textbf{\small{}source document $c$}{\small\par}
\par\end{centering}
\begin{raggedright}
{\scriptsize{}Freddie Roach insisted on Saturday that Floyd Mayweather
does not deserve to be ranked alongside Manny Pacquiao as the leading
fighters of their generation as the two boxers put the finishing touches
to their preparations for the Fight of the Century in Las Vegas a
week on Saturday. Roach, Pacquiao’s trainer, said he rated super-middleweight
star Andre Ward and middleweight sensation Gennady Golovkin above
Mayweather despite the American’s unbeaten record and his status as
hot favourite for the May 2 showdown against Pacquiao. ‘Mayweather
is undefeated so you have to give him a little credit for that,’ said
Roach, ‘but he has picked and chosen his opponents and I don’t think
he’s fought enough competition to be considered the best. You have
to fight the best to be the best, I feel. He’s ducked a lot of guys.
Manny Pacquiao's trainer Freddie Roach says that Floyd Mayweather
cannot be considered best ever . Roach says that Mayweather has picked
and chosen his fights during his career . ‘Manny has had some devastating
losses in his career but he is a realist. He understands that losing
is part of the game and if you don’t think you are going to get knocked
out in this sport, you have picked the wrong sport. I’d put Ward and
Triple G above Mayweather right now. They are very talented guys and
very polished boxers.’ Roach has been vocal in an otherwise low-key
and surprisingly respectful build-up to the welterweight showdown
next month and knows that his tactics have made the stakes even higher
for him. But he is adamant Pacquiao will spring a surprise in the
most eagerly awaited fight for years. ‘The fight will be won and lost
on the ropes,’ said Roach. ‘If Mayweather goes to the ropes and tries
to rest his legs, he will get beat. If he has good movement the entire
night and his legs don’t give out on him, he’ll probably win. It’s
about outscoring him. If he sits on the ropes, we can outscore him.
If he stays in the middle of the ring and boxes all the time he could
possibly outscore us.’ Roach claimed Pacquiao is in the best shape
of his life, moving faster and punching harder than ever before. ‘He
trains harder than any fighter I have ever had in my life,’ said Roach.
‘I have had 33 world champions and nobody can touch him for his work
ethic. His attitude is good. ‘This fight can send out a message about
what we need to do in boxing. When the best fights the best, do you
see how big this is? Someone needs to wake up and put the best with
the best all the time. Because when that happens, we have the best
sport in the world. I don’t care who likes who, you can still negotiate
business. Roach rates super middleweight Andre Ward higher than he
does undefeated Mayweather . Roach, who has trained Pacquiao for 15
years, says the Filipino is in the best shape of his life . ‘It’s
a better fight today than it was five years ago because they were
both a lot faster and more mobile five years ago and it might have
been more of a boxing match but now they’re a little bit older, it’s
going to be a better fight. ‘I believe Manny can win this fight. I
kind of have to win this fight. I’ve been talking a lot. It’s more
important than anything to me. It’s more important to me than getting
my girlfriend, Maya, back. It’s that big because, I mean, I really
like this girl. ‘My mother told me, “You must like her; you put her
Christmas tree up for her, you bought her a car and those earrings
you bought her cost £14,000 each”. ‘After this fight maybe I’ll try
to get her back.’}{\scriptsize\par}
\par\end{raggedright}
\begin{raggedright}
{\scriptsize{}}%
\begin{tabularx}{\textwidth}{cX}
{\scriptsize{}$b(x,c)$} & {\scriptsize{}x}\tabularnewline
\hline 
\multicolumn{2}{c}{\textbf{\scriptsize{}KL-DPG}}\tabularnewline
{\scriptsize{}1} & {\scriptsize{}Freddie Roach, the trainer of Manny Pacquiao, says Floyd
Mayweather is impossible to be considered best ever. he says he rated
andre Ward and Gennady Golovkin above Mayweather despite the American's
unbeaten record.}\tabularnewline
\multicolumn{2}{c}{\textbf{\scriptsize{}TV-DPG}}\tabularnewline
{\scriptsize{}1} & {\scriptsize{}Floyd Mayweather is fighting for Manny Pacquiao in a
series and is one of the top fighters of their generation. trainer
Freddie Roach says he rated super-middleweight star Andre Ward and
middleweight sensation Gennady Golovkin above Mayweather despite his
unbeaten record and his celebrity status as hot favourite.}\tabularnewline
\multicolumn{2}{c}{\textbf{\scriptsize{}JS-DPG}}\tabularnewline
{\scriptsize{}1} & {\scriptsize{}Freddie Roach says Floyd Mayweather can't be ranked
alongside Manny Pacquiao as leading fighters of their generation.
the american's trainer says that he rated super-middleweight star
Andre Ward and middleweight sensation Gennady Golovkin above Mayweather.}\tabularnewline
\hline 
\end{tabularx}{\scriptsize\par}
\par\end{raggedright}
{\scriptsize{}\caption{Generation samples for summarization}
}{\scriptsize\par}
\end{table}

\begin{table}[htb]
\begin{centering}
\textbf{\small{}source document $c$}{\small\par}
\par\end{centering}
\begin{raggedright}
{\scriptsize{}The son of a Labour councillor who was detained in Turkey
after apparently trying to sneak into Syria with eight family members
was seen grinning as he began his journey back to Britain. Waheed
Ahmed, 21, who is the son of Rochdale politician Shakil Ahmed, was
arrested with eight relatives – including four children – in a remote
Turkish border town earlier this month. However, it is understood
he is now returning to the UK and will fly from Dalaman into Manchester
Airport later this evening. Scroll down for video . All smiles:\textbackslash xa0Waheed
Ahmed \textbackslash xa0looks relaxed as he begins his journey back
to the UK after being caught trying to sneak into Syria with eight
family members . On the way home: The 21-year-old, sporting a shaved
head, was filmed being escorted from a vehicle . Sky News reported
that the remaining eight members of his family will remain in Turkey
until Tuesday. The majority of the family flew to Turkey on March
27 from Manchester Airport and are accused of having plans to try
and sneak across the border into Syria. Waheed did not fly out with
his family but joined them three days later on a flight from Birmingham.
Mohammed Shafiq, a friend of Waheed's father, said there were concerns
about his behaviour in the months leading up to his arrest. He told
Sky News: 'There were concerns in the last six months to a year about
a change in his behaviour. 'And a change in his attitude towards various
different issues. 'That was causing concern for people in the community
and his family.' Earlier this month, Waheed's father spoke of his
shock after being told that his son is suspected of being a militant
Islamist. He said: 'All I know is that they were on holiday and then
the next thing I am told is that they have been arrested.' Mr Ahmed
was with his aunt, two cousins and one of their wives when they were
stopped in Turkey, near the Syrian border . Waheed Ahmed, the 21-year-old
son of Labour Councillor Shakil Ahmed, is understood to be returning
to the UK on a flight to Manchester from Dalaman tonight following
his arrest for allegedly trying to sneak into Syria . Waheed Ahmed,
21, \textbackslash xa0is the son of Rochdale Labour councillor Shakil
Ahmed (pictured above with Ed Miliband) The nine Britons, who include
three men, two women and four children aged between one and 11, were
seized in Hatay province, in southern Turkey. It shares a border with
part of Syria controlled by rebel factions including those linked
to Al Qaeda and ISIS. All of those held are from Rochdale and are
the biggest family group caught attempting to enter the unstable territory.
Counter terrorism officers at Greater Manchester Police began an investigation
into their movements and the extended family group were detained at
a checkpoint in Ogulpinar earlier this month. A senior officer questioned
why anyone would take children so young 'and vulnerable' into a warzone.
The three men and two women, aged between 21 and 47, were taken to
a hospital with their children, aged one, three, eight and 11. Waheed
(pictured) was detained in Turkey alongside his aunt, two cousins
and one of his cousin's wives . The nine Britons - four of them children
– were seized by Turkish security forces as they tried to slip into
Syria . Officials said they would be photographed and fingerprinted
before being deported back to the UK. At the time, photographs showed
Waheed, dressed in traditional robes and wearing heavy boots, leading
the group from a minibus into a police station. Several women, all
wearing headscarfs which covered their faces, could be seen carrying
children. Most of the party were wearing walking boots, perfect for
trekking across the rugged region. Shakil Ahmed, a bakery delivery
driver, is a councillor in Kingsway and served alongside Karen Danczuk,
wife of Rochdale MP Simon Daczuk, until her resignation in January.
Speaking as he delivered election campaign leaflets earlier this month,
he said he recognised his son in online newspaper reports of his arrest.
He said the others arrested included Waheed's aunt, Zadia Bi, 50,
and two of her sons and one of their wives. He said: 'I don't know
why they have been arrested. We have no information. Until they ring
we will not know what has happened.' He said that he had seen his
son's photograph and when asked about one picture of his son laughing,
he replied: 'Well, they went on holiday so they shouldn't be crying
on holiday should they?' He added: 'I don't believe my son was on
his way to join Islamic State. I was shocked, worried and extremely
upset to hear that my son has been arrested. During their arrest,
the family were fingerprinted and taken to a police station where
they have been held since . One of the family members, holding a child,
is seen arriving at a Turkish hospital to undergo medical checks .
The family are pictured arriving at a police station in Turkey's southern
Hatay province earlier this month . 'It's a total mystery to me why
he's there, as I was under the impression he was on a work placement
in Birmingham. 'My son is a good Muslim and his loyalties belong to
Britain. If I thought for a second that he was in danger of being
radicalised, I would have reported him to the authorities.' The councillor
added: 'He's studying a degree in politics and sociology at Manchester
University and has a good future ahead of him. I just want to speak
to my son and get him home as soon as possible.' Waheed apparently
called his devastated father to break the news he had been arrested.
Sorry we are not currently accepting comments on this article.}{\scriptsize\par}
\par\end{raggedright}
\begin{raggedright}
{\scriptsize{}}%
\begin{tabularx}{\textwidth}{cX}
{\scriptsize{}$b(x,c)$} & {\scriptsize{}x}\tabularnewline
\hline 
\multicolumn{2}{c}{\textbf{\scriptsize{}KL-DPG}}\tabularnewline
{\scriptsize{}1} & {\scriptsize{}waheed Ahmed, 21-year-old was arrested with eight relatives
in a remote northern Turkish border town this month. it is understood
he is now returning to the UK and will fly from Dalaman into Manchester
Airport later this evening. majority of family flew to Turkey on march
27 and are accused of having plans to try and sneak across the border
into Syria.}\tabularnewline
\multicolumn{2}{c}{\textbf{\scriptsize{}TV-DPG}}\tabularnewline
{\scriptsize{}1} & {\scriptsize{}Waheed Ahmed, 21, is the son of Rochdale Labour councillor
Shakil Ahmed. he was arrested with eight relatives in a remote border
town earlier this month. he is now returning to the UK and will fly
from Dalaman into Manchester Airport later this evening.}\tabularnewline
\multicolumn{2}{c}{\textbf{\scriptsize{}JS-DPG}}\tabularnewline
{\scriptsize{}1} & {\scriptsize{}waheed Ahmed, 21, is the son of Rochdale politician.
he was being held after apparently trying to sneak into Syria. he
was arrested with eight relatives – including four children. but it
is understood he is now returning to the UK.}\tabularnewline
\hline 
\end{tabularx}{\scriptsize\par}
\par\end{raggedright}
{\scriptsize{}\caption{Generation samples for summarization}
}{\scriptsize\par}
\end{table}

\newpage

\end{document}